\documentclass[letterpaper]{article} % DO NOT CHANGE THIS
\usepackage[]{aaai23}  % DO NOT CHANGE THIS
\usepackage{times}  % DO NOT CHANGE THIS
\usepackage{helvet}  % DO NOT CHANGE THIS
\usepackage{courier}  % DO NOT CHANGE THIS
\usepackage[hyphens]{url}  % DO NOT CHANGE THIS
\usepackage{graphicx} % DO NOT CHANGE THIS
\usepackage{natbib}  % DO NOT CHANGE THIS AND DO NOT ADD ANY OPTIONS TO IT
\usepackage{caption} % DO NOT CHANGE THIS AND DO NOT ADD ANY OPTIONS TO IT
\usepackage{booktabs} % for professional tables

\usepackage{subfigure}

\usepackage{algorithm}
\usepackage{algorithmic}

\usepackage{amsmath}
\usepackage{amssymb}
\usepackage{mathtools}
\usepackage{amsthm}

\newtheorem{proposition}{Proposition}
\newtheorem{lemma}{Lemma}

\newtheorem*{theorem*}{Theorem}
\newtheorem*{lemma*}{Lemma}
\newtheorem*{proposition*}{Proposition}
\DeclareMathOperator{\tr}{tr}

\DeclareMathOperator*{\minimize}{minimize}
\DeclareMathOperator*{\subjectto}{subject\;to}

\newcommand{\norm}[1]{\left\lVert#1\right\rVert}

\newcommand{\RR}{\mathbb{R}}

\newcommand{\AllZeroRowCol}{$\sf{AllZeroRowCol}$}
\newcommand{\ZeroSingularValue}{$\sf{ZeroSingularValue}$}
\newcommand{\ConditionGrad}{$\sf{ConditionGrad}$}

\newcommand{\expnumber}[2]{{#1}\mathrm{e}{#2}}

\usepackage{multibib}
\newcites{append}{Appendix References} %\citeappend

\title{Beyond NaN: Resiliency of Optimization Layers in The Face of Infeasibility}
\author{Wai Tuck Wong$^1$, Sarah Kinsey$^2$, Ramesha Karunasena$^1$, Thanh H. Nguyen$^2$, Arunesh Sinha$^3$}
\affiliations{$^1$Singapore Management University, $^2$University of Oregon, $^3$Rutgers University\\
wt.wong.2020@msc.smu.edu.sg, sarahevekinsey@gmail.com, rameshak@smu.edu.sg, thanhhng@cs.uoregon.edu, arunesh.sinha@rutgers.edu}

\begin{document}

\maketitle

\begin{abstract}
Prior work has successfully incorporated optimization layers as the last layer in neural networks for various problems, thereby allowing joint learning and planning in one neural network forward pass. In this work, we identify a weakness in such a set-up where inputs to the optimization layer lead to undefined output of the neural network. Such undefined decision outputs can lead to possible catastrophic outcomes in critical real time applications. We show that an adversary can cause such failures by forcing rank deficiency on the matrix fed to the optimization layer which results in the optimization failing to produce a solution. We provide a defense for the failure cases by controlling the condition number of the input matrix. We study the problem in the settings of synthetic data, Jigsaw Sudoku, and in speed planning for autonomous driving. We show that our proposed defense effectively prevents the framework from failing with undefined output. Finally, 
%we highlight that the weakness should not be viewed in isolation; 
we surface a number of edge cases which lead to serious bugs in popular optimization solvers which can be abused as well.
\end{abstract}

\section{Introduction}

There is a recent trend of incorporating optimization and equation solvers as the \emph{final layer} in a neural network, where the \emph{penultimate layer} outputs parameters of the optimization or the equation set that is to be solved~\cite{amos2017optnet,donti2017task,NEURIPS2019_9ce3c52f,wilder2019melding,wang2019satnet,perrault2020end,li2020end,paulus2021comboptnet}. The learning and optimizing is performed jointly by differentiating through the optimization layer, which by now is incorporated into standard libraries. Novel applications of this method have appeared for decision focused learning, solving games, clustering after learning, with deployment in real world autonomous
driving~\cite{xiao2022differentiable} and scheduling~\cite{wang2022decision}. In this work, we explore a novel attack vector that is applicable for this setting, but we note that the core concepts in this attack can be applied to other settings as well. 
While a lot of work exists in attacks on machine learning, in contrast, we focus on a new attack that forces the decision output to be meaningless via specially crafted inputs. The failure of the decision system to produce meaningful output
can lead to catastrophic outcomes in critical domains such as autonomous driving where decisions are needed in real time. Also, such inputs when present in training data lead to abrupt failure of training. Our work \emph{exploits the failure conditions of the optimization layer} of the joint network in order to induce such failure. This vulnerability has not been exploited in prior literature. % on adversarial learning.

\emph{First}, we present a \emph{numerical instability attack}. Typically, an optimization solver or an equation set solver takes in parameters $\theta$ as input. 
In the joint network, this parameter $\theta$ is output by the learning layers and feeds into the last optimization layer (see Fig.~\ref{fig:general_architecture}).
At its core, the issue lies in using functions which are prone to numerical stability issues in its parameters (Appendix \ref{appendix:additional_attacks}).
Most optimization or equation solvers critically depend on the matrix $A$---part of the parameter $\theta$---to be sufficiently far from a singular matrix to solve the problem.  
Our attack proceeds by searching for input(s) that cause the matrix $A$ to become singular. The instability produces \emph{NaNs}---undefined values in floating-point arithmetic---which may result in undesired behavior in downstream systems that consume them.
We perform this search via gradient descent and test three different ways of finding a singular matrix in neighborhood of $A$; only one of which works consistently in practice. 

\begin{figure*}[t]
\centering
\includegraphics[width=0.85\textwidth]{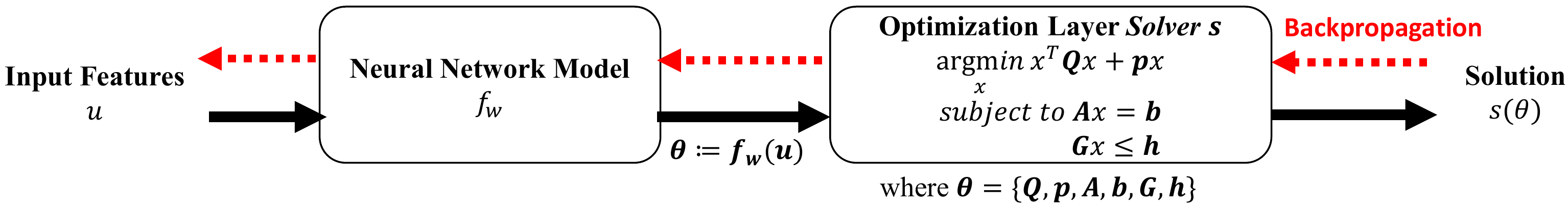}
\caption{Optimization layers in neural networks. The neural network takes input $u$. Some parameters ($Q,p,A,b,G,h$) of the optimization then depend on the output $\theta = f_w(u)$. 
%Note that this differs from prior work, where constraints are bound to a layer and do not change when a different test input is used. The output of the model $s(\theta)$ is the result of the optimization, and loss on the solution can be backpropagated to update model parameters. We also utilize this fact to backpropagate to our input features to construct attack samples when optimizing for instability.
}
\label{fig:general_architecture}
\end{figure*}

%Our second \emph{constraint infeasibility attack} is specialized for optimization with linear inequality constraints, say $A x \leq b$. Most work in optimization as a layer within neural networks use linear constraints. However, none of these handle the failure arising from infeasible constraint $A x \leq b$. While there is no closed form formula (in $A$) for the conditions when the the constraints $A x \leq b$ are infeasible, we utilize a characterization of infeasibility called the Farkas Lemma and derive a novel differentiable loss function using this result. Optimizing this loss function results in successful discovery of inputs that produce intermediate matrix $A$ that causes infeasibility.

\emph{Second}, to tackle the numerical instability attack, we propose a novel powerful defense via an efficiently computable intermediate layer in the neural network. This layer utilizes the singular value decomposition (SVD) of the matrix $A$ and, if needed, approximates $A$ closely with a matrix $A'$ that has bounded condition number; the bound is a hyperparameter. Large condition number implies closeness to singularity, hence
the bounded condition number guarantees numerical stability in the forward pass through the optimization (or equation) solver. Surprisingly, we find that the training performance with our defense in place surpasses the performance of the undefended model, even in the absence of attack, perhaps due to more stable gradients.
%While $A$ with a high condition number still has a valid solution in theory (but fails due to limited precision), an infeasible $A$ for the constraint $Ax \leq b$ implies no solution at all. Thus, we propose a method of achieving robustness against an infeasibility attack  by gracefully (without crashing) discarding such inputs.

\emph{Finally}, we show the efficacy of our attack and defense in (1) a synthetic data problem designed in~\citet{amos2017optnet} and (2) a variant of the Sudoku experiment used in~\citet{amos2017optnet} and (3) an autonomous driving scenario from~\citet{temp_Speedprofileplanning}, where failures can occur even without attacks and how augmenting with our defense prevent these failures. Lastly, we identify other sources of failure in these optimization layers by invoking edge cases in the solver (Appendix \ref{appendix:additional_attacks}). We list serious bugs in the solvers that we encountered.
%; the core ideas here in finding failure conditions will be applicable for attacking other layers of similar nature. 

%Overall, our main contributions are as follows:
%(1) we propose a novel vector of attack that an adversary can use to induce undefined decision output,
%(2) we present a simple yet guaranteed parameterized defense that prevents the failure condition and additionally improves training when the hyperparameter is tuned, and
%(3) we explore other critical bugs in the framework which can potentially be exploited by an adversary to manipulate the output of a network.

\section{Background, Notation, and Related Work}
\label{gen_inst}

\noindent
\textbf{Matrix Concepts and Notation.} 
% In this work we only consider real valued matrices. 
The identity matrix is denoted as $I$ and the matrix dimensions are given by subscripts, e.g., $I_{m \times m}$. The \emph{pseudoinverse}~\cite{laub2005matrix} of any matrix $A$ is denoted by $A^{+}$; if $A$ is invertible then $A^{+} = A^{-1}$.% Pseudoinverse $A^{+}$ generalizes the concept of matrix inverse $A^{-1}$ to any (especially non-square) matrix and $A^{+}$ always exists; if $A \in \mathbb{R}^{m \times n}$ then $A^+ \in \mathbb{R}^{n \times m}$. 
The \emph{condition number}~\cite{belsley1980condition} of non-singular matrix $A$ is defined as $\kappa(A) = \norm{A^+}\norm{A}$ for any matrix norm. We use $\kappa_2(A)$ when the norm used is 2-operator norm and $\kappa_F(A)$ when the norm used is the Frobenius norm. The (thin) SVD of a matrix $A$ is given by $A = U \Sigma V^T$ where $U, V$ have orthogonal columns ($U^T U = I = V^T V$) and $\Sigma$ is a diagonal matrix with non-negative entries. If $A$ is of dimension $m \times n$, then $U, \Sigma, V^T$ are of dimension $m \times r$, $r \times r$, $r \times n$ respectively. The diagonal entries of $\Sigma$ denoted as $\sigma_i = \Sigma_{i,i}$ are the singular values of the matrix $A$; singular values are always non-negative. The condition number directly depends on the largest and smallest singular value as follows: $\kappa_2(A) = \sigma_{\max}/\sigma_{\min}\;$. Also, $\norm{A}_2 = \sigma_{\max}$. $\tr(A)$ denotes the trace of a matrix. 

%\subsection{Background}
\noindent
\textbf{Embedding Optimization in Neural Networks.} Embedding a solver (for optimization or a set of equations) is essentially a composition of a standard neural network $f_w$ and the solver $s$, where $w$ represents weights. The function $f_w$ takes in input $u$
 and produces parameters $\theta$ for the problem that the solver $s$ solves. The solver layer takes $\theta$ as input and produces a solution $s(\theta)$. The composition $s \circ f_w$ can be jointly trained by differentiating through the solver $s$ (see Fig.~\ref{fig:general_architecture}). 
 The main enabler of this technique is efficient differentiation of the solver function $s$. Prior work has shown how to differentiate through solver $s$ where $s$ is a convex optimization problem~\cite{amos2017optnet}, linear equation solver~\cite{etmann2020iunets}, clustering algorithm~\cite{NEURIPS2019_8bd39eae}, and game solver~\cite{li2020end}. Such joint networks have been shown to provide better solution over separate learning and solving~\cite{perrault2020end}.  %Typically, these joint networks directly output the planning decision of the overall AI system.
 
 Many applications of optimization layers focus on training the network end-to-end with the final output representing some decision of the overall AI system, typically called \emph{decision focused learning}~\cite{donti2017task,wilder2019melding,wang2022decision}. Though this has performed well in certain settings, such networks have not been investigated in terms of robustness. %Our work studies and reproduces the conditions under which such optimizations fail and provides sound workarounds without hampering training accuracy. 

%\subsection{Related Work}
\noindent
\textbf{Adversarial Machine Learning.} There is a huge body of work on adversarial learning and robustness that studies vulnerabilities of machine learning algorithms, summarized in many surveys and papers~\cite{43405,42503,akhtar2018threat,biggio2018wild,he2018decision,papernot2018sok,li2019generative,anil2019sorting,tramer2020adaptive}. Our work is  different from prior work as our attack targets the \emph{stability} of the optimization solver that is embedded as a layer in the neural network and our defense stops the attack by preventing singularity. To the best of our knowledge, our work is the first work to explore this aspect.

%Further, we adapt existing datasets to present a novel yet natural setting in which such constraints are different at each instance and have to be inferred from the data, thus allowing the use of such optimization layers in more general settings.

\noindent
\textbf{Robustness against Numerical Stability in Optimization.}
Repairing is an approach proposed in recent work~\cite{barratt2021automatic} to compute the closest solvable optimization when the input generic convex optimization is infeasible. While possessing the same goal as our defense, this repairing approach is computationally prohibitive for use in neural networks as the repairing requires solving tens to hundreds of convex optimization problems just to repair a single problem instance. Optimization layers are considered slow even with just one optimization in the forward pass~\cite{amos2017optnet,NEURIPS2019_9ce3c52f,NEURIPS2020_6d0c9328}, hence multiple optimizations to repair the core optimization in every forward pass is not practical for neural networks. Our defense is computationally cheap due to the targeted adjustment of specific parameters of the optimization.

Pre-conditioning~\cite{wathen2015preconditioning} is a standard approach in optimization 
%for managing high condition number of matrix parameters. 
%There are many variants of pre-conditioning and some of these variants are computationally very expensive or apply in narrow setting with square matrices only; 
%Pre-conditioning 
that helps the solver deal with ill-conditioned matrices better than without pre-conditioning. However, \emph{even with preconditioning, solvers cannot handle specially crafted singular input matrices}. Our defense does not allow any input that the solver cannot handle.

\section{Methodology} \label{sec:methodology}

\noindent
\textbf{Threat Model:} We are given a trained neural network which is a composition of two functions $f_w$ and $s$, where $w$ represents neural network weights and $w$ is known to the adversary (i.e., the adversary has whitebox access to the model). The function $f_w$ takes in input $u$
and produces $\theta = f_w(u)$. $\theta$ defines some of the parameters to our solver (Fig. \ref{fig:general_architecture}). In this paper, we analyze a specific component of $\theta$ which corresponds to the intermediate matrix $A$ (in $Ax =b$). For example, if $\theta$ only consists of $A$, it can be formed by reshaping $\theta$, where the $i,j$ entry of $A$ is $\theta_{i,j}$. The solver layer takes $A$ as input and produces a solution $s(A)$. The attacker's goal is to craft any input $u^*$ such that $s(f_w(u^*))$ fails to evaluate successfully due to issues in evaluating $s$ stemming from numerical instability, effectively causing a denial of service. Note that the \emph{existence} of any such input $u^*$ is problematic and we allow latitude to the attacker to produce \emph{any} such input as long as syntactical properties are maintained, e.g., bounding image pixel values in 0 to 1. In this setting, an attacker can also craft an attack input that is close to some original input if needed (Fig. \ref{fig:attack_images}), e.g., when they need to foil a human in the loop defense. Even in this worst case scenario of allowing the attacker to provide any input, our proposed defense prevents NaNs in all cases.  %We aim for $u^*$ to be close to some given $u$ in order for such input to be realistic.

We emphasize the distinction between the goal of our attack inputs and that of adversarial examples. In traditional adversarial examples, small perturbations to the input image is sought in order to show the surprising effect that two images that appear the same to the human eye are assigned different class labels, but these misclassified labels can still be consumed by downstream systems. In contrast, in our work, the \emph{surprise} is the \emph{existence} of inputs that cause a complete failure in the outcome of the system, which to our knowledge have not been previously studied. Here, we show the existence of specially crafted inputs, which may be semantically close to a valid input, that evaluate to outputs that cause a complete denial of service, i.e., NaNs are produced, 
%and downstream systems cannot meaningfully interpret the output at all, 
leading to undefined behavior in the system. A naive remediation of a default safe action for NaN outputs can fail in complex domains (e.g.,
autonomous driving) which have context-dependent safe actions
(e.g., the safest action on a highway with a speed-limit
road sign depends on various conditions such as speed of the
car in front, need to change lane, etc.). It is thus impossible
to provide a rule-based safe default action since there can
be infinitely many contexts.
%which may prove dangerous in safety critical domains.

\subsection{Numerical Instability Attack}

%The range of the function $f_w$ determines whether the output $A$ can be ill-conditioned. As such the range is almost impossible to determine for a fixed $w$; in order to analyze an average case we state a prior result~\citep{chen2005condition} about the probability of a matrix $A \in \mathbb{R}^{m \times n}$ being ill-conditioned when sampled uniformly on the sphere of normalized matrices given by $\norm{A} = 1$ (or equivalently each element is sampled from standard normal, as shown in ~\cite{}):
%$$
%\left(\frac{c}{x}\right)^{|n-m| + 1} \leq P\left(\kappa(A) \geq \frac{n x}{|n-m| + 1} \right) \leq \left(\frac{C}{x}\right)^{|n-m| + 1}
%$$
%The above results shows that high condition number occupies a smaller space when $|n-m|$ is higher, thus, we hope to see numerical instability attacks work better on square matrices. We validate this claim in our experiments.

\begin{figure}[t]
\centering
\subfigure[Original Images]{
    \includegraphics[width=0.12\columnwidth]{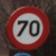} \includegraphics[width=0.26\columnwidth]{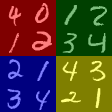}
}
\subfigure[Attack Images]{
    \includegraphics[width=0.26\columnwidth]{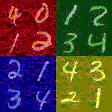}
    \includegraphics[width=0.12\columnwidth]{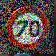}
}
\caption{Left shows original image $u$, right shows $u'=u+\delta$ which is semantically close. All attacks were found using \AllZeroRowCol{} with a upper bound on the perturbations.}
\label{fig:attack_images}
\end{figure}

In our attack, we seek to find an input $u^*$ that evaluates to a rank deficient intermediate matrix $A$ (Fig. \ref{fig:general_architecture}). For any $m \times n$ matrix $A$, $A$ is rank-deficient if its rank is strictly less than $\min(m,n)$. A rank deficient matrix is also singular, hence the system of equations $A x = b$ ($b \neq 0$) produces undefined values (NaN) when solved directly or as constraints in an optimization. Even matrices close enough to singularity can still produce errors due to the limited precision of computers. Depending on the neural network $f_w$ (Fig. \ref{fig:general_architecture}), finding $u$ that produces an arbitrary singular $A$ (e.g., $0_{m \times n}$) is not always possible (see Appendix~\ref{appendix:targetzero}). Our approach is guided by the following known result

\begin{proposition}[\citet{demko1986condition}] \label{prop:demko}
For any matrix $A$, the distance to closest singular matrix is $\min_B\{\norm{A-B}_2: B \mbox{ is singular}\} = \norm{A}_2 / \kappa_2(A) = \sigma_{\min}$
\end{proposition}
Thus, increasing the condition number of $A$ moves $A$ closer to singularity; at singularity $\kappa_2(A)$ is $\infty$.
 Following Alg. \ref{alg:algorithm_attack}, we start with a given $u$ producing a well-conditioned matrix $A$ and aim to obtain $u^*$ producing singular $A'$ in the \emph{vicinity} of $A$ using three approaches: \AllZeroRowCol{}, \ZeroSingularValue{}, and \ConditionGrad{}. 
 %All three approaches should work in theory, but we find that \AllZeroRowCol{} works the most consistently in practice. In the following, we present these three approaches in detail.

\begin{algorithm}[tb]
\small
\caption{Numerical instability attack}
\label{alg:algorithm_attack}
\textbf{Input}: input features $u$, loss function $\ell$, victim model $f_w$ \\
\textbf{Parameters}: learning rate $\alpha$ \\
\textbf{Output}: attack input $u^*$
\begin{algorithmic}
\STATE Let $u^*= u$.
\WHILE{$\kappa_2(f_w(u^*)) \neq \infty $}
\STATE $l = \ell(f_w(u^*))$ \COMMENT{$\ell$ is a technique dependent loss function}
\STATE Update $u^*$ based on $\alpha{},\frac{\delta{}l}{\delta{}u^*}, \ell$
\ENDWHILE
\STATE \textbf{return} $u^*$
\end{algorithmic}
\end{algorithm}

\noindent
{\AllZeroRowCol{}:} An approach to obtain a rank-deficient matrix $A'$ from $A$ is to zero out a row (resp. column) in case $m < n$ (resp. $m > n$) in $A$. Then, we use $A'$ as a target matrix for which a gradient descent-based search is performed to find an input $u^*$, that yields $A' = f_w(u^*)$. In our experiments, we choose the first row/column to zero out, though choosing other rows/columns is equally effective. 

\noindent{\ZeroSingularValue{}:} From Prop.~\ref{prop:demko}, $A'$ is a \emph{closest singular matrix} if $\norm{A - A'}_2 = \sigma_{\min}$. An approach to obtain this rank-deficient matrix $A'$ from $A$ is to perform the SVD $A = U \Sigma V^T$, then zero out the smallest singular value in $\Sigma$ to get $\Sigma'$, and then construct $A' = U \Sigma' V^T$. It follows from the construction that $\norm{A - A'}_2 = \sigma_{\min}$. Then, using $A'$ as a target matrix a gradient descent-based search is performed to find $u^*$ that yields $A' = f_w(u^*)$. In theory, since $A'$ is a \emph{closest singular matrix} it should be easier to find by gradient descent compared to \AllZeroRowCol{}. However, this approach fails in practice because precision errors make $A'$ non-singular even though $\Sigma'$ has a zero singular value. 
%\thanh{Is there any advantage of ZeroSingularValue compared to AllZeroRowCol in theory?}

\noindent{\ConditionGrad{}:}  From Prop.~\ref{prop:demko}, we can also use gradient descent to find $u^*$ such that the matrix $A$ has a very high condition number. The overall gradient we seek is $\frac{\partial \log \kappa_2(A)}{\partial u}$, where we use $\log$ as condition numbers can be large. Following chain rule, we get $\frac{\partial \log \kappa_2(A)}{\partial u} = \frac{1}{\kappa_2(A)} \frac{\partial \kappa_2(A)}{\partial \theta} \frac{\partial \theta}{\partial u}$. Since $\theta = f_w(u)$, the third term is simply the gradient through the neural network. The second term can be obtained component wise in $\theta$ as $\frac{\partial \kappa_2(A)}{\partial \theta_{i,j}}$ for all $i,j$. The following result provides a closed form formula for the same (proof in Appendix \ref{appendix:proofs}).
\begin{lemma} \label{lemma:gradcond}
Let $A \in \mathbb{R}^{m \times n}$ with thin SVD $A = U \Sigma V^T$ and $\sigma_{\max} = \sigma_1 \geq \ldots \geq \sigma_r = \sigma_{\min}$ for $r = \min(m,n)$. Then, $\frac{\partial \kappa_2(A)}{\partial \theta_{i,j}}$ is given by $\tr\Big( \frac{\partial (\lvert\lvert A^{+}\rvert\rvert_2 * \lvert\lvert A\rvert\rvert_2)}{\partial A} \cdot\frac{\partial A}{\partial \theta_{i,j}}\Big)$ where
{\small
\begin{align*}
     &\frac{\partial (\lvert\lvert A^{+}\rvert\rvert_2 * \lvert\lvert A\rvert\rvert_2)}{\partial A}  = B^T   -  (A^{+} C A^{+})^T  + \\
    & \quad (A^+)^T  A^+ C (I - A^+A)  +  (I - AA^+) C A^+ (A^+)^T 
\end{align*} }
with $B \!=\! \lvert\lvert A^+ \rvert\rvert_2 V e_1 e^T_1 U^T$\!, $C \!=\! \lvert\lvert A\rvert\rvert_2 U e_r e^T_r V^T $ and $e_i$ is the unit vector with one in the $i^{th}$ position.
\end{lemma}
\ConditionGrad{} still works less consistently than \AllZeroRowCol. This is mainly because the gradient descent often saturates at a condition number that is high but not large enough for instability. 

A low-dimension illustration of the approaches is in Fig.~\ref{fig:explaination}, which shows the 2D space of the two singular values $\sigma_1, \sigma_2$ of all $2 \times 2$ matrices. The condition number ($\sigma_{\max} / \sigma_{\min}$) approaches infinity near the axes and the $\infty$ condition number on the axes is difficult to reach in \ConditionGrad{}. The illustration also shows why \AllZeroRowCol{} works more consistently than \ZeroSingularValue{} as recovering a matrix from $\Sigma'$ involves multiplication which leads to loss of singularity (more so in high dimension) whereas \AllZeroRowCol{} directly obtains a singular matrix. This is reflected in our experiments later.

We note that simple approaches such as attempting to use gradient descent or other existing approaches to directly maximize model output to very high values fails due to saturation (see results in Appendix~\ref{appendix:maxpgd}). Further, the optimization output and ill-conditioning of $A$ can have no relation at all: 

\begin{lemma}\label{lemma.1}
For an optimization $\min_{\{x|Ax = b\}} f(x)$ with $f$ convex, the solution value (if it exists) can be made arbitrarily large by changing $\theta = \{A,b\}$ while keeping $A$ well-conditioned.
\end{lemma}

Lemma~\ref{lemma.1} implies that $A$ can remain well-conditioned even though output $\min_{\{x|Ax = b\}} f(x)$ is large. Thus, specifically targeting to directly obtain a singular matrix $A$ is important for a successful NaN attack (proof in Appendix~\ref{appendix:proofs}).
%\thanh{I think we need to elaborate the meaning/importance of this lemma more.}

%We can also use $\kappa_F(A)$ instead of $\kappa_2(A)$, the gradient for which is presented in the appendix. 
%Overall, the algorithm is summarized in Algorithm~\ref{alg:stable}.\thanh{We can remove this algorithm.}

\begin{figure}[t]
\centering \includegraphics[width=0.95\columnwidth]{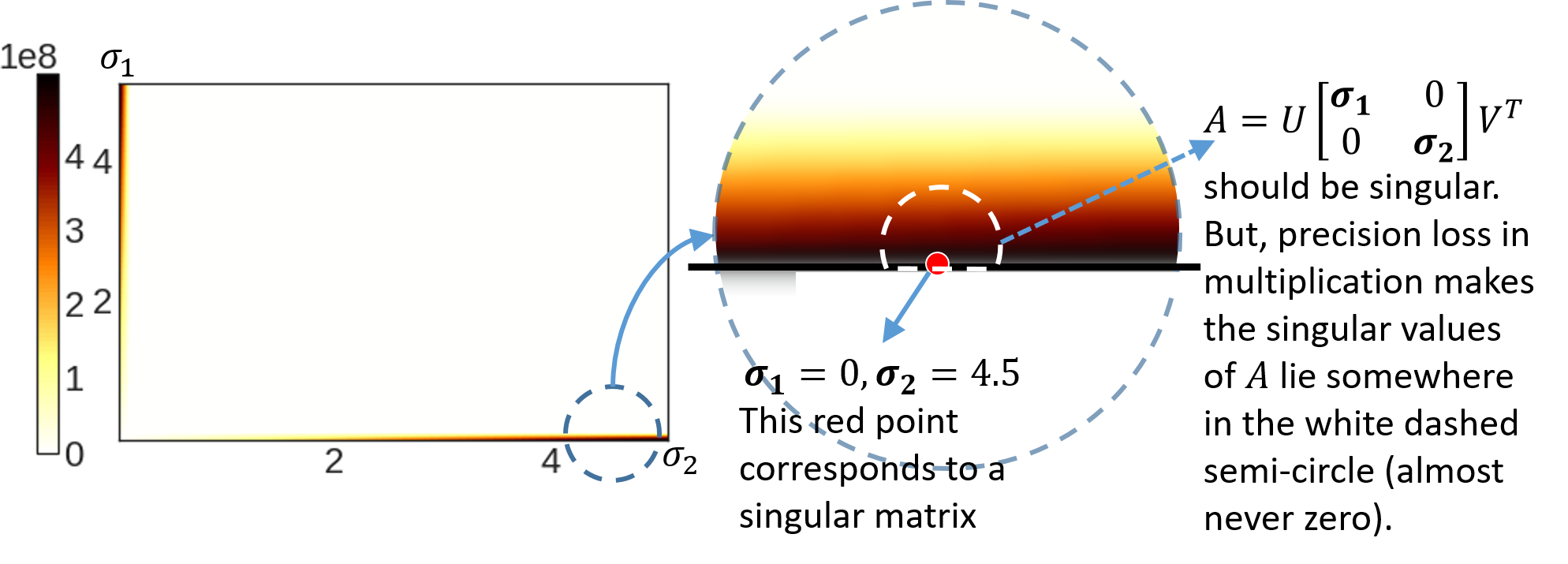}
\caption{Left is a heatmap of condition numbers ($\sigma_{\max} /\sigma_{\min}$) for 2D singular value space ($\sigma_1, \sigma_2$) of $2\times 2$ matrices (high condition number only near axes, as one of $\sigma_1$ or $\sigma_2$ approaches 0 since $\sigma_{\min} = \min \{\sigma_1,\sigma_2\}$). Right is an enlarged version of the smaller dashed circle.
}
\label{fig:explaination}
\end{figure}

\subsection{Defense Against Numerical Instability}
\label{sec:defense}

\begin{algorithm}[tb]
\small
\caption{Numerical instability defense}
\label{alg:algorithm_defense}
\textbf{Input}: model $f_w$, input features $u$ \\
\textbf{Parameter}: condition number bound $B$ \\
\textbf{Output}: well-conditioned $A'$ 
\begin{algorithmic}
\STATE Let $A' = A = f_w(u) = U\Sigma{}V^T$.
\IF {$\kappa_2(A) > B$}
\STATE For all $i$, let $\Sigma'_{i,i} = $ min($\sigma_i$, $\sigma_{max}/B$)
\STATE $A' = U\Sigma{}'V^T$
\ENDIF
\STATE return $A'$.
\end{algorithmic}
\end{algorithm}

First, we note that our goal is to fix the instability in the optimization used in the final layer, which is distinctly different from the general problem of instability of training neural networks~\cite{colbrook2022difficulty}. Next, we discuss defense for \emph{square matrices} $A$. For symmetric square matrices, the condition number can be stated in term of eigenvalues: $\kappa_2(A) = \frac{|\lambda|_{\max}}{|\lambda|_{\min}}$ where $|\lambda|_{\max}$ is the largest eigenvalue by magnitude. A typical heuristic to avoid numerical instability for square matrices is to add $\eta I$ for some small $\eta$~\cite{StableDNN}. 
%This is because the eigenvalues of $A + \eta I$ is $\lambda + \eta$ and hence would move the eigenvalues to improve the condition number. 
However, this approach \emph{only works for square} positive semi-definite (PSD) matrices. If some eigenvalue of $A$ happens to $-\eta$ then this heuristic actually makes the resultant matrix non-invertible (i.e., infinite condition number). Besides, clearly this heuristic \emph{does not apply for non-square matrices}. 

As a consequence, we propose a differentiable technique (Alg. \ref{alg:algorithm_defense}) that directly \emph{guarantees} the condition number of \emph{any} intermediate matrix to be a bounded by a hyperparameter $B$. In the forward pass, we perform a SVD of $A = U \Sigma V^T$; the computation steps in SVD are differentiable and the matrix $\Sigma$ gives the singular values $\sigma_i$'s. %\thanh{Does that mean $U$, $V$, and $\Sigma$ are trained from the same network?}. 
Recall that the condition number $\kappa_2=\sigma_{\max} / \sigma_{\min}$. The condition number can be controlled by clamping the $\sigma_i$'s to a minimum value $ \sigma_{\max}/B$ to obtain a modified $\Sigma'$. 
%where $B$ is a hyperparameter to be set to the desired maximum tolerance of the condition number. 
Then, we recover the approximate $A' = U \Sigma' V^T$. 
%Note that in the case that matrix $A$ already has its condition number lower than $B$ then there is no modification, i.e., $A' = A$, and if all singular values are zero then we add some small noise to all the singular values. Finally, $\sigma_{\max}$ can be obtained by ${\mathsf{torch.max}}$ in a \emph{differentiable} manner.
%by logsumexp: $(1/t)\log (\sum_i e^{t \sigma_i})$, where higher temperature $t$ makes $(1/t)\log (\sum_i e^{t \sigma_i})$ a better approximation of $\sigma_{\max}$. 
We present the following proposition (proof in Appendix \ref{appendix:proofs}).

\begin{figure*}[t]
\centering
\includegraphics[width=0.85\textwidth]{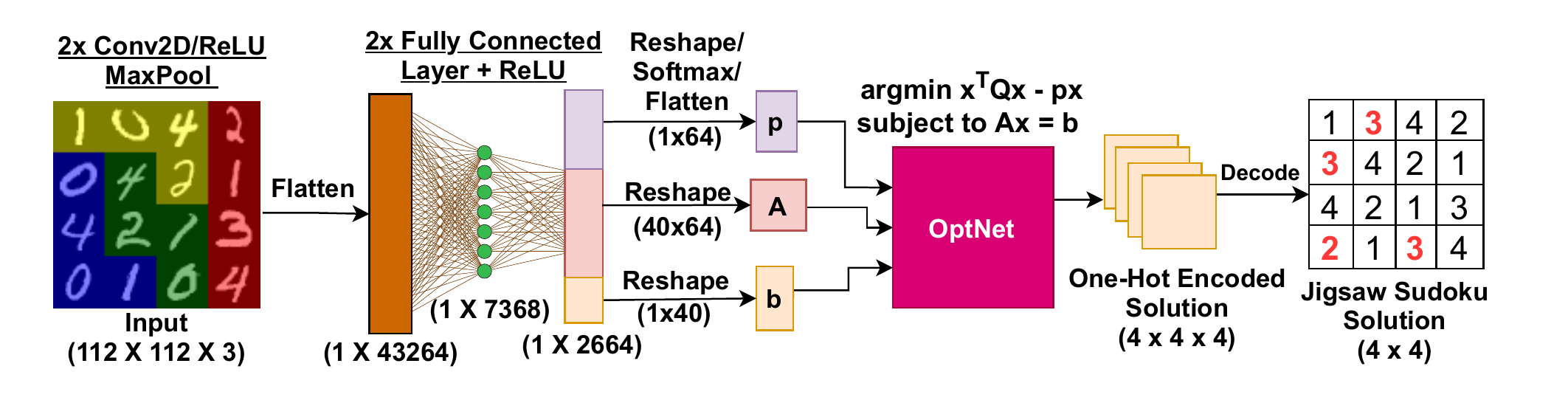}
\caption{The Jigsaw Sudoku network architecture. The input is an image of the Jigsaw Sudoku puzzle (0 indicates blank cell that needs to be filled with a value in $\{1, 2, 3, 4\}$) and 
%where it is passes through a 3x3 zero padding convolution layer with a ReLU activation, followed by a 2x2 max pooling layer. This is repeated before flattening and passing the result through to 2 fully-connected layers with ReLU activations. The resulting vector is then reshaped to $p$, representing the puzzle, and $A, b$, the equality constraints corresponding to the jigsaw Sudoku puzzle. These are then passed as parameters to the quadratic program $\argmin_{z}{z^TQz - pz}$ subject to the constraints $Az = b$. The 
the output is the solution to the puzzle given the constraint that no two numbers in the same colored region are the same. The solution to the blank cells given by the neural network is indicated in red.}
\label{fig:network_diagram_jigsaw_sudoku}
\end{figure*}

% The following result can be proved about the approximation introduced above.
\begin{proposition}
\label{prop:cond_def}
For the approximate $A'$ obtained from $A$ as described above and  $x'$ a solution for $A' x = b$, the following hold: (1)
$\norm{A' - A}_2 \leq  \sigma_{\max} /B $ 
and (2) $\frac{||x^* - x'||_2}{||x'||_2} \leq \kappa_2(A) / B$ for some solution $x^*$ of $A x = b$.
\end{proposition}

The second item (2) shows that approximation of the solution obtained from the solver depends on $\kappa_2(A)$, which can be large if $\kappa_2(A)$ is close to infinity. This error estimate can be provided to downstream systems which can be used in the decision on whether to use the solver's output.

\section{Experiments}
\label{experiments_section}

We showcase our attacks and defenses on three different domains: (i) synthetic data modelling an assignment problem, (ii) decision-focused solving of Jigsaw Sudoku puzzles, and (iii) real world speed profile planning for autonomous driving. We compare the success rate of the three attack methods, namely \AllZeroRowCol{}, \ZeroSingularValue{} and \ConditionGrad{}.
%For defense, we apply our defense and compare the results
We show that the defense is effective by comparing the condition numbers of the constraint matrix $A$ during test time. We also show that the attack fails with the defense for varying values of $B$: 2, 10, 100, and 200.
Further, 
%to prevent spurious NaNs during training, 
we augment our model with the defense during training time and show that it effectively prevents NaNs in training while not sacrificing performance compared to the \emph{original} model.
We discuss the results in detail at the end in Section~\ref{sec:results}.

For all experiments, we used the \texttt{qpth} batched QP solver as the optimization layer~\cite{amos2017optnet} and PyTorch 1.8.1 for SVD. 
%We note that SVD implementations differ in the final output matrix depending on how well-conditioned the input matrix is \cite{10.1137/050639193}, but such differences do not affect performance of our method (see Appendix). 
For the synthetic data, we ran the experiments on a cloud instance (16 vCPUs, 104 GB memory) on CPU. For other settings, we ran the experiments on a server (Intel(R) Xeon(R) Gold 5218R CPU, 2x Quadro RTX 6000, 128GB RAM) on the GPU. 

\subsection{Synthetic Data}

The setting used follows prior work ~\cite{amos2017optnet} to test varying constraint matrix sizes used in optimization.  We interpret this prior abstract problem as an assignment problem under constraints, where inputs are assigned to bins with constraints that are learnable.  The model learns parameters in the network to best match the bin assignment in the data. 
%We synthetically construct the data as follows: 
The input features of input $u$ are generated from the Gaussian distribution and assigned one out of $n$ bins uniformly. Bin assignment is a constrained maximization optimization, where only the constraint affects binning; thus, the objective is arbitrarily set to $ \norm{x}_2$ with the constraint $Ax = b$, where $A, b$ are learned and $x \in \mathbb{R}^n$ gives the bin assignment.  Here, $A$ has the size $m\times n$, where $m$ is the number of equalities and $n$ is the number of bins. A softmax layer at the end enforces an assignment constraint.

\noindent
\textbf{Experimental Setup:}
For the training of the network, in each of the randomly seeded training run, we draw 30 input feature vectors $u \in \RR^{500}$ from the Gaussian distribution and assign them uniformly to $n$ bins. We do the same for the test set comprising of 10 test samples. We ran the training over 1000 epochs using the Adam optimizer~\cite{DBLP:journals/corr/KingmaB14} with a fixed learning rate of $\expnumber{1}{-3}$. For the attack experiments, we ran each of the attacks for 5000 epochs on 30 input samples drawn from the Gaussian distribution on each of the 10 models that were trained. An attack is marked successful if any of the modified inputs produces a NaN. For the training of the defended models, we varied the hyperparameter $B$. The models are evaluated using \emph{cross-entropy loss} against the true bin in which the sample was assigned. The test loss is averaged over 10 randomly seeded runs. 

\noindent
\textbf{Results:} In this setting where an attacker can arbitrarily change the input vector at test time, we report the success rate of each of the attack methods in Table~\ref{tab:eq_attack_rate} and the loss results of models trained with the defense in Table~\ref{tab:eq_training_def} for the non-square matrix $A \in \RR^{40 \times 50}$ case and square matrix $A \in \RR^{50 \times 50}$ case. We see our methods are broadly applicable to all matrices as both the attack and the defense achieve their goals regardless of the shape of the matrix. Further, test performance in the \emph{baseline} $\eta I$ defense (with $\eta=10^{-8}$, applicable \emph{only} for square matrices) in the $A \in \RR^{50 \times 50}$ case is worse than \emph{both} the original and our proposed defense when $B=200$, with a higher loss at $4.86$ $\pm$ $1.74$ (Appendix~\ref{appendix:etai}).

\subsection{Jigsaw Sudoku}
\label{sec:jigsawsudoku}

%\begin{figure}[t]
%\centering
%\subfigure{
%\includegraphics[width=.225\textwidth]{jigsaw0.png}
%}
%\subfigure{
%\includegraphics[width=.225\textwidth]{jigsaw1.png}
%}
%\subfigure{
%\includegraphics[width=.225\textwidth]{jigsaw2.png}
%}
%\subfigure{
%\includegraphics[width=.225\textwidth]{jigsaw3.png}
%}
%\caption{\textbf{Jigsaw Sudoku Puzzles.} 0 indicates an empty spot that is to be filled with one of $1,2,3,4$. The colors indicate shape of constraint. Each shape should have only one of 1, 2, 3, or 4.}
%\label{fig:jigsaw_sudoku_examples}
%\end{figure}

%\thanh{I am still confused about the MNIST Jigsaw Sudoku. As I understand from Google search :D, each region must be fill with value from 1 to 9 but it seems we consider a smaller range. Also, what is the meaning of initial digit numbers in MNIST dataset? We need to explain what is the learning and what is the optimization here. What is the meaning of the input. etc}
 Sudoku is a constraint satisfaction problem, where the goal is to find numbers to put into cells on a board (typically $9\times9$) with the constraint that no two numbers in a row, column, or square are the same. In prior work~\cite{amos2017optnet}, optimization layers were used to learn constraints and obtain solutions satisfying those constraints on a simpler $4\times4$ board. We note that in the above setting, the constraints $(A,b)$ are fixed and do not vary with the input Sudoku instances and hence, our test time attack does not apply in this case. Instead, we consider a popular variant of the $4\times4$ Sudoku---Jigsaw Sudoku---where constraints are not just on the rows and columns, but also on other geometric shapes made from four contiguous cells. In this setting, the constraints now vary with input puzzle instances. We represent each Sudoku puzzle as an image (see Fig.~\ref{fig:network_diagram_jigsaw_sudoku}) and mark each constraint on contiguous shapes with a different color. %The 0 in the image denotes an empty slot that needs to be filled to solve the puzzle. %Such images serve as input to the neural network. 

The network (Fig.~\ref{fig:network_diagram_jigsaw_sudoku}) has to (i) \emph{infer} the one-hot encoded representation of the Sudoku problem $p$ (a $4\times 4\times4$ tensor
with a one-hot encoding for known entries and zeros for
unknown entries); (ii) \emph{infer} the constraints to apply ($A, b$ in $Ax = b$); and (iii) \emph{solve} the optimization task to output the solution that satisfies the constraints of the puzzle --- all these steps have to be derived just from the image of the Jigsaw Sudoku puzzle. 
%We generated the puzzles (as images) programmatically and replaced the digits on the Sudoku board with digits from the MNIST dataset, and marking the regions that the network is supposed to solve for to be 0.
A small $Q=0.1I$ ensures strict positive definiteness and convexity of the quadratic program. 
% Note that we differ from the setup in \cite{amos2017optnet} in that constraints vary for different puzzles and hence are inferred from each input puzzle image, rather than having one fixed set of learned constraints for all puzzles.

\begin{figure*}[t]
\centering
\includegraphics[width=0.85\textwidth]{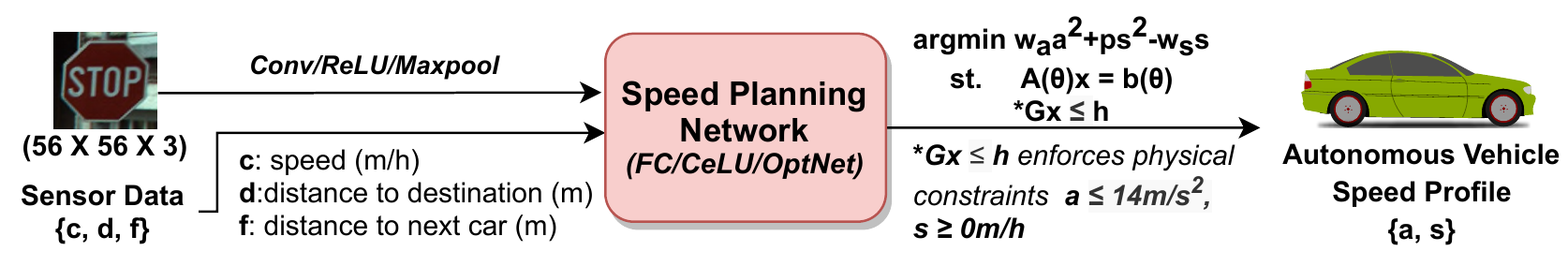}
\caption{Autonomous vehicle speed planning architecture with CeLU~\cite{barron2017continuously} activated layers.}
\label{fig:network_diagram_speedplanning}
\end{figure*}

\noindent
\textbf{Experimental Setup:} We generated 24000 puzzles of the form shown in Fig.~\ref{fig:network_diagram_jigsaw_sudoku}  by composing and modifying images in the MNIST dataset \cite{lecun-mnisthandwrittendigit-2010} using a modified generator from \cite{amos2017optnet} (details in Appendix \ref{appendix:jigsaw_sudoku}).
We trained the model architecture in Fig. \ref{fig:network_diagram_jigsaw_sudoku} on 20000 Jigsaw Sudoku images, utilizing Adadelta \cite{DBLP:journals/corr/abs-1212-5701} with a learning rate of 1, batch size of 500, training over 20 epochs, minimizing the MSE loss against the actual solution to the puzzle. We then test the model on 4000 different held-out puzzles. We repeat the experiments with 30 random seeds for each configuration. For the defense, we restrict the condition number by applying our defense in Section~\ref{sec:defense} over several values of the hyperparameter $B$. For all models, we measure MSE loss and \emph{accuracy} which is the percentage of cells with the correct label in the solution produced by the network. For all attack methods, we apply a model-tuned learning rate and optimize for the attack loss for a given image for 500 epochs until we generate an image that causes failure. We then repeat this for 30 test images.

\noindent
\textbf{Results:}  In this setting, the attack may only modify the input image, which is constrained as a tensor with pixel values in the range $[0,1]$. Even with these constraints,  \AllZeroRowCol\ consistently finds an input that results in NaNs in the output (Table~\ref{tab:eq_attack_rate}), showing the effectiveness of our attack. Looking at the difference in loss (Fig.~\ref{fig:lossattack}) and condition number of the matrix $A$ (Fig.~\ref{fig:conditionattack}) for the defended and original undefended network, we see the efficacy of our defense in controlling the condition number and preventing the NaN outputs during test time. Finally, plotting the change in training loss over the epochs for $B=100$ and the original model in Fig. \ref{fig:defensetraintest}, we see virtually no difference in epochs to convergence when the defense is applied in training time. We note the observations above apply for all experimental settings for all reported values of $B$, see Appendix~\ref{appendix:additional_experiment_results} for details.

\subsection{Autonomous Vehicle Speed Planning}

\begin{table*}[t]
\centering
%\vspace{0.1in}
\begin{tabular}{l c c c c c c}
\toprule
 &
  \AllZeroRowCol &
  \ZeroSingularValue &
  \ConditionGrad \\
  
  \midrule 
\multicolumn{1}{l}{\textit{\textbf{Synthetic (m=40, n=50)}}} &
  100.00 &
  0.67 &
  85.33 &
 \\ 
\multicolumn{1}{l}{\textit{\textbf{Synthetic (m=50, n=50)}}} &
  98.00 &
  0.00 &
  0.00 &
 \\ 
\multicolumn{1}{l}{\textit{\textbf{Jigsaw Sudoku}}} &
  100.00 &
  6.67 &
  53.33 &
 \\ 
\multicolumn{1}{l}{\textit{\textbf{Speed Planning}}} &
 100.00 &
 0.00  &
 0.00  &
\\
\multicolumn{1}{l}{\textit{\textbf{Defense (B=2,10,100,200)}}} &
 \textbf{0.00} &
 \textbf{0.00}  &
 \textbf{0.00}  &
\\
\bottomrule
\end{tabular}
%}
\caption{Comparison of attack success (\% of Successful NaNs) for all methods and datasets. Last row shows defense.}
\label{tab:eq_attack_rate}
\end{table*}

\begin{figure}[t]
\centering \includegraphics[width=0.85\columnwidth]{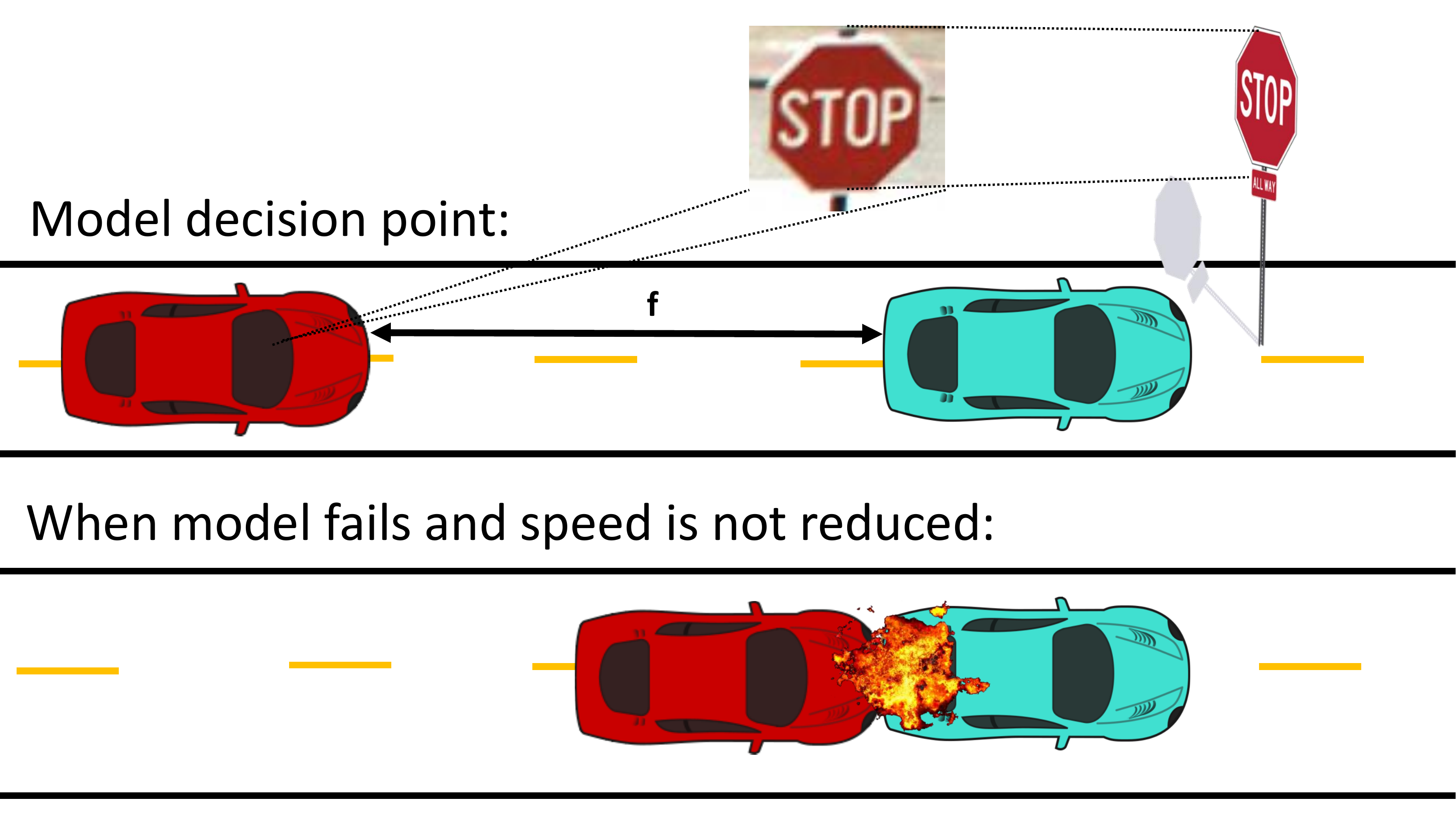}
\caption{The car determines the optimal speed profile based on the signage and distance $f$ from the car ahead. When the optimization fails, the car exhibits undefined behavior and may assume an unsafe profile, leading to a crash.}
\label{fig:speed_planning_illustrated}
\end{figure}

In autonomous driving, a layered framework with separate path planning and speed profile generation is often used due to advantages in computational complexity~\cite{7353382}. Here, we focus on speed profile generation, where constrained optimization is employed to maximize comfort of the passengers while ensuring their safety and adhering to physical limitations of the vehicle~\cite{6856581}. We consider the scenario where a traffic sign is observed and the autonomous vehicle has to make a decision on the acceleration and target speed of the vehicle as shown in Fig.~\ref{fig:speed_planning_illustrated}.

The autonomous vehicle seeks to make the optimal decision in speed planning taking into account the constraints presented. The learning problem involves identifying the traffic sign and inferring the rules to apply based on the current aggregate state of the autonomous vehicle collected from sensors. We provide as input $u$ an image of the traffic sign along with the state of the vehicle, defined as $V = \{c, d, f\}$, where $c, d, f \in \RR_{\geq 0}$, where, $c$ is the current speed of the vehicle (in meters per hour), $d$ is the distance to the destination (in meters), and $f$ is the distance to the vehicle ahead (in meters).  % The optimization problem seeks to minimize the discomfort of the passenger and the time needed to get to the destination. 
%Note that the speed profile is planned on a single direction in this scenario and can be generalized to take into account longitudinal speed, acceleration and jerk. 
Similar to \cite{temp_Speedprofileplanning}, we aim to minimize discomfort $a^2$ ,where acceleration $a \in \RR$, and maximize the target speed $s \in \RR_{\geq 0}$ using the quadratic program shown in Fig.~\ref{fig:network_diagram_speedplanning},
where $w_a, w_s \in \RR$ are tunable weights on the speed and acceleration, and $p \in \RR_{\geq 0}$ is a small penalty term to ensure the problem is a quadratic program. When input $u$ is fed into the network, $\theta$ is the output of the network right before the optimization layer, and $A(\theta)$ and $b(\theta)$ depend on $\theta$. These equalities encode rules that will apply based on the traffic sign observed, e.g. a \emph{Stop} sign would signal to the vehicle to set its target speed $s$ to 0. We encode physical constraints of the autonomous vehicle (e.g., maximum acceleration, positivity constraints on speed) in $G$ and $h$ which do not depend on $\theta$. 

\begin{table*}[t]
\centering
\begin{tabular}{l c c c c c}
\toprule
 &
  \multicolumn{2}{c}{\textbf{Synthetic Data (CE)}} &
  \multicolumn{2}{c}{\textbf{Jigsaw Sudoku}}
  &
  \multicolumn{1}{c}{\textbf{Speed Planning}}\\ \cmidrule{2-6} 
 &
  \textbf{m=40, n=50} &
  \textbf{m=50, n=50} &
  \textbf{Test Loss} &
  \textbf{Test Acc.} &
  \multicolumn{1}{c }{\textbf{Test Loss}} \\ 
  \midrule 
\multicolumn{1}{l}{\textit{\textbf{Original}}} &
  24.99 $\pm$ 2.03 &
  $4.43 \pm 0.93$ &
  $1.09 \pm 1.03$ &
  $0.91 \pm 0.20$ &
  7928.76 $\pm$ 38565  \\ 
\multicolumn{1}{l}{\textit{\textbf{B=2}}} &
  \textbf{9.14 $\pm$ 0.77} &
  $6.41 \pm 0.58$ &
  $0.93 \pm 0.73$ &
$0.94 \pm 0.15$ &
 \textbf{6.64 $\pm$ 3.52}  \\ 
\multicolumn{1}{l}{\textit{\textbf{B=10}}} &
 $11.67 \pm 1.26$ &
 $11.53 \pm 0.71$ &
 \textbf{0.82 $\pm$ 0.53} &
 \textbf{0.96  $\pm$ 0.09} &
 83.3 $\pm$ 404.8 \\
\multicolumn{1}{l}{\textit{\textbf{B=100}}} &
 $23.36 \pm 2.08$ &
 $6.36 \pm 1.50$ &
 $0.96 \pm 0.87$ &
 $0.93 \pm 0.16$ &
 117.8 $\pm$ 443 \\ 
\multicolumn{1}{l}{\textit{\textbf{B=200}}} &
 23.27 $\pm$ 1.75 &
 \textbf{3.93 $\pm$ 0.07} &
 $0.99 \pm 0.87$ &
 $0.93 \pm 0.16$ &
 1381.9 $\pm$ 6905.5 \\ 
\bottomrule
\end{tabular}
%}
\caption{Performance of different models trained with defense in place. The loss is on a held-out validation set.
For Synthetic Data, loss is cross-entropy (CE). For Jigsaw Sudoku, loss is MSE in order of $\times10^{-4}$. For Speed Planning, loss is $\mathcal{L}_{\textit{total}}$.
%\vspace{0.1in}
%\textit{*Statistics only reported for successful runs; only 12.8\% of runs with undefended model for Speed Planning were successful.}
} 
\label{tab:eq_training_def}
\end{table*}

\noindent
\textbf{Experimental Setup:} To generate the input, we utilize 5 traffic sign classes of the BelgiumTS dataset \cite{Timofte-MVA-2011} for the images. These traffic signs require an immediate change in speed/acceleration (e.g. Stop, Yield). We combine the image with the current state of the vehicle $V = \{c, d, f\}$ at the decision point and generate 10000 training samples and 1000 test samples for use in our holdout set. We enforce the following constraints through the matrix $G, h$:  all acceleration is at most $14m/s^2$, and speed must be positive. We setup the network as in Fig.~\ref{fig:network_diagram_speedplanning} and employ the Adam optimizer \cite{DBLP:journals/corr/KingmaB14} with a learning rate of $\expnumber{1}{-4}$, run the experiments for 30 epochs, and average the result over 30 random seeds. The models were evaluated using loss functions which penalizes different aspects of the decision:  $\mathcal{L}_{\textit{safety}}$ loss due to impact of collision with vehicle ahead, $\mathcal{L}_{\textit{comfort}}$ loss due to discomfort from acceleration,
$\mathcal{L}_{\textit{distance}}$ loss due to slowing down and
$\mathcal{L}_{\textit{violation}}$ loss from violating the traffic sign. The $\mathcal{L}_{\textit{total}}$ loss which is a weighted sum of the above losses is reported (details in Appendix~\ref{appendix:additional_experiment_results}). For the attack, we apply a model-tuned learning rate and optimize for various attack losses for 500 epochs over 30 random images from the test set. 

\noindent
\textbf{Results:} Even in this restricted and complex setting where we only allow the attacker to modify the images and not the state of the vehicle, the test time attack success rate is still high for \AllZeroRowCol{} (see Table~\ref{tab:eq_attack_rate}). We note that this setting is analogous to the real world where attackers can easily control the environment but not the sensor inputs of the vehicle or the physical constraints of the car (the attacker has no control over $Gx \leq h$). We report the overall loss $\mathcal{L}_{\textit{total}}$ when training with defense in Table~\ref{tab:eq_training_def}.

\subsection{Discussion of Results}
\label{sec:results}

\begin{figure}[t]
\centering
\subfigure[Training Loss]{
\includegraphics[width=.48\columnwidth,height=2.6cm]{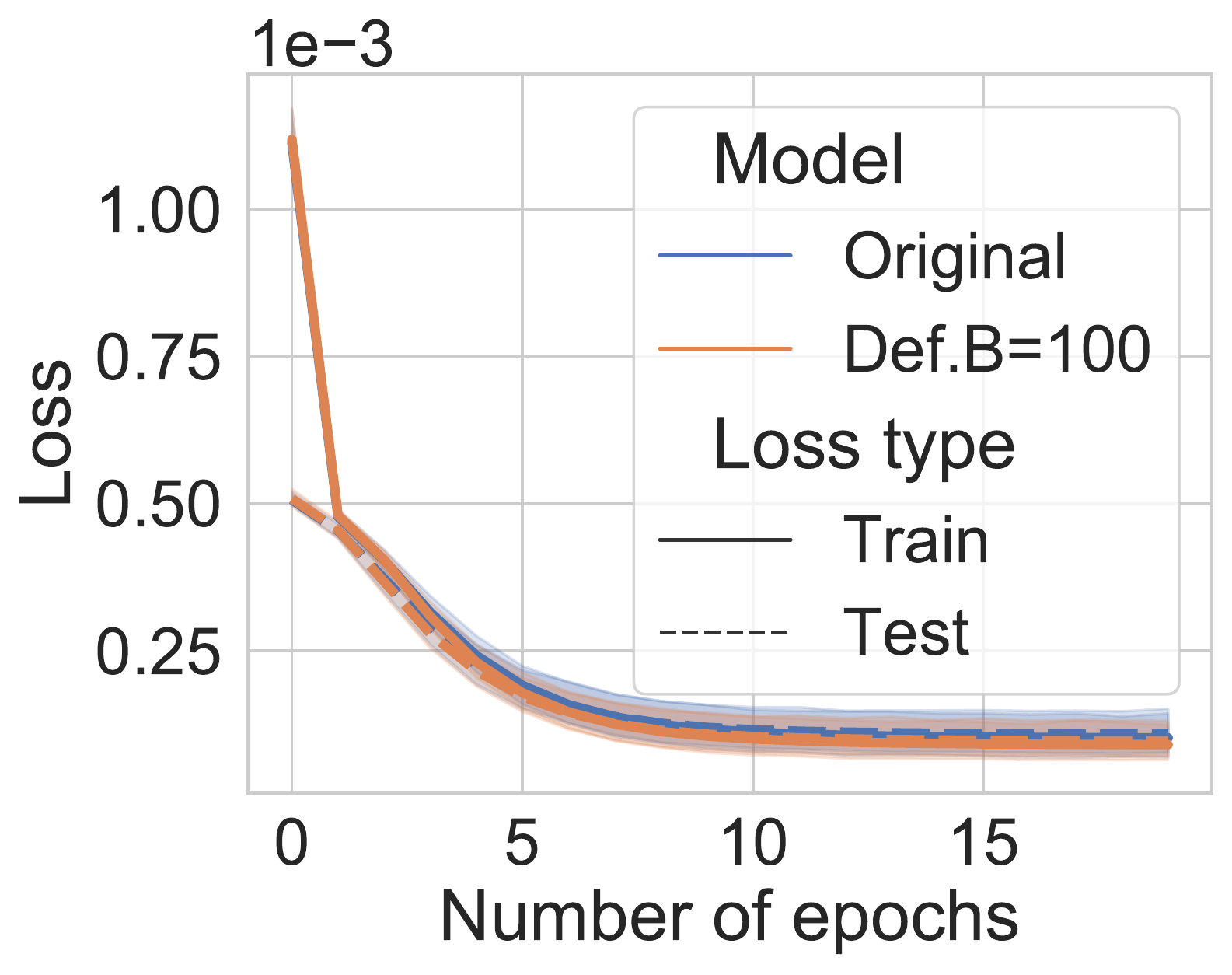}\label{fig:defensetraintest}
}%
\subfigure[MSE Loss in Attack]{
\includegraphics[width=.48\columnwidth,height=2.6cm]{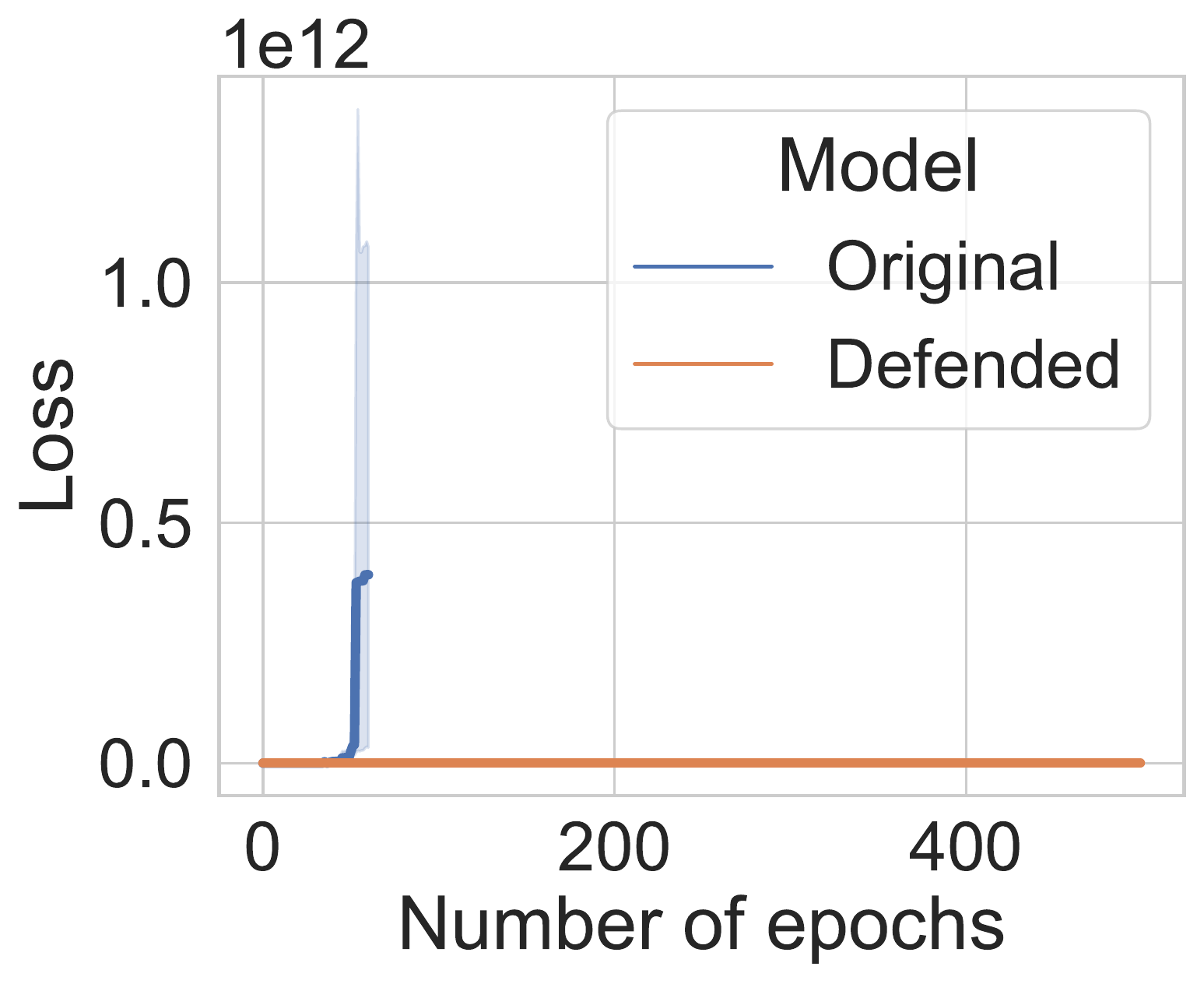}\label{fig:lossattack}
}%
\newline
\subfigure[Cond. Num. in Attack]{
\includegraphics[width=.47\columnwidth,height=2.6cm]{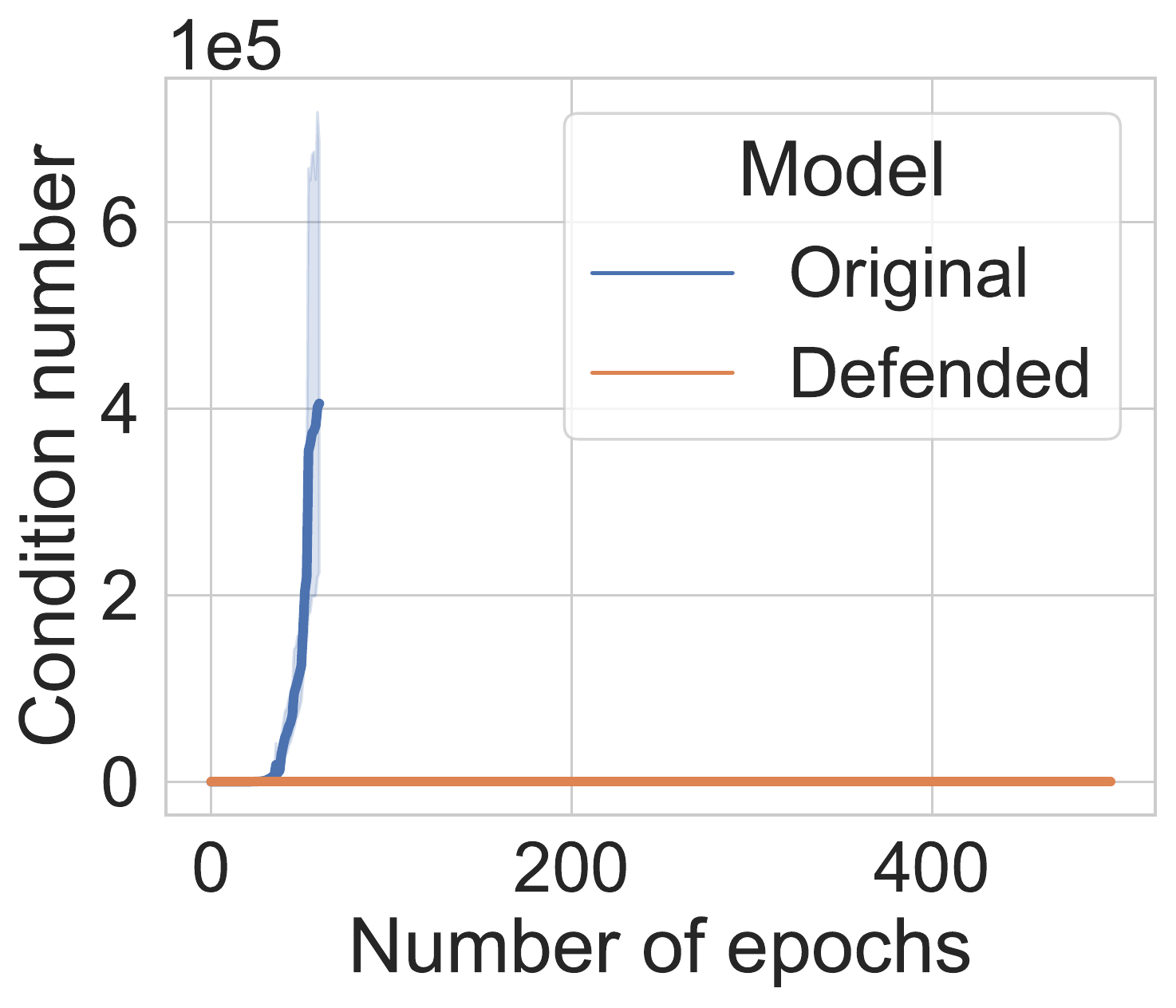}\label{fig:conditionattack}
}
\subfigure[Comparing Attacks]{
\includegraphics[width=.47\columnwidth,height=2.6cm]{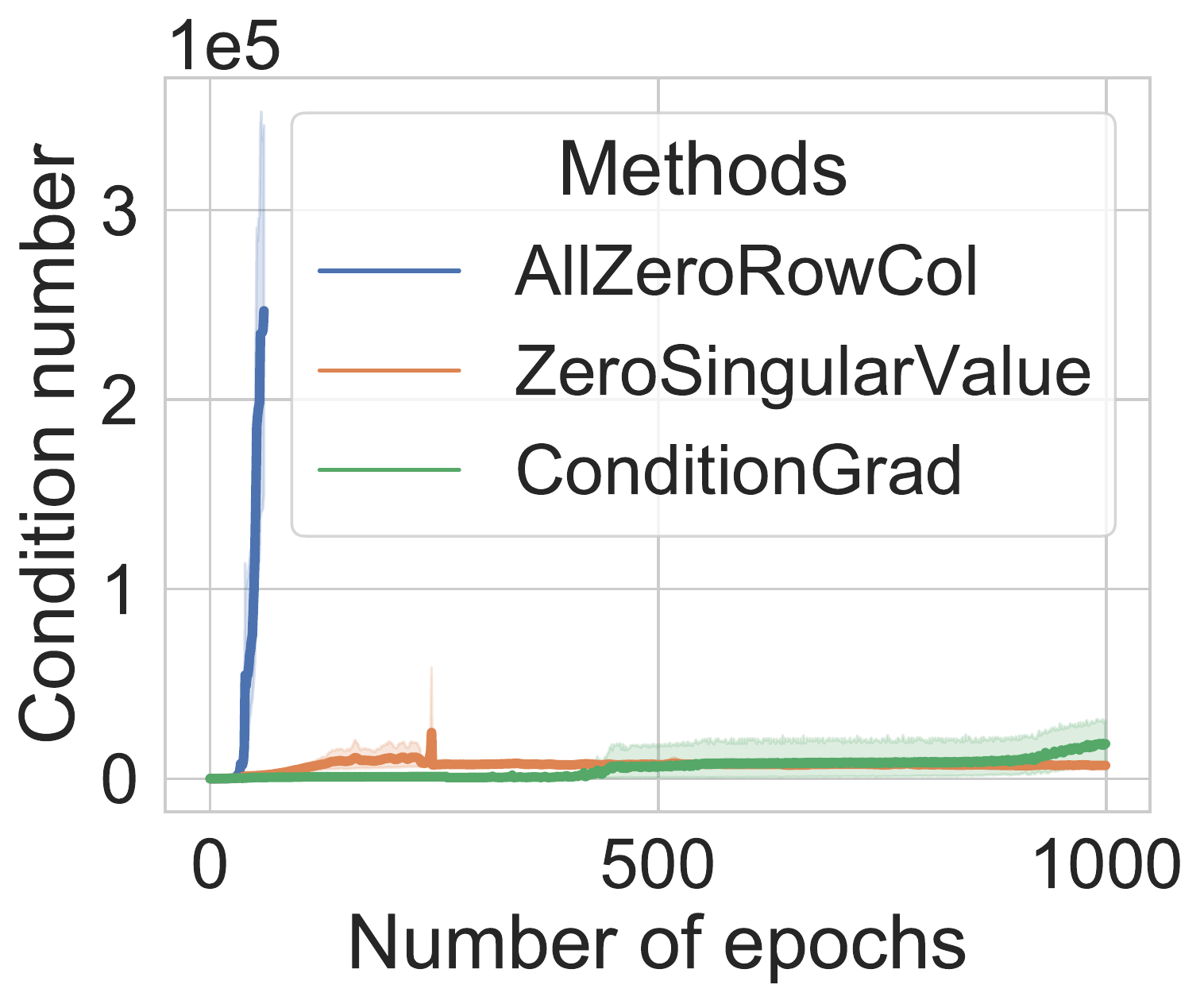}\label{fig:cond_saturation}
}
% \begin{figure}[H]
% \label{fig:cond_saturation}
% \includegraphics[width=\textwidth]{images/Condition_number_newplot_wo_img13.png}
% \caption{Graph}
% \end{figure}
%\subfigure{
%\includegraphics[width=.46\textwidth,%height=5cm]{images/Loss_eps.png}
%}
\caption{Effect of attacks and defense in Jigsaw Sudoku}%\thanh{I assume (a) is without attack. We need to clarify the difference between accuracy in (a) and loss in (b) and why we need both?}}
\label{fig:jigsaw_sudoku_defense_training}
\end{figure}

\noindent
\textbf{Efficacy of the attack methods:} The simplest \AllZeroRowCol{} attack was the most effective for all experiment settings (see Table~\ref{tab:eq_attack_rate}) for a range of activation functions---ReLU, CeLU, and also tanh (see Appendix~\ref{appendix:activation_functions}). This is surprising, given that the theoretically principled \ZeroSingularValue{} worked the least consistently empirically, and even the theoretically motivated \ConditionGrad{} attack worked inconsistently. This highlights the difficulty of 
%verifying and 
transferring theoretical results to the real world, especially 
%in the face of 
with limited numerical precision. 

 %We also observe that effective attacks can drastically affect loss (Fig.~\ref{fig:lossattack}), likely as a result of numerical instability (Fig. \ref{fig:conditionattack}).

\noindent
\textbf{Stabilizing effect of defense:} All attacks are thwarted by our defense, showing the effectiveness of controlling the condition number. We also observe that by tuning the $B$ hyperparameter, we are able to train with the defense without any tradeoffs in terms of learning or accuracy (Table~\ref{tab:eq_training_def}), and across all domains, we find that it adds less overhead (a constant factor less than 2) than the actual optimization. In fact, at certain values of $B$, we achieve lower loss and higher accuracy compared to the original undefended model. We conjecture that this occurs when 
%the condition number of the constraint matrix is controlled carefully, such that 
the true matrix exists within the space of the bounded condition number, and the low number makes the gradients of the optimization layer stable.

%For all domains, we augment the training of the model with our defense with minimal impact to the number of iterations to reach best performance (see Fig. \ref{fig:defensetraintest} for Jigsaw Sudoku, Appendix \ref{appendix:additional_experiment_results} for other domains). 

\noindent
\textbf{On achieving general trustworthiness:} Further auditing library functions, we noted several related issues which can be abused (some found by us, others by practitioners as benign flaws). Attackers can exploit these to produce solutions that violate the constraints of the optimization or even produce incorrect results when the input is a singular matrix (Appendix \ref{appendix:additional_attacks}). %In fact, the solver can be made to provide a solution that violates the constraints to the optimization without raising any error, which is an extremely dangerous bug---violating the positivity constraint on speed means that a vehicle may suddenly reverse and crash into a car waiting behind. 
%We believe that as machine learning systems become increasingly relied on for critical missions, we can no longer afford to treat the frameworks we use as mathematical black boxes. 
Careful audits should be performed on the implementation, down to potential edge cases in the data types.

\section{Conclusion}
\label{sec:discusion}

Our work scratches the surface of a new class of vulnerabilities that underlie neural networks, where rogue inputs trigger edge cases that are not handled in the underlying math or engineering of the layers, which lead to undefined behavior. We showed methods of constructing inputs that force singularity on the input matrix of equality constraints for optimization layers, and proposed a guaranteed defense via controlling the condition number. We hope that our work raises awareness of this new class of problems so the community can band together to resolve them with novel solutions.

\bibliography{icml22}
%\bibliographystyle{icml2022}

%%%%%%%%%%%%%%%%%%%%%%%%%%%%%%%%%%%%%%%%%%%%%%%%%%%%%%%%%%%%%%%%%%%%%%%%%%%%%%%
%%%%%%%%%%%%%%%%%%%%%%%%%%%%%%%%%%%%%%%%%%%%%%%%%%%%%%%%%%%%%%%%%%%%%%%%%%%%%%%
% APPENDIX
%%%%%%%%%%%%%%%%%%%%%%%%%%%%%%%%%%%%%%%%%%%%%%%%%%%%%%%%%%%%%%%%%%%%%%%%%%%%%%%
%%%%%%%%%%%%%%%%%%%%%%%%%%%%%%%%%%%%%%%%%%%%%%%%%%%%%%%%%%%%%%%%%%%%%%%%%%%%%%%
\newpage
\appendix
%\onecolumn
\section*{Appendix for Submission titled ``Beyond NaN: Resiliency of Optimization Layers in The Face of Infeasibility''\footnote{All reproduction code is available at \url{https://github.com/wongwaituck/attacking-optimization-layers-public}}}

\section{Bugs and Additional Attacks}
\label{appendix:additional_attacks}

\subsection{BUG: lu\_solve and Solve Failures for Singular Matrix}
\label{appendix:lu_solve}
\paragraph{Code:} Reproduction code can be found in the folder \texttt{attacks/lu\_solve\_singular}.

\paragraph{Background} Some optimization layer libraries (e.g. \texttt{qpth}) and network architectures (e.g. iUNet \cite{etmann2020iunets}) rely on fast equation solvers like \texttt{lu\_solve} to compute the forward pass. However, during our tests, we found that \texttt{lu\_solve} gives wildly incorrect results when a singular matrix is provided. This is a known issue in PyTorch (see: \cite{pytorch_2020_48572}), but this can be reproduced in TensorFlow as well.

The reproduction scripts can be found in \texttt{attack/lu\_solve\_singular} as part of the supplementary materials. This was tested on CPU and should give similar results in GPU as well. The versions of PyTorch used was 1.8.1 and tensorflow 2.5.0. \texttt{SingularSolvers.py} contains an attack for both frameworks, while \texttt{MinimalErrorProof.py} contains a minimal PoC based on the above pytorch issue filed.

\paragraph{Impact} If a rogue input is provided or an input just happens to satisfy the conditions above to the \texttt{lu\_solve}, it would result in unpredictable behavior in the output. In a classification problem, the impact would be a misclassification - which could be as serious as a misdiagnosis. If something more complicated is determined from the neural network (e.g. controlling an autonomous vehicle), the results can be life threatening.

\paragraph{Remediation} There's currently no consensus on how this should be resolved, but the expected behavior is to throw an error that can be caught by the caller.

\subsection{BUG: Inequality Violation Gives Wrong Result in \texttt{qpth}}
\label{sec:qptherrors}

\paragraph{Code:} Reproduction code can be found in the folder \texttt{attacks/inequality\_incorrect}.

\paragraph{Background} \texttt{qpth} is a library that provides a differentiable solver that plugs into pytorch so that a differentiable quadratic program solver can be included as part of a neural network. We first note that the quadratic program (in Eq.~\ref{eq:optnet}) allows for arbitrary $G$, $h$, including a $G$ and $h$ that is infeasible. In this case, no warning is produced and some value is outputted as  We note that we first observed this issue when observing the output of the Speed Profile Planning scenario, and we reproduce a minimum proof of concept here.

The reproduction scripts can be found in \texttt{attack/inequality\_incorrect} as part of the supplementary materials. The versions of pytorch used was 1.8.1. \texttt{constraint-inequality-marabou-\\qpth-adv.ipynb} contains the attack for the default solver in \texttt{qpth}, while \texttt{constraint-inequality-marabou-\\cvxpy.ipynb} contains a minimal proof of concept for the \texttt{cxvpy} solver - in this case the \texttt{cxvpy} clearly states that it is infeasible, but the \texttt{qpth} solver gives a completely incorrect result without any error. For the above notebooks, the attack sample was generated using the Marabou \cite{10.1007/978-3-030-25540-4_26} verification framework.

\paragraph{Impact}  The incorrect results can be used by downstream systems (for instance, in our speed profile planning scenario, the autonomous vehicle) which will lead to potentially disastrous outcomes since the output no longer satisfies the constraints of the quadratic program.

\paragraph{Remediation} Consider implementing an evaluation mode at test time which checks whether the problem is feasible and throws an error to the caller at test time instead of failing silently.

\subsection{Non-PSD Q}

\paragraph{Code:} Reproduction code can be found in the folder \texttt{attacks/non\_psd\_q}.

\paragraph{Background} Following the formulations detailed in \cite{amos2017optnet}, we note the optimization problem that is embedded in the layer is of the following form:

\begin{equation}
\label{eq:optnet}
\begin{split}
\minimize_{z} \;\; & \frac{1} {2}z^T Q z + q^T z \\
\subjectto \;\; & A z = b, \; G z \leq h
\end{split}
\end{equation}
where $z \in \mathbb{R}^n$ is our optimization variable
$Q \in \mathbb {R}^{n \times n} \succeq 0$
(a positive semidefinite matrix),
$q \in \mathbb {R}^n$, $A\in \mathbb{R}^{m \times n}$,
$b \in \mathbb{R}^m$,
$G \in \mathbb{R}^ {p \times n}$ and
$h \in \mathbb{R}^{p}$ are problem data. The problem is then solved by the method of Lagrange multipliers as shown below
\begin{equation}
    L(z,\nu,\lambda)=\frac{1}{2}z^TQz+q^Tz+\nu^T(Az-b)+\lambda^T(Gz-h)
\end{equation}
The gradient with respect to the parameters of the optimization problem are then derived using the standard implicit differentiation through the KKT conditions (the full derivation is available at \cite{amos2017optnet}). However, the same paper notes that it performs the following weight updates to $Q$.
\begin{equation}
 \frac{\delta \ell}{\delta Q} = \frac{1}{2}(d_z z^T + zd_z^T) 
\end{equation}

We note that the weight update operation to $Q$ in general may not respect convexity. For example, we can easily have positive weighted sums (so the weight updates when applied to Q will be negated and hence concave) which means convexity of $Q$ is no longer preserved. 

\paragraph{Impact} This results poor training, where loss starts to fluctuate when the PSD assumption no longer holds, and performance starts to vary wildly. The model is no longer able to converge correctly, and can be implemented as a training time attack. The attack \textbf{proof of concept} is available at \texttt{attacks/non\_psd\_q} as part of this supplementary materials package.

\paragraph{Remediation} This attack has been resolved in \cite{NEURIPS2019_9ce3c52f}, though it was not explicitly mentioned. They do so by converting it to a second-order conic program. We first note that the objective function of the quadratic program can be re-expressed in the following form
\begin{equation}
\begin{split}
\minimize_{z,k} \;\; & k \\
\subjectto \;\; &z^T P^T P x + q^T z \leq k.
\end{split}
\end{equation}
where $P^TP = Q$. We convert the above constraint to the following second order conic form
\begin{equation}\left\|
\begin{matrix}
(1 + q^T z - k)/2\\
Pz
\end{matrix} \right\|_2
\leq (1 - q^T z + k)/2 \end{equation}

The linear constraints can also be trivially reformulated, and hence will not shown here. Updates are then applied with respect to $P$, which means that convexity is preserved across gradient updates (since $P^T P$ is always positive semidefinite).

\subsection{Inequality Constraint Infeasibility Attack}

\paragraph{Code:} Reproduction code can be found in the folder \texttt{attacks/inequality\_attack}. \texttt{ineq\_feasibility.py} trains and runs the attack, \texttt{loss\_funcs.py} defines the loss function, and \texttt{optnet\_modules.py} defines the network. Instructions on how to run the scripts are in \texttt{README.md}. All results can be found under \texttt{data}. 

\paragraph{Background:} We wish to find a $u^*$ such that the matrix $A$ formed from $\theta = f_w(u^*)$ makes the constraints given by $A x \leq b$ infeasible. For this part, we assume that only $A$ depends on $\theta$ and $b$ is fixed; the attack is actually easier if $b$ also depends on $\theta$.
Farkas' Lemma is a well known result that characterizes feasibility of linear inequality constraints $Ax \leq b$, where $A \in \RR^{m\times n}$ and $b \in \RR^m$. Farkas' Lemma  states that $Ax \leq b$ has no solution if and only if $\exists$  $y \in \RR^m$ such that $y \geq 0$, $A^T y=0$ and $b^T y < 0$. Thus, finding an $A$ that makes $Ax \leq b$ infeasible is equivalent to finding an $A$ for which $\exists$  $y$ such that $y \geq 0$, $A^T y=0$ and $b^T y < 0$. We call the conditions in $y$ as the \emph{infeasibility condition}. 

In order to convert the infeasibility condition to an optimization form, first observe that WLOG $b^T y < 0$ can be written as $b^T y = -1$. This is WLOG because $y$ can be scaled by any positive number in the infeasibility condition and $b^T y < 0$ is a strict inequality. Given this observation, define two convex set $X = \{ y ~\vert~ A^Ty =0, b^Ty = -1\}$ and $Z = \{y~\vert~ y \geq 0\}$. Then, the infeasibility condition is equivalent to checking that that $X \cap Z$ is not empty.

The set $X$ comprises of the solutions of the linear equations
$
Ky =
\begin{pmatrix}
% & \vdots & \\
 A^T   \\
% & \vdots  &  \\
 b^T   \\
\end{pmatrix}
\begin{pmatrix}
 \\
y \\
 \\
\end{pmatrix}=\begin{pmatrix}
0_{n \times 1} \\
-1 \\
\end{pmatrix} 
= q
$. From standard theory of linear equations, the solutions to this system of equations is given by:%\thanh{we should use a different notation instead of $w$}
$$ 
X = \{K^{+}q + (I_{m \times m} - K^{+}K)v ~\vert~ v \in \RR^{m} \}
$$
and solutions exist if and only if $KK^{+}q = q$. To take the solution existence prerequisite into account, we define the prerequisite loss 
\begin{equation} \label{eq:prereq}
\mathcal{L}_{\textit{prereq}} = \norm{KK^{+}q - q}
\end{equation}
that should be minimized to zero.
%\begin{equation}
%    \mathcal{L}_{\textit{prereq}} = \norm{KK^{+}q - q}_2
%\end{equation}
%Now that we have taken into account this condition, we can In particular, the solutions are defined by the following set
%\begin{equation}  y = \{ K^{+}q + [I - K^{+}K]w \mathop {|} w \in \RR^{|y|} \}\end{equation}
Next, the easiest way to ensure that $X \cap Z$ is not empty is to drive the distance between the $X$ and $Z$ to zero, where the distance between $X, Z$ is the Euclidean norm between closest points of $X$ and $Z$. In order to achieve this, we define a loss $\mathcal{L}_{dist}$ defined as the solution value of the following quadratic convex optimization:
%\begin{equation*}
\begin{alignat}{3}
\mathcal{L}_{dist} \; \!=\! & \min_{y,v \in \mathbb{R}^m}  \norm{ y \!-\! K^{+}q \!-\! (I \!-\! K^{+}K)v}^2_2 \qquad (Optdist) \label{eq:dist}\\
\nonumber & \subjectto  \; y \geq 0  
\end{alignat}
%\end{equation*}
A solution value of zero for $Optdist$ ensures that $X \cap Z$ is not empty. This can be seen easily as the constraint forces feasible $y$'s to be exactly the set $Z$ and the objective minimizes the distance between any two points in $X$ and $Z$. 

Overall, the attack involves minimizing the loss
$ \gamma  \mathcal{L}_{prereq} + (1- \gamma ) \mathcal{L}_{dist}
$, where $\gamma$ is a hyper-parameter that is set closer to one (to ensure that the prerequisite is definitely zero). The derivative of $\mathcal{L}_{dist}$ can be obtained by differentiating through $Optdist$ using optimization layer techniques itself. We obtain the following closed form for required gradients and another result about $Optdist$. Before that, we rewrite the optimization in a standard form using the variable $z = [y - K^+ q, v] \in \mathbb{R}^{2m}$. Let 

\begin{equation}
B= \begin{pmatrix}
I_{m \times m} \quad   (K^{+}K - I_{m \times m}) \end{pmatrix}, A = \begin{pmatrix}
I_{m \times m} \quad   0_{m \times m}
\end{pmatrix}
\end{equation}
Then, it can be seen that $Bz = y - K^{+}q - (I - K^{+}K)v$ and $Az = y - K^+ q$ and then $Optdist$ can be written as
%\begin{equation}
\begin{alignat}{5}
\nonumber \mathcal{L}_{dist} = & \qquad \min_{z}   & \norm{Bz}^2_2  \qquad & \quad & \\
\nonumber & \subjectto \;\; & A z \geq  -K^{+}q & &  
\end{alignat}
\begin{lemma}
Matrix $B^TB$ has at least one zero eigenvalue, hence is non-invertible.
\end{lemma}
\begin{proof}[Proof of Lemma]
From definition of $B$, we get

\begin{equation}
\begin{split}
B^TB &=   \begin{pmatrix}
 I  \\
 (K^{+}K - I)^T\\
\end{pmatrix} \times 
\begin{pmatrix}
I  & K^{+}K - I  \\
\end{pmatrix} \\
& = \begin{pmatrix}
 I  & K^{+}K - I  \\
(K^{+}K - I)^T & (K^{+}K - I)^T (K^{+}K - I) \\
\end{pmatrix} \\
\end{split}
\end{equation}

For any real matrix $B^T B$ is always positive semi-definite. Denoting the above block matrix $B^T B$ as $\begin{pmatrix}
 A & X \\  
 X^T & C
\end{pmatrix}$, we apply a property of the Schur's complement of a block matrix, which states the following for a symmetric block matrix $Y = \begin{pmatrix}
 A & X \\  
 X^T & C
\end{pmatrix}$ :
if $A^{-1}$ exists then $det(Y) = det(A)det(C - X^T A^{-1} X)$, where $det$ denotes determinant.

We can apply it to the block matrix $B^TB$ since $A = I$, so
\begin{equation}
\begin{split}
det(B^TB) &=   det(I) det((K^{+}K - I)^T (K^{+}K - I) -  \\
& \quad (K^{+}K - I)^T I^{-1} (K^{+}K - I)) \\ 
&=  det((K^{+}K - I)^T (K^{+}K - I) - \\
& \quad (K^{+}K - I)^T (K^{+}K - I)) \\
&= 0 
\end{split}
\end{equation}
While $B^TB$ above is positive semi-definite by construction, it is only weakly so as one eigenvalue is zero (implied by zero determinant).
\end{proof}
In particular, the non-invertibility of square matrix $B^TB$ in the above result is a concern as the forward pass when solving $Optdist$ itself fails due to numerical instability. We avoid this using the numerical stability defense and we describe this in the next paragraph. Moreover, as the defense introduces approximation and our attack depends on exact result, we tighten the constraint of standard form $Optdist$ to $A z \geq -K^+q + \nu$ (equivalent to $y \geq \nu$ in $Optdist$) for small $\nu > 0$, which has the effect of shrinking the space $Z$ to $Z' \subset Z$. This ensures that even with approximation the solution found ($y$ in $Optdist$) is still very likely in the overlap of $X$ and $Z$.

\textbf{Fix for non-invertible $B^T B$}: ; Non-invertible $B^T B = \begin{pmatrix}
 A & X \\  
 X^T & C
\end{pmatrix}$ in practice results in severe numerical instability when trying to solve the optimization problem. To alleviate the issue, we add some small perturbation $Q_\eta = \eta I$, where $\eta \geq  0$ is a hyperparameter, so that the block matrix corresponding to $C$ becomes more positive. For simplicity sake we assume that entries of $Q_\eta$ corresponding to $A$ in the block matrix are set to 0. We can easily see that this small addition leads us to a stronger conditions on the above equation, as follows:

\begin{equation}
\begin{split}
det(B^TB + Q_\eta) &=  det\Big((\eta I + (K^{+}K - I)^T (K^{+}K - I)) \\
& \quad -  (K^{+}K - I)^T I^{-1} (K^{+}K - I) \Big) \\ 
&=  det\Big((\eta I + (K^{+}K - I)^T (K^{+}K - I)) \\
& \quad - (K^{+}K - I)^T (K^{+}K - I)\Big) \\
&= det(\eta I) = \eta > 0.\\
& \Rightarrow B^TB \mbox{ is positive definite}
\end{split}
\end{equation}

Let $z^{*}(x)$ be the minimizer of the above program for a given input $x$. We create a poisoned example $x'$ from $T$ iterations of gradient descent using a learning rate $\alpha$ starting from a benign input $x$, by taking the gradient of the loss function above with respect to the input at the previous iteration as such
\begin{equation}
\begin{split}
x_{t+1} & = x_t - \alpha \nabla_{x_t}(\gamma\mathcal{L}_{\textit{prereq}}(x_t) + (1-\gamma)\mathcal{L}_{\textit{dist}}(x_t)) \\
& = x_t - \alpha \nabla_{x_t}(\gamma \norm{K(x_t)K(x_t)^{+}q - q}_2 \\
& + (1-\gamma)(z^{*}(x_t))^\top B(x_t)^\top B(x_t) z^{*}(x_t))
\end{split}
\end{equation}

where  $\gamma$ is a hyperparameter for the weights on the two loss functions $\mathcal{L}_{\textit{prereq}}$ and $\mathcal{L}_{\textit{dist}}$.

\paragraph{Impact:} We demonstrate the attack, running gradient descent using the loss function defined above, using synthetic data for $G \in \RR^{2\times 2}$  and show that the loss function does indeed find inequalities that lead to infeasibility. We plot how the inequalities change over time while under attack in Fig. \ref{fig:ineq_attack}.

For larger matrices, because we can't use the default \texttt{qpth} solver implementation which is faster (it already has issues with inequality, see Section \ref{sec:qptherrors}), we have to fall back to the alternative solver. In this case, attacking is much slower. Further, on large matrices, the forward pass in the alternative solver may occasionally either already face issues without being attacked, or face stability issues while being attacked. We leave further exploration of this as well as remediation for this issue for future work. 
\begin{figure}[ht]
\centering
\subfigure{
\includegraphics[width=0.45\columnwidth]{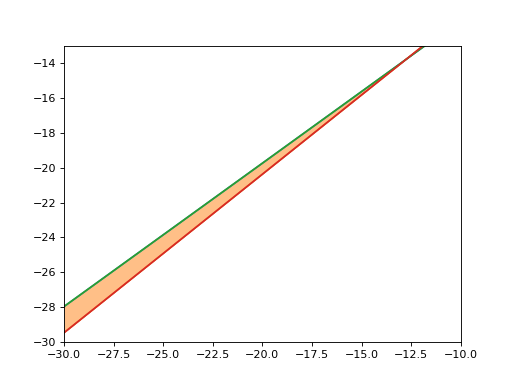}
}
\subfigure{
\includegraphics[width=.45\columnwidth]{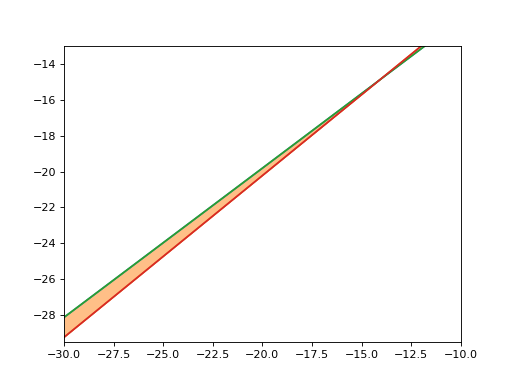}
}
\subfigure{
\includegraphics[width=.45\columnwidth]{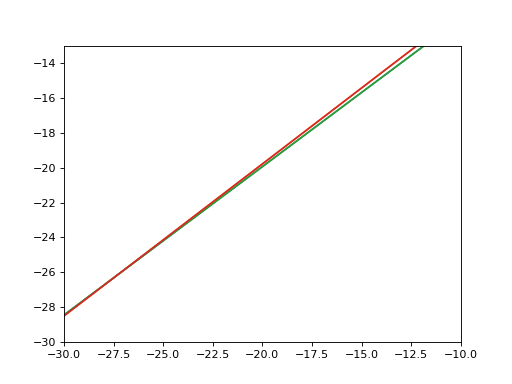}
}
\subfigure{
\includegraphics[width=.45\columnwidth]{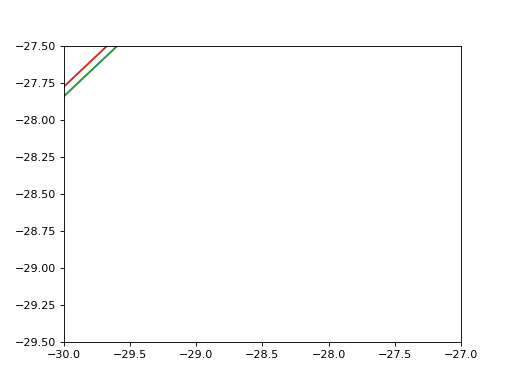}
}
\caption{Graph of inequalities as they are attacked. We plot the inequalities chosen from the last 20 iterations of the attack that show the greatest change, as well as the last iteration which resulted in the infeasibility. Figures should be interpreted left to right, starting from the top row. Highlighted regions indicate the feasible regions. The last figure on the bottom right shows that the optimization is no longer feasible as the inequalities are no longer intersecting, as needed.}
\label{fig:ineq_attack}
\end{figure}

\section{Additional Experimental Results}
\label{appendix:additional_experiment_results}

\subsection{Synthetic Data}

\paragraph{Code:} Reproduction code can be found in the folder \texttt{synthetic\_data}, where \texttt{eq\_train\_all\_cond.py} trains the model given the parameters, and \texttt{eq\_attack\_all.py} runs all attack methods on a given model. Instructions on how to run the scripts are in \texttt{README.md}. All results can be found under \texttt{results\_data}. We show relevant plots for training and attacks in the results below.

\subsubsection{Network Description}

In all the experiments, we have a network of 2 fully-connected layers of 500 nodes each with ReLU activations, followed by the OptNet optimization layer. 

\subsubsection{40x50 Results}

\begin{figure}[ht]
\centering
\includegraphics[width=0.60\columnwidth]{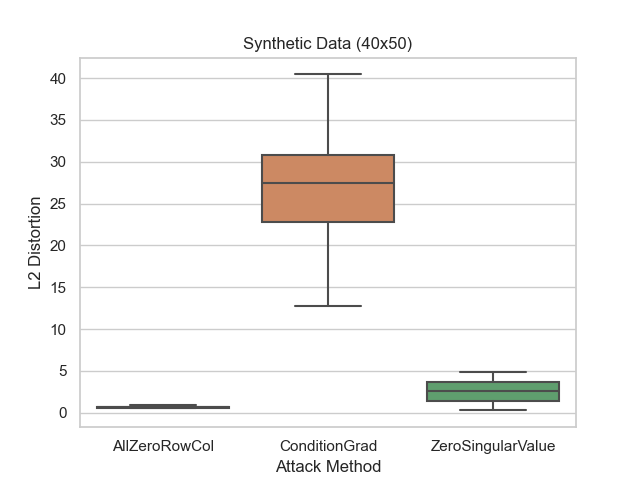}
\caption{Synthetic Data (40x50) -- L2 distortion measure as proposed in \cite{DBLP:journals/corr/SzegedyZSBEGF13} for attack perturbations, which is a normalized form of the L2 magnitude of attack and is defined as $\sqrt{\sum{\frac{(x_i - x^\prime_{i})^2}{n}}}$, where $n$ is 500 for the Synthetic Data test case.}
\label{fig:40x50distortionplot}
\end{figure}

\begin{figure}[ht]
\includegraphics[width=0.95\columnwidth]{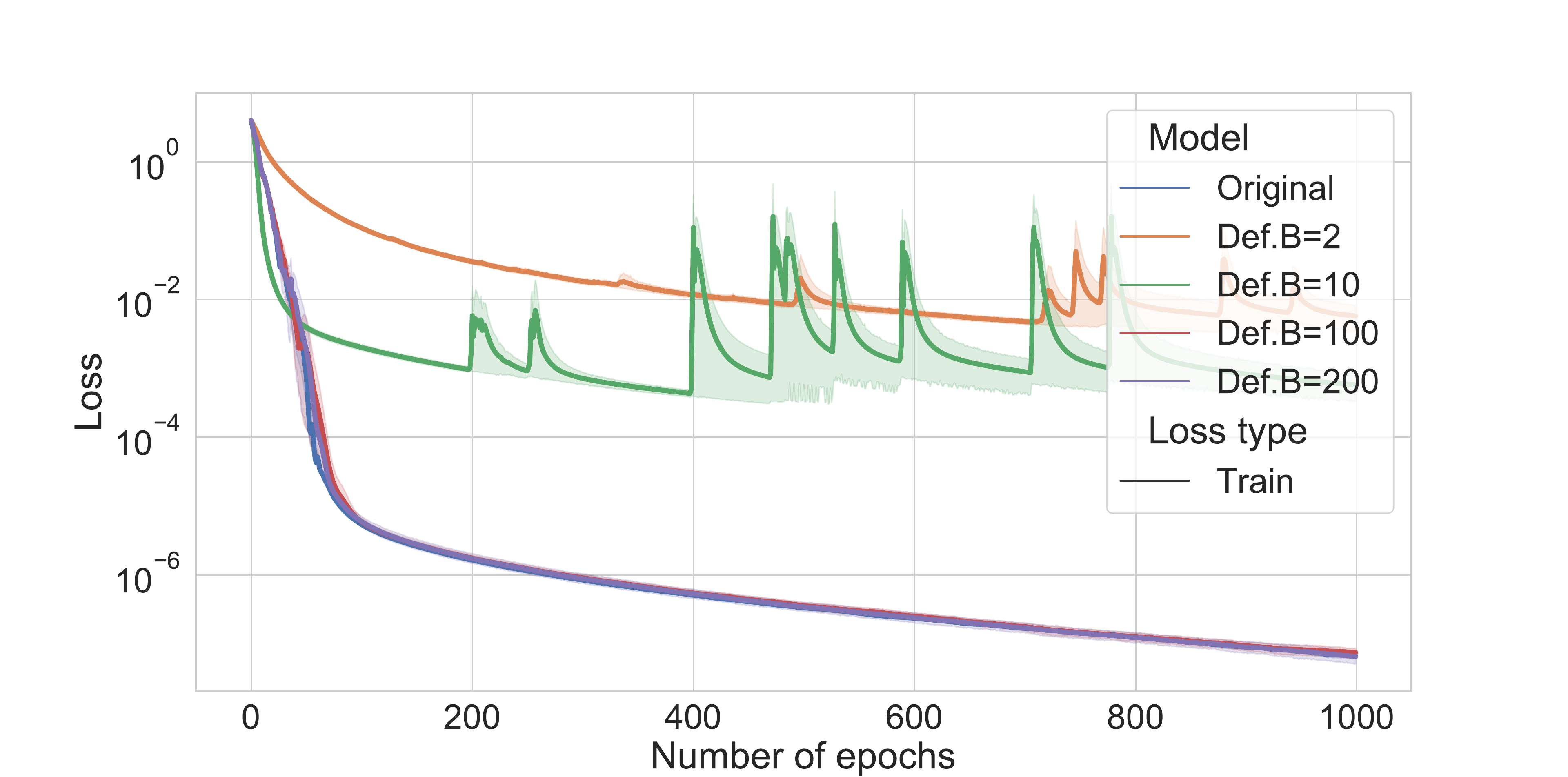}
\caption{Synthetic Data (40x50) -- Effect of defense on training averaged over 10 runs. The y-axis follows a log scale, and we see for large values of B there is virtually no difference in training compared to the original model.}
\label{fig:40x50trainplot}
\end{figure}

\begin{figure}[ht]
\centering
\includegraphics[width=0.85\columnwidth]{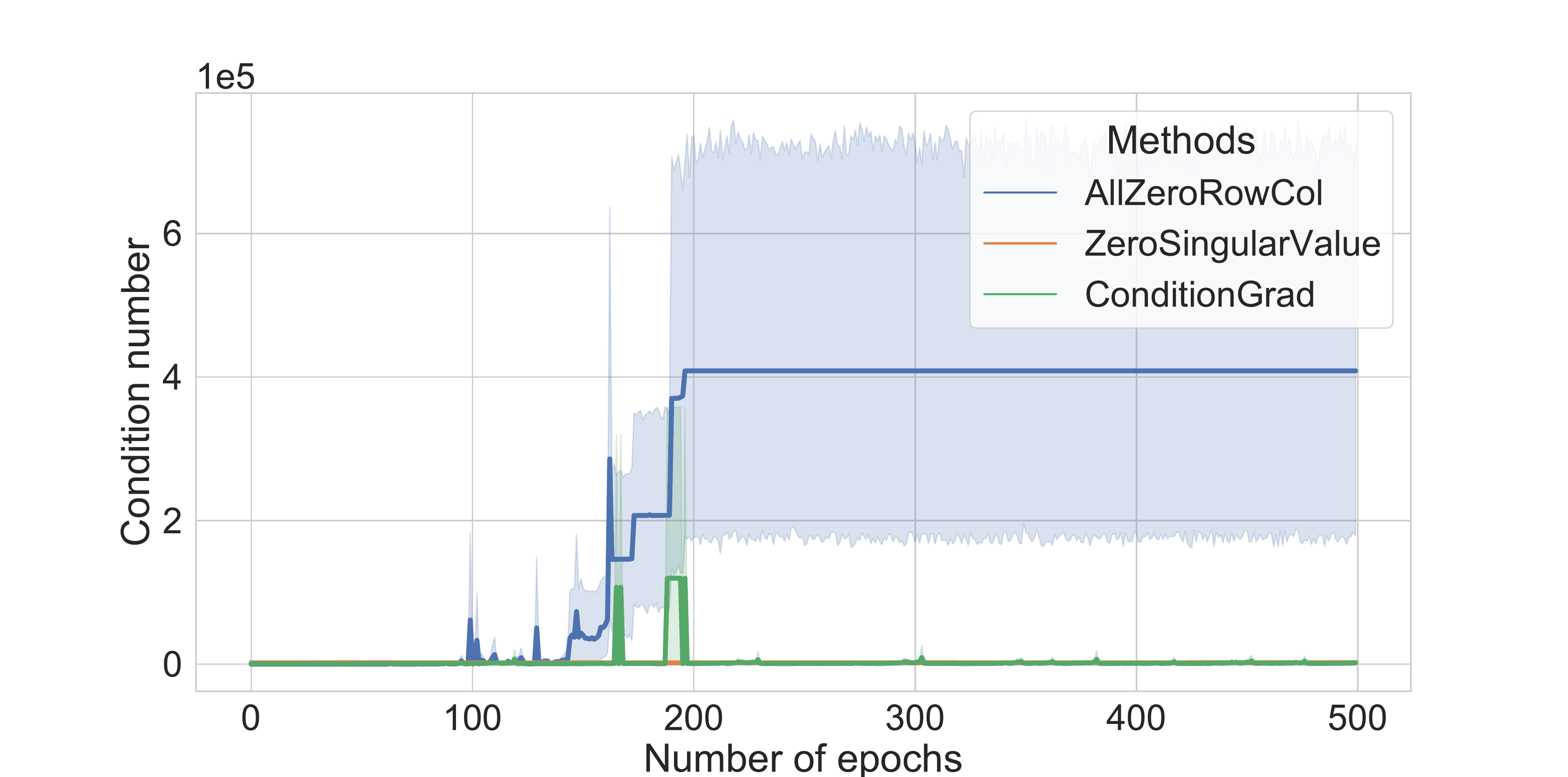}
\caption{Synthetic Data (40x50) -- Effect of attack methods on condition number averaged over 30 samples for the undefended model. For \AllZeroRowCol{}, condition number in later epochs take on the last value of condition number if an attack was found, hence the flat line.}
\label{fig:40x50attackundefplot}
\end{figure}

\begin{figure}[ht]
\centering
\includegraphics[width=0.85\columnwidth]{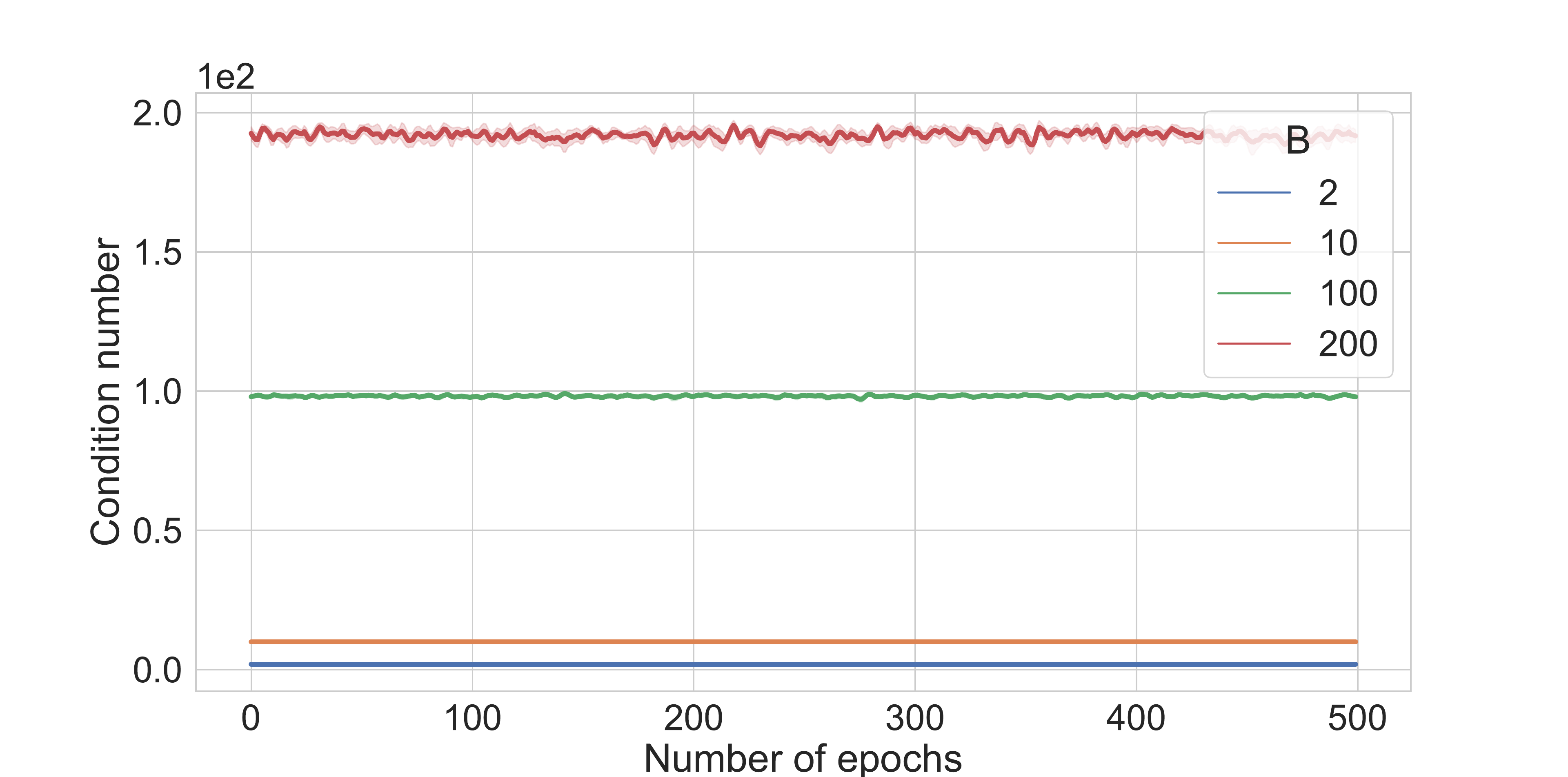}
\caption{Synthetic Data (40x50) -- Effect of attack on condition number averaged over 30 samples for the \emph{defended} model, using \AllZeroRowCol{}. We see the defense is highly effective in controlling the condition number, hence preventing any attack.}
\label{fig:40x50attackdefplot}
\end{figure}

The overall results in the synthetic data for 40x50 are congruent to our findings in the main paper. In Fig. \ref{fig:40x50trainplot}, we see that the training with the defense can be as good as the original undefended model with some fine tuning of the parameter $B$. We see that the most effective attack on the 40x50 Synthetic Data is \AllZeroRowCol{} in Fig. \ref{fig:40x50attackundefplot}. Finally, we see that the defense is effective in preventing the attack, and the condition numbers never exceed our threshold while under attack in Fig. \ref{fig:40x50attackdefplot}.

\begin{figure}[ht]
\centering
\includegraphics[width=0.60\columnwidth]{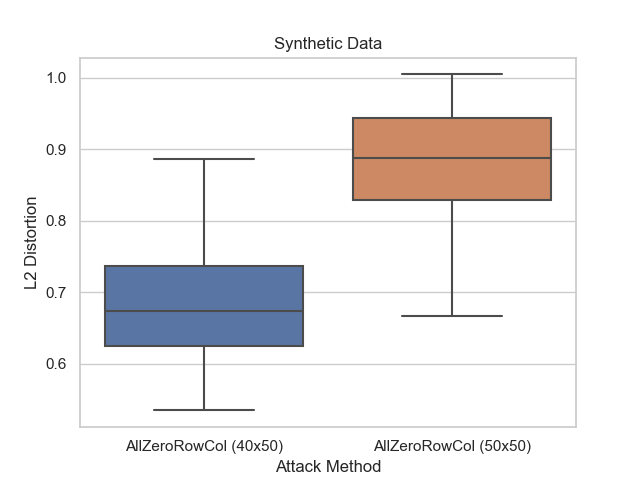}
\caption{Synthetic Data -- Comparison of L2 distortion measure as proposed in \cite{DBLP:journals/corr/SzegedyZSBEGF13}, which is a normalized form of the L2 magnitude of attack and is defined as $\sqrt{\sum{\frac{(x_i - x^\prime_{i})^2}{n}}}$ for attack perturbations, where $n$ is 500 for the Synthetic Data test case. We omit \ZeroSingularValue{} and \ConditionGrad{} attack data as there were not enough attack samples.}
\label{fig:50x50distortionplot}
\end{figure}

\subsubsection{50x50 Results}

The overall results in the synthetic data for 50x50 are also congruent to our findings in the main paper. In Fig. \ref{fig:50x50trainplot}, we see that the training with the defense can be as better as the original undefended model with some fine tuning of the parameter $B$. We see that the most effective attack on the 50x50 Synthetic Data is also \AllZeroRowCol{} in Fig. \ref{fig:50x50attackundef}. Finally, we see that the defense is effective in preventing the attack, and the condition numbers never exceed our threshold while under attack in Fig. \ref{fig:50x50attackdef}.

\begin{figure}[ht]
\centering
\includegraphics[width=0.95\columnwidth]{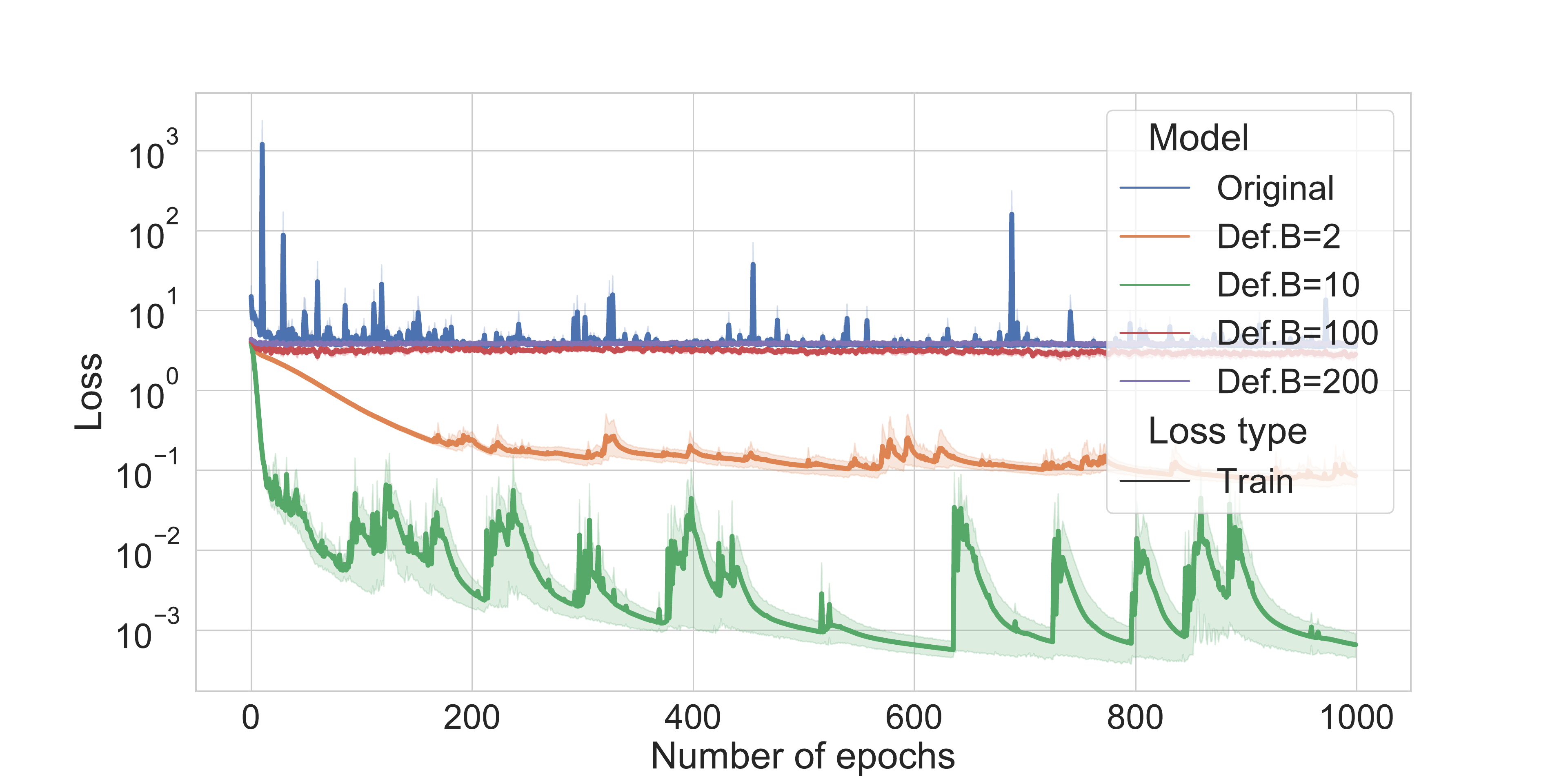}
\caption{Synthetic Data (50x50) -- Effect of defense on training averaged over 10 runs.}
\label{fig:50x50trainplot}
\end{figure}

\begin{figure}[ht]
\centering
\includegraphics[width=0.95\columnwidth]{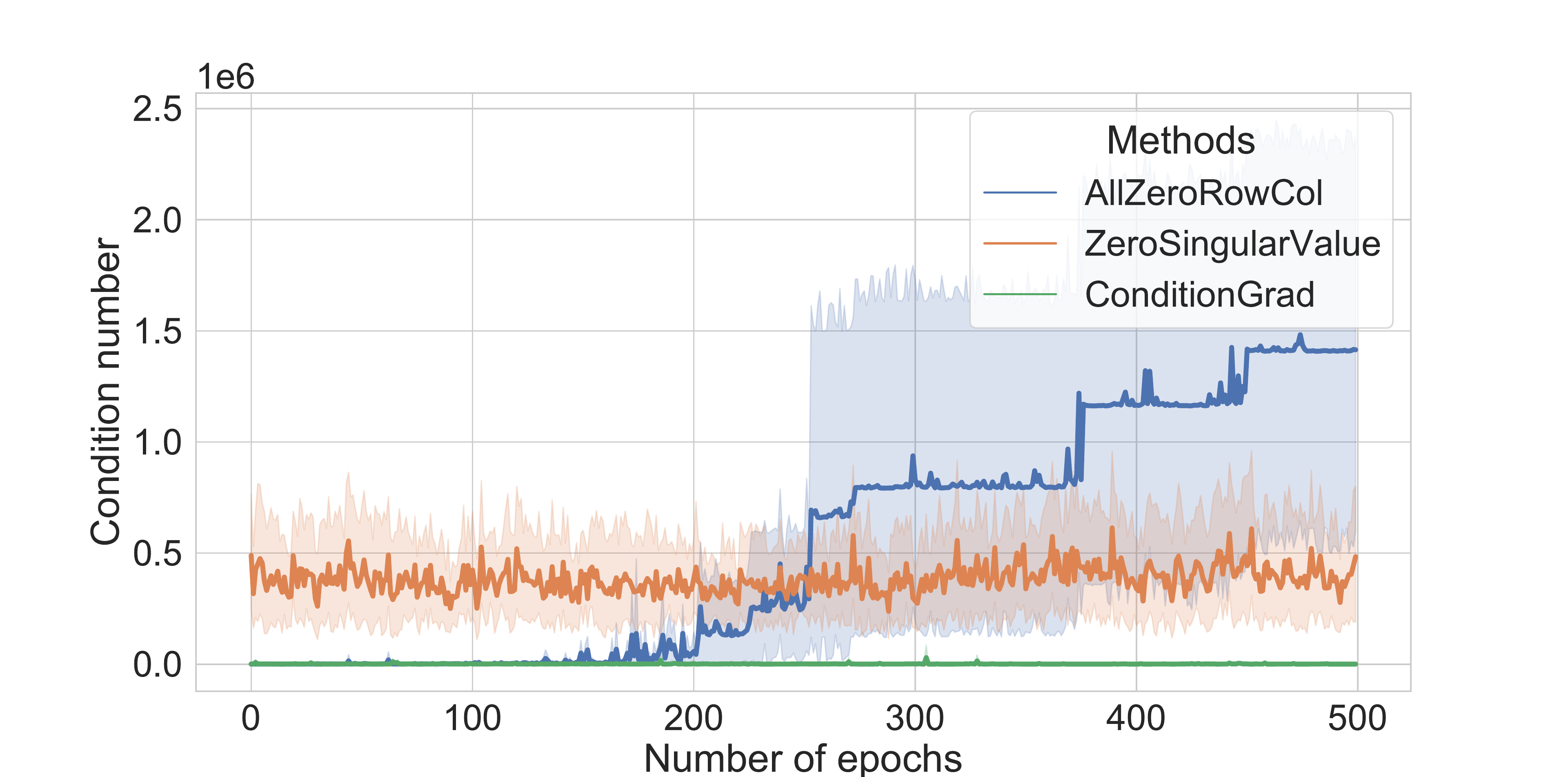}
\caption{Synthetic Data (50x50) -- Effect of attack methods on condition number averaged over 30 samples on the undefended model. Again, we see \AllZeroRowCol{} having the most success in this domain as well.}
\label{fig:50x50attackundef}
\end{figure}

\begin{figure}[ht]
\centering
\includegraphics[width=0.95\columnwidth]{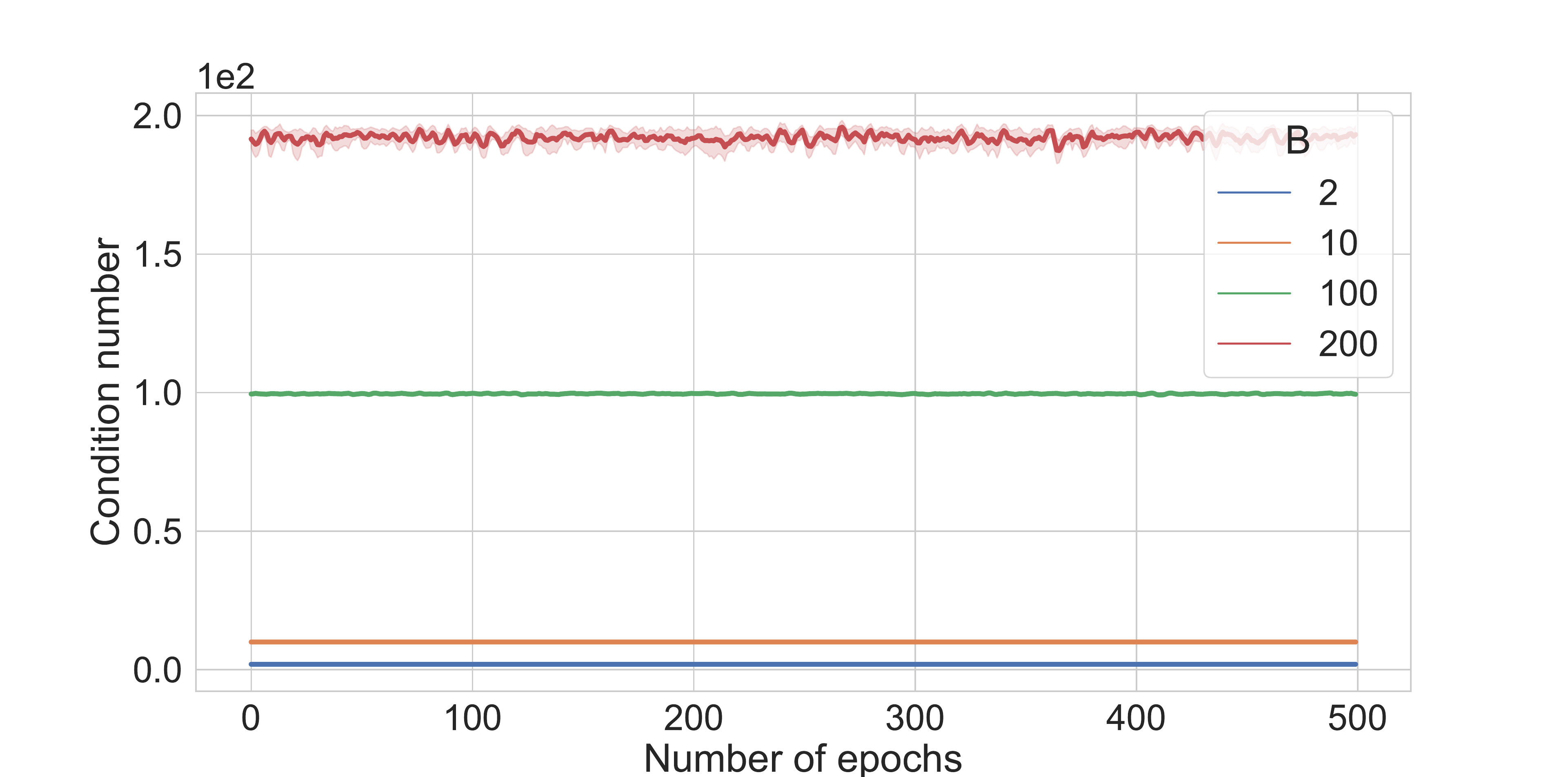}
\caption{Synthetic Data (50x50) -- Effect of attack methods on condition number averaged over 30 samples for the defended model. The input matrix $A$'s condition number is kept controlled at a level below the bound $B$, thus preventing failures in the model.}
\label{fig:50x50attackdef}
\end{figure}

The above models seem to train as well or better when the defense is applied. Note that the best performing models for train loss did not result in the best performing models for test loss at the end of 1000 iterations - this is likely due to over-fitting. The two matrices differ in being nonsquare
and square matrix (non-square $m = 40$ has fewer
rows than $m=50$) and the spread of condition numbers of
the randomly generated ground truth matrix differs with $|n-
m|$ (see~\cite{chen2005condition}), which we believe
makes for different behavior with varying B for the two different sized matrices (Fig.~\ref{fig:40x50trainplot} and ~\ref{fig:50x50trainplot}). The interactions are
also likely complex when the condition number is bounded
too small, e.g., for $m = 50$, the $B = 2$ case may benefit
from more stable gradients but the $B=10$ case might have
neither stable gradients nor the ability to be close to the true
matrix, and hence performed worse.

\subsection{Jigsaw Sudoku}
\label{appendix:jigsaw_sudoku}

\paragraph{Code:} Reproduction code can be found in the folder \texttt{jigsaw\_sudoku}. \texttt{train\_nodef.py} and \texttt{train\_defense.py} trains the undefended model and defended model respectively, given the parameters, and \texttt{attack\_num\_cond\_fn.py} runs all attack methods on a given model. Instructions on how to run the scripts are in \texttt{README.md}. We include the script used to create the dataset in \texttt{create.py} which adapts the MNIST dataset \cite{lecun-mnisthandwrittendigit-2010} and tiles them to form a puzzle using an modified script based on \cite{amos2017optnet}. All results can be found under \texttt{results\_data}. 

\begin{figure}[ht]
\centering
\subfigure{
\includegraphics[width=.21\columnwidth]{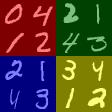}
}
\subfigure{
\includegraphics[width=.21\columnwidth]{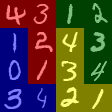}
}
\subfigure{
\includegraphics[width=.21\columnwidth]{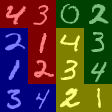}
}
\subfigure{
\includegraphics[width=.21\columnwidth]{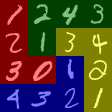}
}
\caption{Jigsaw Sudoku puzzle dataset. Images generated by randomly selecting corresponding digits from the MNIST dataset. The zeroes represent the cell that needs to be filled. }
\label{fig:jigsaw_sudoku_dataset}
\end{figure}

% \begin{figure}[t]
% \centering
% \subfigure{
% \includegraphics[width=0.4\columnwidth]{images/jigsaw_sudoku_images/atk_orig.png}
% }
% \hspace{20em}% Space between image A and B
% \subfigure{
% \includegraphics[width=.30\columnwidth]{images/jigsaw_sudoku_images/adv_ex6_ConditionNumberLoss_0.jpg}
% }
% \subfigure{
% \includegraphics[width=.30\columnwidth]{images/jigsaw_sudoku_images/adv_ex6_ZeroSingularValue_0.jpg}
% }
% \subfigure{
% \includegraphics[width=.30\columnwidth]{images/jigsaw_sudoku_images/adv_ex6_ZeroRowLoss_0.0.jpg}
% }
% \caption{\textbf{Jigsaw Sudoku Attack Images.} Top row shows the original images. Bottom row shows the images produced after running the attack successfully for CondGrad, ZeroSingularValue and ZeroRowCol respectively. We see that ZeroSingularValue looks the closest to the original image, as predicted by our theoretical discussion in our attack Section in the main paper.}
% \label{fig:jigsaw_sudoku_attackimages}
% \end{figure}

\subsubsection{Overall Results}
Graphs are generated from 30 random seeds. Results found in the main paper (on condition number, effect on attacks) are not reproduced here. We showcase the convergence rates of training with the defense and without in Fig. \ref{fig:additional_plots}.  

\begin{figure}[ht]
\centering
\includegraphics[width=0.60\columnwidth]{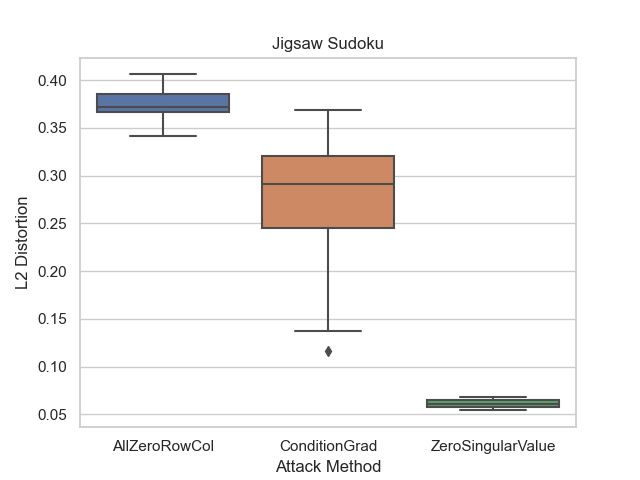}
\caption{Jigsaw Sudoku -- Comparison of L2 distortion measure as proposed in \cite{DBLP:journals/corr/SzegedyZSBEGF13}, which is a normalized form of the L2 magnitude of attack and is defined as $\sqrt{\sum{\frac{(x_i - x^\prime_{i})^2}{n}}}$ for attack perturbations, where $n$ is 37632 for the Jigsaw Sudoku case. We see that \ZeroSingularValue{} produces the smallest perturbations (in terms of magnitude).}
\label{fig:sudokudistortionplot}
\end{figure}

\begin{figure}[ht]
\centering
\subfigure[Train and test accuracy]{
\includegraphics[width=.45\columnwidth]{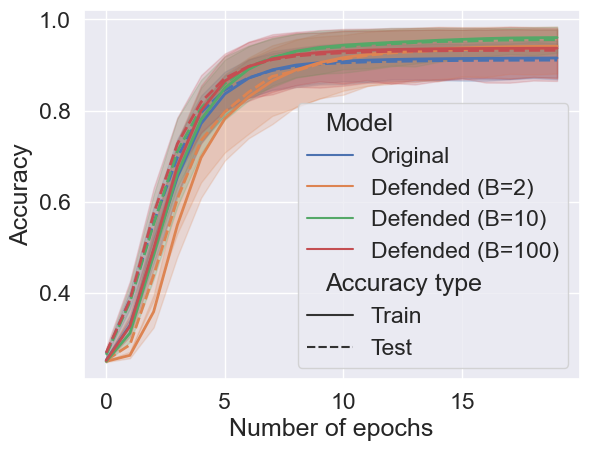}\label{fig:defensetraintest_OLD}
}%
\subfigure[Train and test loss]{
\includegraphics[width=.45\columnwidth]{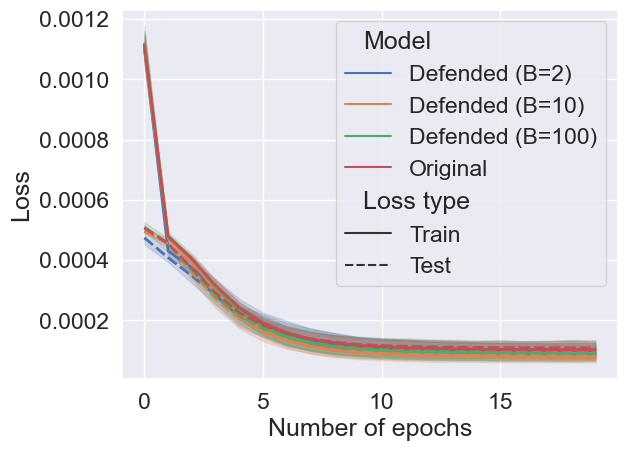}
}%
\caption{Effect of defense in the Jigsaw Sudoku setting. We see that in all the settings, training and test convergence rates for all defense configurations are similar. For more details on how the models performed empirically, please refer to the results in the main paper.}
\label{fig:additional_plots}
\end{figure}

\subsection{Speed Profile Planning}

\paragraph{Code:} Reproduction code can be found under the folder \texttt{quadratic\_speed\_planning}. \texttt{models.py} trains the undefended model and defended models, given the parameters, and \texttt{attack.py} runs all attack methods on a given model. Instructions on how to run the scripts are in \texttt{README.md}. All results can be found under \texttt{results}. We include the script used to create the dataset in \texttt{create.py} which adapts the a subset of the BelgiumTS dataset \cite{Timofte-MVA-2011} and augments it with a randomized state of the autonomous vehicle.

\begin{figure}[ht]
\centering
\subfigure{
\includegraphics[width=.1\columnwidth]{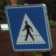}
}
\subfigure{
\includegraphics[width=.1\columnwidth]{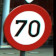}
}
\subfigure{
\includegraphics[width=.1\columnwidth]{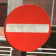}
}
\subfigure{
\includegraphics[width=.1\columnwidth]{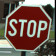}
}
\subfigure{
\includegraphics[width=.1\columnwidth]{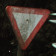}
}
\caption{Speed profile planning images. We sample images from the BelgiumTrafficSign dataset from the following classes shown above: \emph{Pedestrian Crossing, Speed Limit, No Entry, Stop,} and \emph{Yield}. }
\label{fig:speed_profile_dataset}
\end{figure}

% \begin{figure}[t]
% \centering
% \subfigure{
% \includegraphics[width=0.4\columnwidth]{images/speed_profile_planning/original.png}
% }
% \hspace{20em}% Space between image A and B
% \subfigure{
% \includegraphics[width=.30\columnwidth]{images/speed_profile_planning/adv_ex0_ConditionNumberLoss_0.jpg}
% }
% \subfigure{
% \includegraphics[width=.30\columnwidth]{images/speed_profile_planning/adv_ex0_ZeroSingularValue_0.jpg}
% }
% \subfigure{
% \includegraphics[width=.30\columnwidth]{images/speed_profile_planning/adv_ex0_ZeroRowLoss_0.jpg}
% }
% \caption{\textbf{Speed Profile Planning Attack Images.} Top row shows the original images. Bottom row shows the images produced after running the attack successfully for \ConditionGrad{}, \ZeroSingularValue{} and \AllZeroRowCol{} respectively. The model infers decisions taking into account both internal state and external sensors, and in this restricted attack model, we only control the outside image and none of the internal state. \ConditionGrad{} looks to be a plausible image that can occur in the real world. }
% \label{fig:jigsaw_sudoku_attackimages}
% \end{figure}

\subsubsection{Loss Functions}

Let $V = \{c, d, f\}$, where $c, d, f \in \RR_{\geq 0}$. Here, $c$ is the current speed of the vehicle (in meters per hour), $d$ is the distance to the destination (in meters), and $f$ is the distance to the vehicle ahead (in meters). Let $z$ be the class of the traffic sign. The loss functions are defined based on aspects of the decision on acceleration $a$ and target speed $s$ that we want to minimize, and a weighted sum over the loss functions $\mathcal{L}_{\textit{total}}$ is used as the final loss.

$\mathcal{L}_{\textit{safety}}$, loss due to impact of collision with vehicle ahead. We define it as safe (i.e. loss of 0) if our acceleration and current speed does not cause us to collide with the car ahead within 2 seconds. As a simplifying assumption. we assume the car ahead is travelling as the same speed as our car. This loss is defined as
\begin{equation}
\mathcal{L}_{\textit{safety}} = max(0, 2a - f)
\end{equation}

$\mathcal{L}_{\textit{comfort}}$, loss due to discomfort from acceleration is defined similarly as in prior work \cite{temp_Speedprofileplanning} as some threshold above a maximum comfortable threshold $\hat{a}$, where
\begin{equation}
\mathcal{L}_{\textit{comfort}}= max(0, |a| - \hat{a})
\end{equation}
In our work, we use the same threshold as in \cite{temp_Speedprofileplanning} as $2.5m/s^2$.

$\mathcal{L}_{\textit{distance}}$ loss due to slowing down and taking a longer time to reach the destination. It is defined as the difference in time needed to reach the destination given the new target speed, plus the time needed to accelerate to the new speed. We avoid dividing by zero  by clamping to small values.
\begin{equation}
\mathcal{L}_{\textit{distance}}= -(d/max(c-s,-1)) + max(s-c, 1)/max(a, 1)
\end{equation}

$\mathcal{L}_{\textit{violation}}$ loss due to traffic violation is when our target speed exceeds the stipulated speeds in size $z$. We penalize any speed exceeding the stipulated limits, as below 
\begin{equation}
\mathcal{L}_{\textit{violation}}=max(0, s-l_z)
\end{equation}
where $l_z$ is $0km/h, 0km/h, 10km/h, 20km/h, 50km/h$ where $z$ is \texttt{Stop}, \texttt{No Entry}, \texttt{Pedestrian Crossing}, \texttt{Yield} and  \texttt{Speed Limit} respectively.

\subsubsection{Results}

We run the training with the following parameters. We set $w_a = 1, p=0.5, w_s=1000$ for all our experiments .We weigh safety and traffic violations higher as we deem them more serious, so we use the weights $k_1 = 100, k_2=0.1, k_3=\expnumber{1}{-3}, k_4=10$ in our experiments. We randomly generate data such that $c \leq 100000$, $d \leq 10000$, and $f \in [0, d]$.

We initially experimented with ReLU activations, but found that the undefended model was not able to train in a stable way in this setting because ReLU often returned 0 as the gradient during backpropagation leading to higher risk of divide by zero errors. We combated this with Continuously Differentiable Exponential Linear Units (CeLU)~\cite{barron2017continuously} which gave us more stable gradients in training. We first enumerate the results for both the ReLU and CeLU activations below.

\textbf{ReLU Activation Results: } We first see that in this setting, the training of the undefended performs better on average (from train loss plot in Fig. \ref{fig:qptrain} and test loss plot in Fig. \ref{fig:qptest}). However, this is with an important caveat that we discard unsuccessful runs, only 12.8\% of the runs were successful (hence we ran over 250 random seeds to achieve the 30 runs needed here) . Further, sufficiently tune values of $B$ can perform similarly, and the best performing models of $B=2$ performs better than all runs in the original model (5.01676 vs. 5.01677 in the undefended case). 

\begin{figure}[ht]
\centering
\includegraphics[width=0.85\columnwidth]{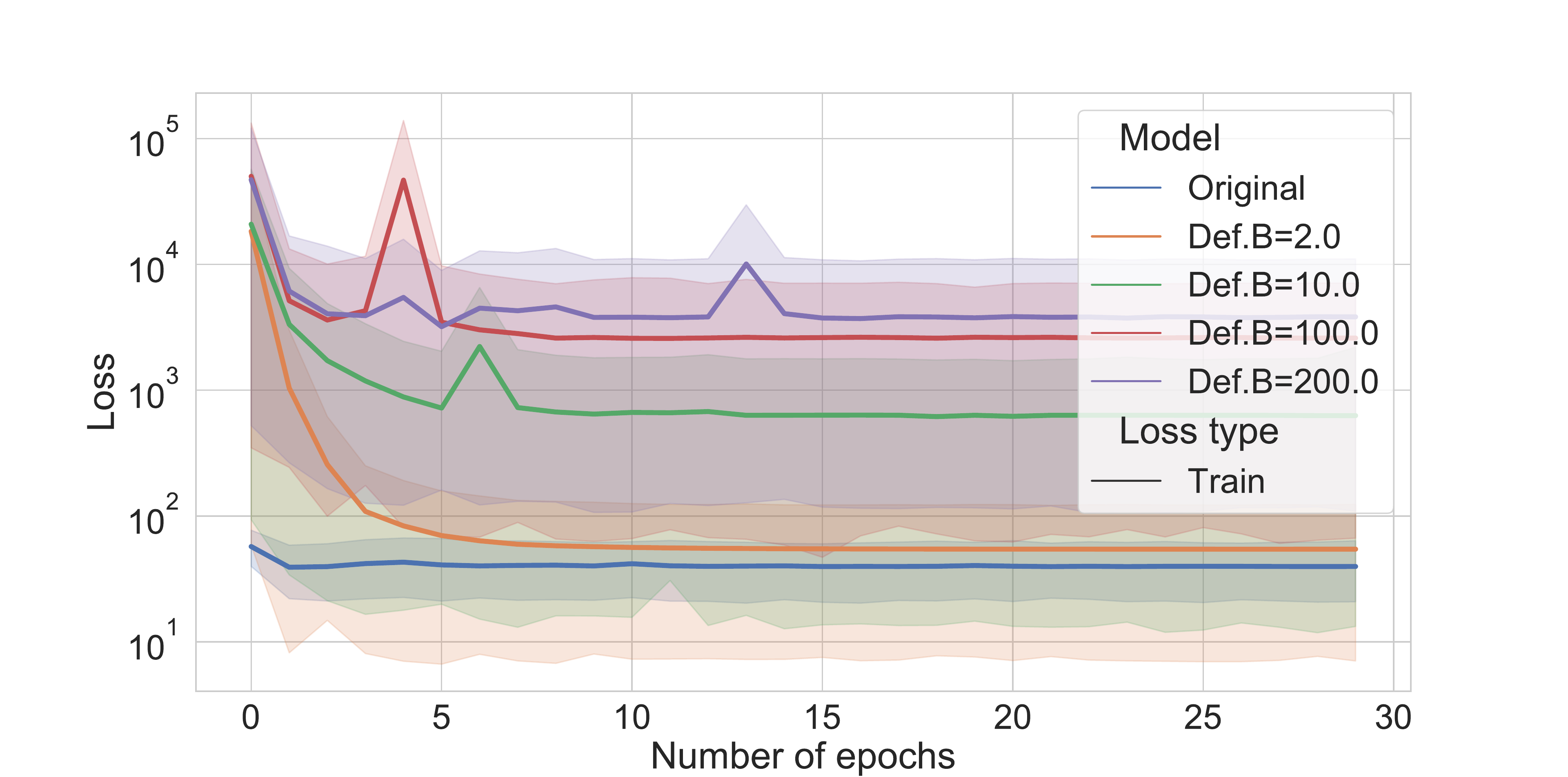}
\caption{Speed Profile Planning -- Training loss averaged over 30 runs. Note that while the original performs the best, these runs only included successful runs of the original training, which accounts for 12.8\% of the total runs. Further, we did not tune B to the lowest possible value.  Note that the best performing model for B=2 performs slightly better than the original model.}
\label{fig:qptrain}
\end{figure}
\begin{figure}[ht]
\centering
\includegraphics[width=0.75\columnwidth]{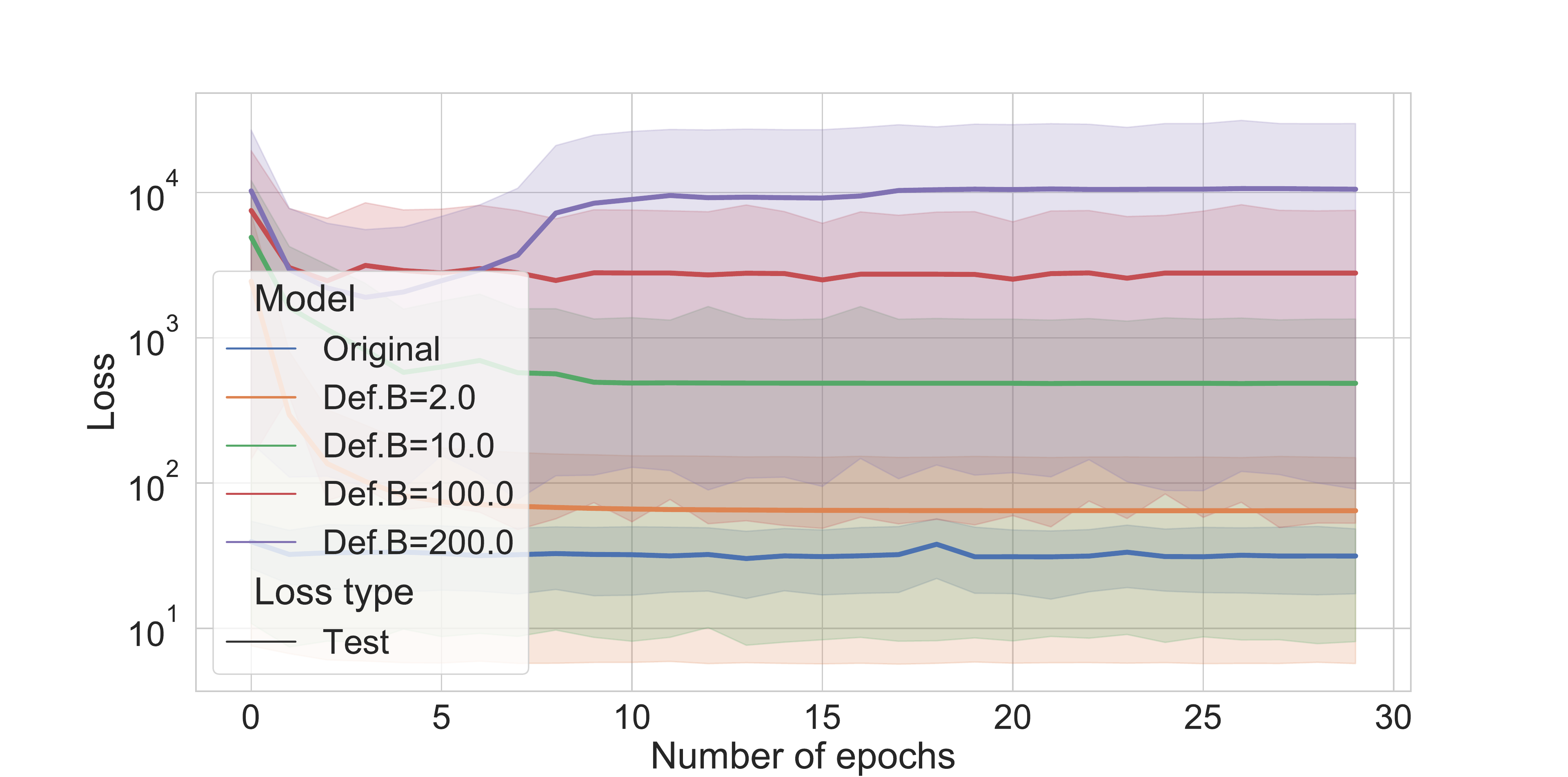}
\caption{Speed Profile Planning -- Test loss averaged over 30 runs. Note that while the original performs the best, these runs only included successful runs of the original training, which accounts for 12.8\% of the total runs. Further, we did not tune B to the lowest possible value. Note that the best performing model for B=2 performs slightly better than the original model.}
\label{fig:qptest}
\end{figure}

We also found that in the speed planning setting, the optimization layer gave an incorrect output that violated the constraints  $38.6\% \pm 3.16\%$ of the time (averaged across all models), despite us explicitly coding in constraints to enforce a non-negative speed. Any usage of the output would be disastrous in a critical setting. 

We performed attacks using all the attack methods on 30 images. The average and standard deviation are not plotted due to the amount of \texttt{inf} and  \texttt{NaN} in the condition numbers at different epochs for different runs, but we observe from the data that all methods can result in \texttt{NaN} in the output on the model that was attacked.

Finally, the defense works as expected, and the condition numbers are regulated at the value of $B$, giving us a plot similar to that in Fig. \ref{fig:50x50attackdef}.

\begin{figure}[ht]
\centering
\includegraphics[width=0.55\columnwidth]{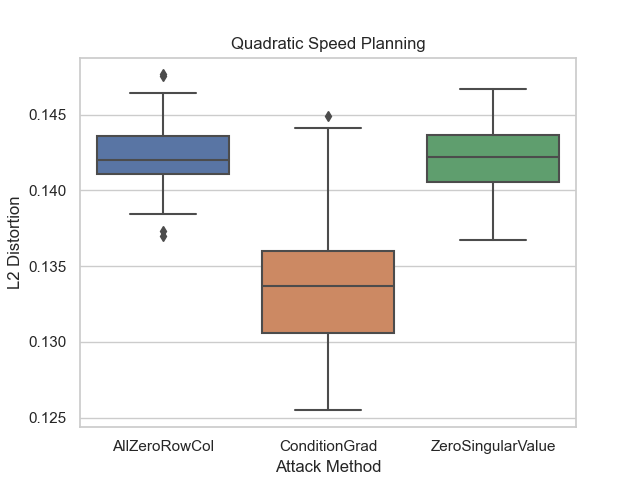}
\caption{Speed Profile Planning -- Comparison of L2 distortion measure as proposed in \cite{DBLP:journals/corr/SzegedyZSBEGF13}, which is a normalized form of the L2 magnitude of attack and is defined as $\sqrt{\sum{\frac{(x_i - x^\prime_{i})^2}{n}}}$for attack perturbations, where $n$ is 9411 for the Speed Profile Planning case. We see that \ConditionGrad{} produces the smallest perturbations, followed by \ZeroSingularValue{}, then \AllZeroRowCol{}. We hypothesize that the models are more brittle and therefore amenable to attacks by \ConditionGrad{} in this case.}
\label{fig:qpdistortionplot}
\end{figure}

\textbf{CeLU Activation Results: } We first see that in this setting, the training of the models with the defense performs better (from train loss plot in Fig. \ref{fig:qp_celu_train} and test loss plot in Fig. \ref{fig:qp_celu_test}). This supports our claim of the stabilizing effect of the defense on gradients, leading to healthier and more stable backpropagation and better performance.

\begin{figure}[ht]
\centering
\includegraphics[width=0.85\columnwidth]{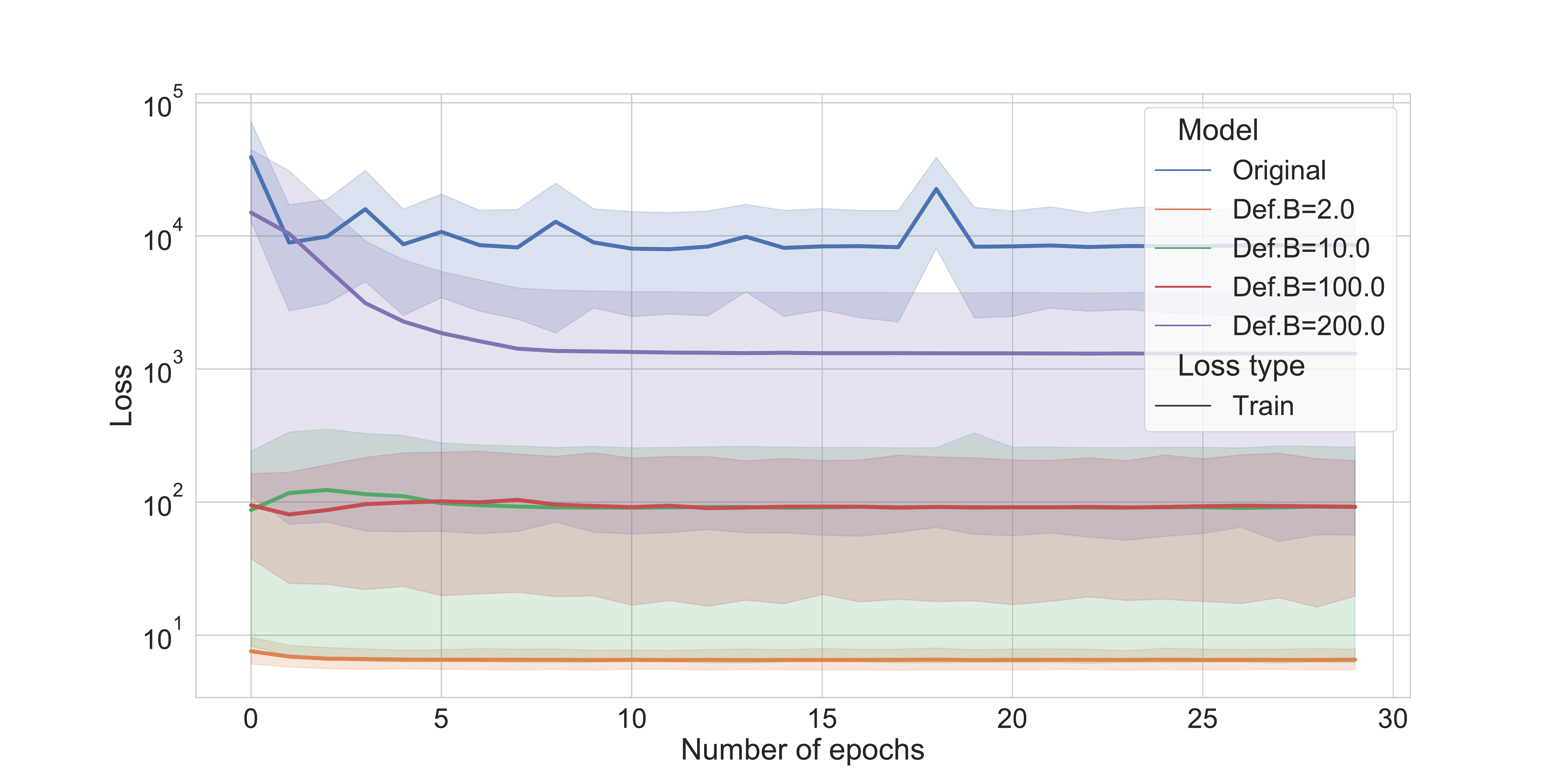}
\caption{Speed Profile Planning -- Training loss averaged over 30 runs. Loss is plotted on a log scale. Note that $B=2$ performs the best in this case, and the original model performs much more poorly, likely due to how unstable the training is.}
\label{fig:qp_celu_train}
\end{figure}
\begin{figure}[ht]
\centering
\includegraphics[width=0.85\columnwidth]{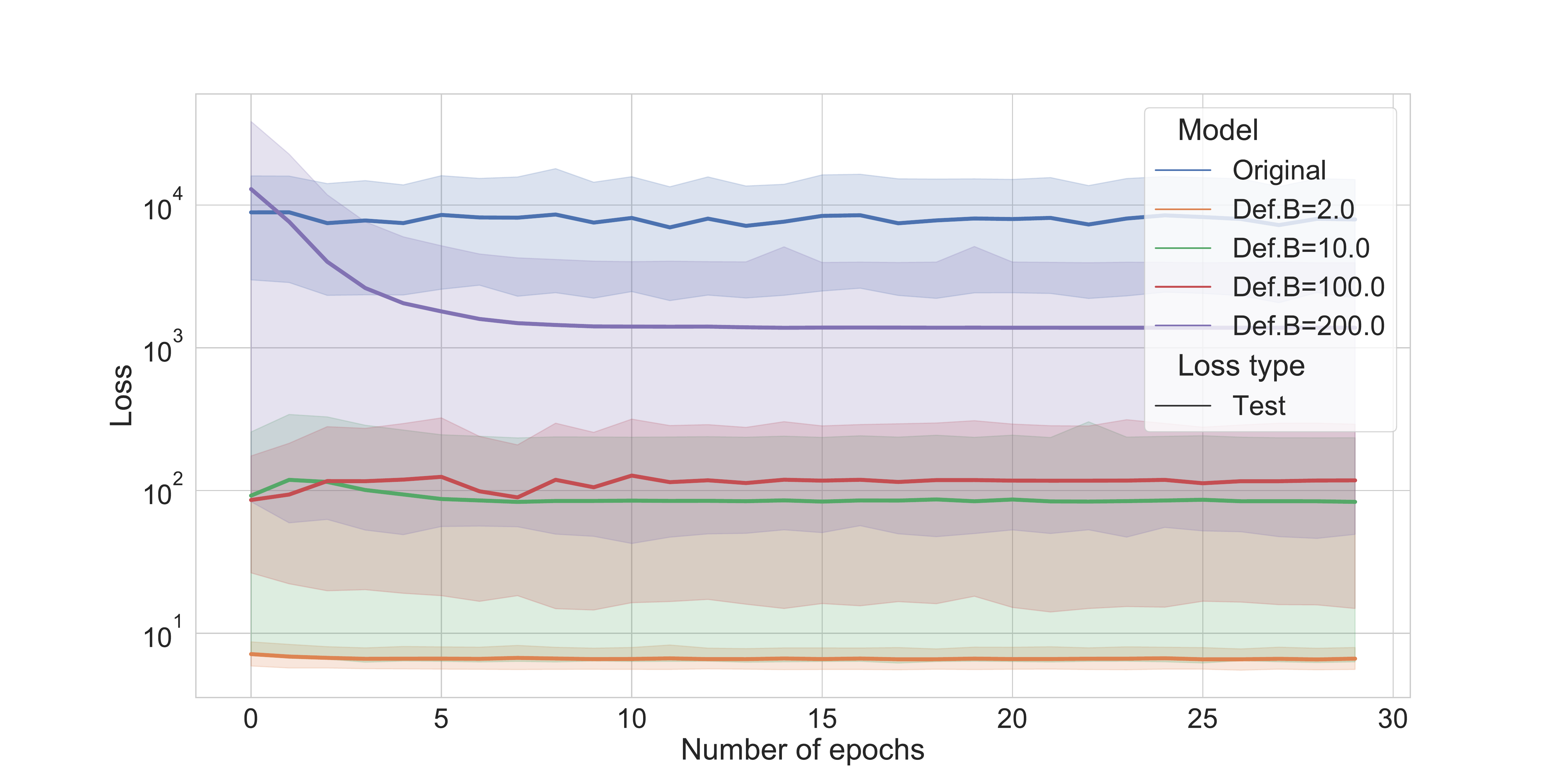}
\caption{Speed Profile Planning -- Test loss averaged over 30 runs. Loss is plotted on a log scale. Note that $B=2$ performs the best in this case, and the original model performs much more poorly due to the high condition number of the matrix.}
\label{fig:qp_celu_test}
\end{figure}

We performed attacks using all the attack methods on 30 images. The average and standard deviation are not plotted due to the amount of \texttt{inf} and  \texttt{NaN} in the condition numbers at different epochs for different runs, and we plot the L2 distortion of the successful attacks in Fig.~\ref{fig:qp_celu_distortionplot}, which shows reasonable deviations from the original image.

\begin{figure}[ht]
\centering
\includegraphics[width=0.5\columnwidth]{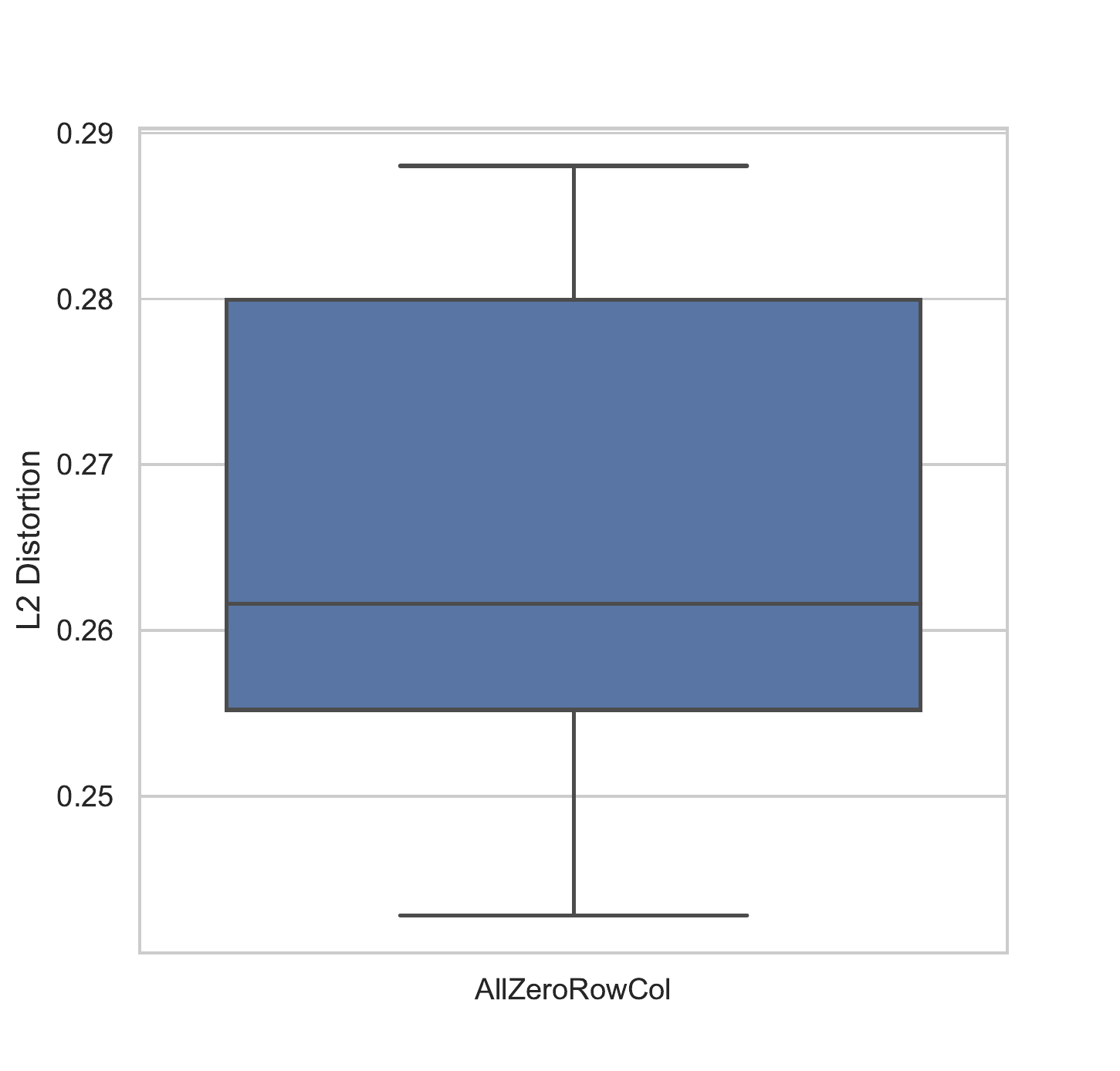}
\caption{Speed Profile Planning -- Comparison of L2 distortion measure as proposed in \cite{DBLP:journals/corr/SzegedyZSBEGF13}, which is a normalized form of the L2 magnitude of attack and is defined as $\sqrt{\sum{\frac{(x_i - x^\prime_{i})^2}{n}}}$for attack perturbations, where $n$ is 9411 for the Speed Profile Planning case. We see that  \AllZeroRowCol{} reasonable perturbations, even when the attacker's ability to modify the image is unconstrained.}
\label{fig:qp_celu_distortionplot}
\end{figure}

Finally, the defense works as expected, and the condition numbers are regulated at the value of $B$, giving us a plot similar to that in Fig. \ref{fig:50x50attackdef}.

\subsection{Other Activation Functions}
\label{appendix:activation_functions}

For completeness, we explore if other loss functions can be attacked. We first note that the exploration here is not evaluated on all the scenarios in our paper, but rather, they are evaluated on a small toy example to show that the attack is indeed possible.

% \paragraph{LeakyReLU} One natural extension would be to apply the same attack techniques in the ReLU setting to LeakyReLU. For this, we replaced the activation function in the Synthetic Data example, and trained it with the same settings. We then attacked it using the \AllZeroRowCol{}. Fig. \ref{fig:leakyrelu} shows the result of attacking on the condition number:

% \begin{figure}[t]
% \centering
% \includegraphics[width=0.95\columnwidth]{images/synthetic_data/leaky_relu_cond_num50x50.pdf}
% \caption{Condition number under \AllZeroRowCol{} attack, with the LeakyReLU activation function. We note that the condition number is unable to reach a high enough number to cause infeasibility, even when we search for a singular matrix directly through \AllZeroRowCol{}.}
% \label{fig:leakyrelu}
% \end{figure}

% This is not a straightforward extension as ReLU activated layers are able to more easily produce zeroes, however we do believe exact solvers may be able to find an input that creates the singular matrix that we need. We leave this exploration to future work.

\paragraph{Tanh} The tanh function is an interesting activation function (similar to CeLU) which could be used since it allows for both negative and positive values, rather than just non-negative values like ReLU. We explore this function in a restricted setting attacking a $\mathbb{R}^{2\times 2}$ matrix, and the exploration can be found in \texttt{attacks/tanh\_exploration/tanh.ipynb}, where we show that we achieve a successful attack in this setting in Fig \ref{fig:attacking_tanh}.

\begin{figure}[ht]
\centering
\subfigure{
\includegraphics[width=.35\columnwidth]{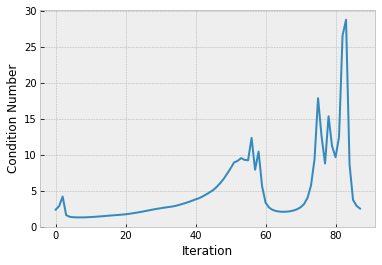}
}
\subfigure{
\includegraphics[width=.35\columnwidth]{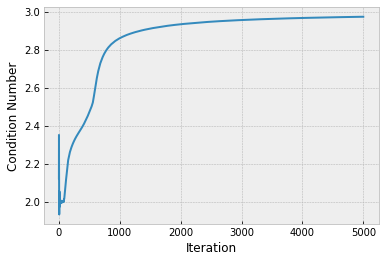}
}
\caption{Condition number in the tanh setting. The left diagram shows the effect of \ConditionGrad{} on a tanh activated neural network, with it successfully finding an attack sample. We note that while the condition number of the the constraint matrix was low in the attack sample, the intermediate matrix used in \texttt{qpth} was ill-conditioned which resulted in the \texttt{NaN}. The right diagram shows the defense with $B=2$ attacked over 5000 epochs, and in this case the attack was unsuccessful.}
\label{fig:attacking_tanh}
\end{figure}

\subsection{Targeting $0_{m\times n}$} 
\label{appendix:targetzero}

Large changes to $A$ are frequently restricted by the expressive power of the portion of the network that produces $A$, since changes to $A$ depend on perturbations in $u$. To explore the feasibility of making arbitrary changes to $A$, we conducted an experiment to explore the trivial attack (setting the target matrix $A'$ to be the all zero matrix and doing a gradient-descent based search on it) and found that with the same learning rate and running for 5000 epochs, \AllZeroRowCol{} manages to find attack inputs while the trivial attack fails to find such any input that results in an all zero $A$ in the Jigsaw Sudoku setting. The code for the experiment can be found under the folder \texttt{attack\_zero}. Instructions on how to run it are found in \texttt{README.md}.

\subsection{SVD Implementations}

Differing SVD implementations are known to behave differently based on the condition number and we discuss some of the ramifications on our results here. We use the PyTorch's default SVD implementation which uses \texttt{gesvdj} on the GPU and \texttt{gesdd} on the CPU, which could potentially vary in performance especially under large condition numbers. Empirically, we find that the different SVD implementations in PyTorch won't differ too much in the final network performance. For reference, in 10 randomly seeded runs on the 50x50 synthetic data experiment, the test loss was 4.01 $\pm$ 0.20 and 4.21 $\pm$ 0.52 for \texttt{gesvdj} on the GPU and \texttt{gesdd} on the CPU respectively for a very high condition number bound B of $10^7$. The code for the experiment can be found under the folder \texttt{svd}.

\subsection{Lipschitz Defense}

Lipschitz-constrained networks have been proposed as a means of defending against adversarial examples, and we discuss its applicability here. First, a known result is that the Lipschitz constant of a matrix is the highest singular value $\sigma_{\max}$ (for 2-norm) and the condition number is $\kappa_2 = \sigma_{\max} / \sigma_{\min}$. We can still construct inputs here that make $\sigma_{\min} = 0$ and in such cases controlling $\sigma_{\max}$ does not improve condition number at all. Of course, if $\sigma_{\min}$ is not exactly zero, controlling $\sigma_{\max}$ can help in lowering the condition number. In contrast, we directly control $\sigma_{\min}$ in our defense we always produce a well-conditioned A with any given desired bound on condition number. We ran experiments on the 50x50 synthetic data setting over 10 random seeds and found that the it did not perform as well as the setting without the defense (test loss of 14.51 $\pm$ 15.66 vs 4.43 $\pm$ 0.93), and further \AllZeroRowCol{} managed to find attack inputs for all models even when $\sigma_{\max}$ is controlled. The code for the experiment can be found under the folder \texttt{lipschitz}.

\subsection{$\eta I$ Defense}
\label{appendix:etai}

We discussed in detail the $\eta I$ Defense in the main paper, particularly its lack of general applicability and lack of theoretical guarantees. Further, what prohibited us from using this as a baseline was that $A$ is often a non-square matrix (this arises naturally in Jigsaw Sudoku, and we exercised this in our synthetic dataset as well). However, for the sake of completion, we provide the result for the 50x50 matrix synthetic data case, which we will later include in the appendix. We chose a small $\eta$ (1e-8, similar to the default penalty term for stability used in PyTorch) ran the experiment over 10 random seeds - we find that the test performance is worse than the original, with a slightly higher loss at 4.86 $\pm$ 1.74 vs 4.43 $\pm$ 0.93, and is also much higher than the case where B=200 (3.93 $\pm$ 0.07). 

However, we find that the defense does indeed work as one would expect and prevents attack inputs, at least empirically in the square matrix case, but as mentioned in the paper we don't have the same worst-case theoretical guarantees as the SVD defense as well as non-applicability of this baseline defense for non-square matrices. 

The code for the experiment can be found under the folder \texttt{etaI}.

\subsection{Constructing Semantically Meaningful Images}

We showcase that with some hand tuning of parameters of the attack, we are able to get semantically meaningful images for some models, employing an epsilon bound similar to the highest (0.25) proposed in \cite{43405} , as well as via restricting the attack to some color channels (Fig.~\ref{fig:attacks_small_jigsawsudoku}).

\begin{figure*}[ht]
\centering
\begin{minipage}{.6\textwidth}
\centering
    \includegraphics[width=.23\columnwidth]{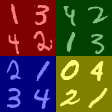}
    \includegraphics[width=.23\columnwidth]{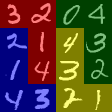}
    \includegraphics[width=.23\columnwidth]{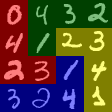}
    \includegraphics[width=.23\columnwidth]{images/photorealism/50000.png}
\end{minipage}
\begin{minipage}{.6\textwidth}
\centering
\includegraphics[width=.23\columnwidth]{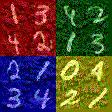}
\includegraphics[width=.23\columnwidth]{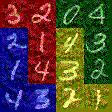}
\includegraphics[width=.23\columnwidth]{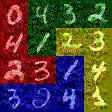}
\includegraphics[width=.23\columnwidth]{images/photorealism/adv_50000_linf_0.32_2channel_attack.jpg}
\end{minipage}

\caption{Jigsaw Sudoku images. For the first 3 columns, top row shows original images $u$, bottom row shows attack images $u'$ which satisfy $u' = u + \delta, \norm{\delta}_{\infty} \leq 0.2$. For the last column, we attack specific colors in Jigsaw Sudoku images. Top shows original image $u$; right shows $u' = u + \delta, \norm{\delta}_{\infty} \leq 0.32$, plus a restriction on only attack the 1st and 3rd color channels. The attack image looks less grainy in this formulation.We note that images in this setting are represented by a $3 \times 112 \times 112$ tensor, with pixel values normalized from $[0,1]$. All attacks were found using \AllZeroRowCol{}.}
\label{fig:attacks_small_jigsawsudoku}
\end{figure*}

\begin{figure}[ht]
\centering
\begin{minipage}{.51\textwidth}
\centering
\subfigure{
\includegraphics[width=.12\columnwidth]{images/photorealism/236.jpg}
}
\subfigure{
\includegraphics[width=.12\columnwidth]{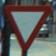}
}
\subfigure{
\includegraphics[width=.12\columnwidth]{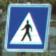}
}
\end{minipage}
\begin{minipage}{.51\textwidth}
\centering
\subfigure{
\includegraphics[width=.12\columnwidth]{images/photorealism/adv_236.jpg}
}
\subfigure{
\includegraphics[width=.12\columnwidth]{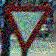}
}
\subfigure{
\includegraphics[width=.12\columnwidth]{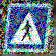}
}
\end{minipage}

\caption{Speed planning attack images. Top row shows original images $u$, bottom row shows attack images $u'$ which satisfy $u' = u + \delta, \norm{\delta}_{\infty} \leq 0.32$. We note that images in this setting are represented by a $3 \times 56 \times 56$ tensor, with pixel values normalized from $[0,1]$. Images may look worse than they appear due to artifacts from stretching the attack perturbations since input images are of low resolution. All attacks were found using \AllZeroRowCol{} in the ReLU setting.}
\label{fig:attack_speedplanning}
\end{figure}

General photorealistic attack images may be possible with more machinery, using unrestricted colorization/texture transfer techniques as in \cite{DBLP:conf/iclr/BhattadCLLF20} and we leave exploration and implementations of such techniques to future work.

The code to create the images can be found under the folder \texttt{photorealism}.

\subsection{Results for Maximizing Model Output}
\label{appendix:maxpgd}

We consider the scenario where an attacker aims to produce a NaN output via maximizing the output of the model. We apply this to the Jigsaw Sudoku setting as per Section~\ref{sec:jigsawsudoku}. However, we modify the goal of the adversary for this scenario. Here, the goal is to cause an overflow in the parameters so that output of the optimization layer would be numerically unstable and result in undefined behavior, leading to NaNs. We run the attack for 1000 iterations, optimizing to maximize the output of the model (measured as the absolute sum of all components of the output of the optimization layer). We run this attack for 30 different images, and plot the mean and standard errors.

The results are presented in Figure~\ref{fig:jigsaw_sudoku_maxpgd}. We first note that none of the attack produced an image that allowed the model to output a NaN. This is not feasible for the following reasons: the attack magnitude saturates at around 1700 on all attack images, thus the postulated overflow never happens. The condition number is also shown to saturate at around 1000, preventing a successful attack. We further note that it is possible to increase the model output without increasing the condition number of $A$, as proven in Lemma~\ref{lemma.1} and hence is not a principled way of attacking the system.

The code for the experiment can be found under the folder \texttt{max\_output}.

\begin{figure}[ht]
\centering
\subfigure{
    \includegraphics[width=.3\textwidth]{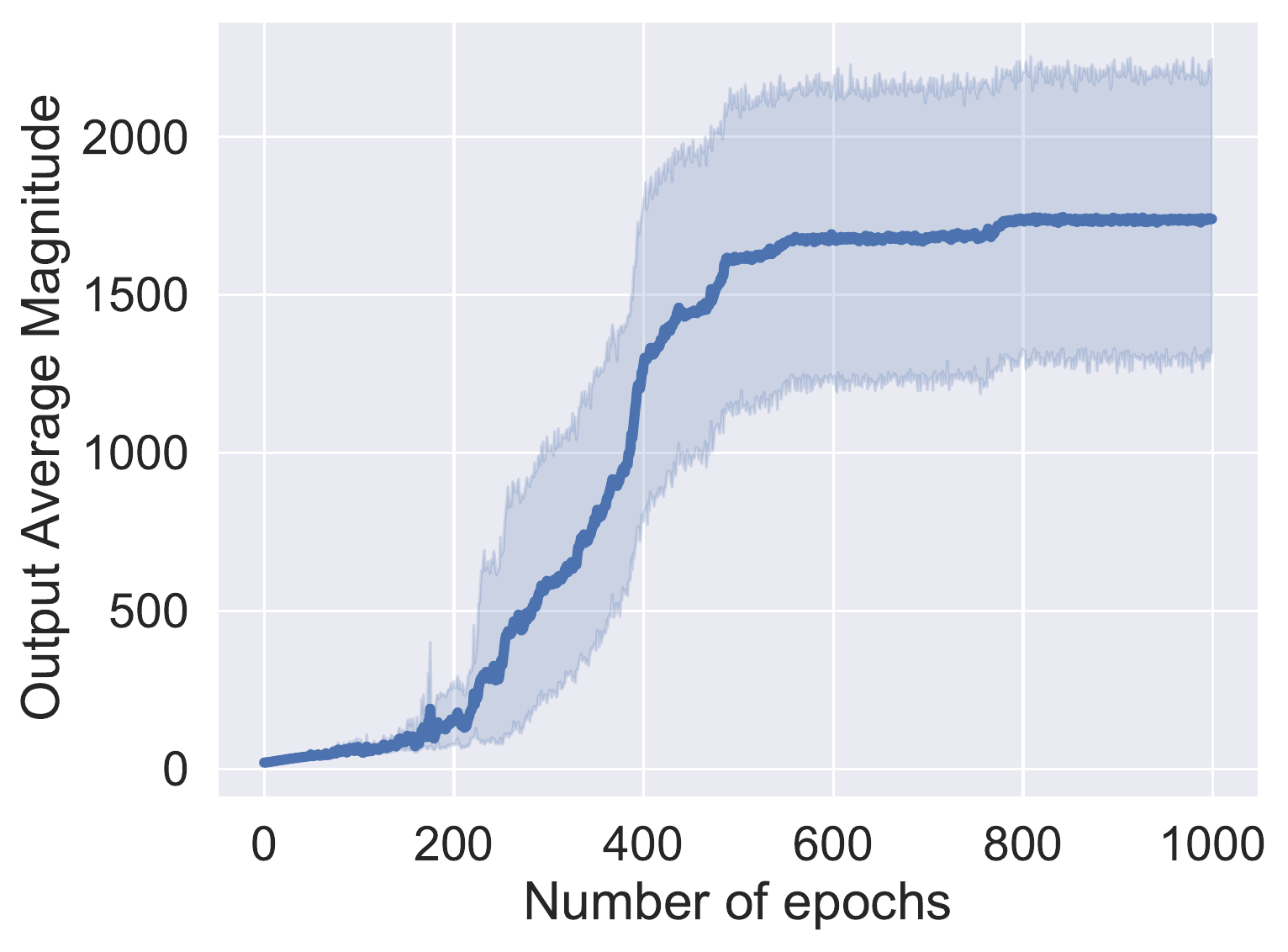}
}
\subfigure{
    \includegraphics[width=.3\textwidth]{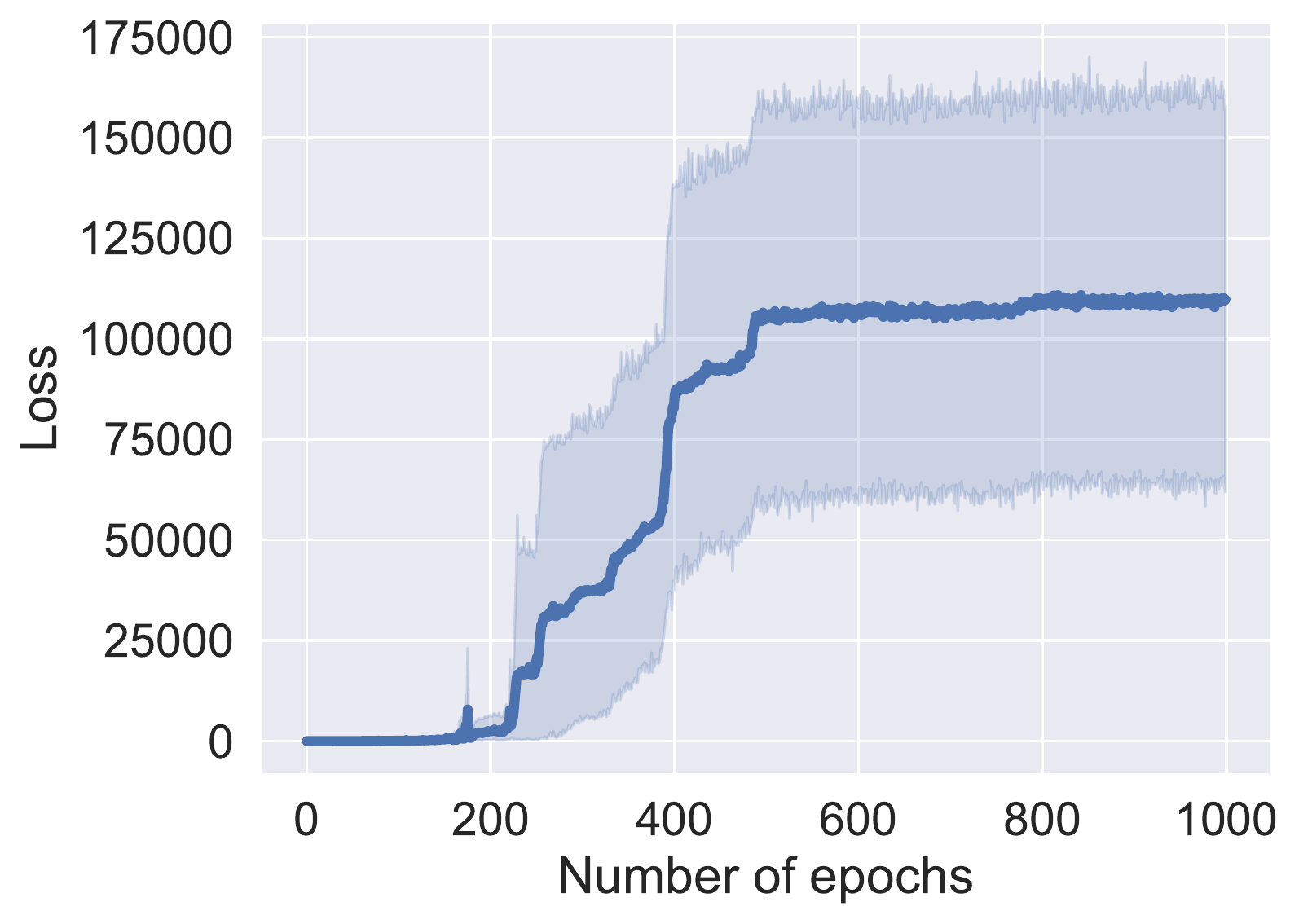}
}
\subfigure{
    \includegraphics[width=.3\textwidth]{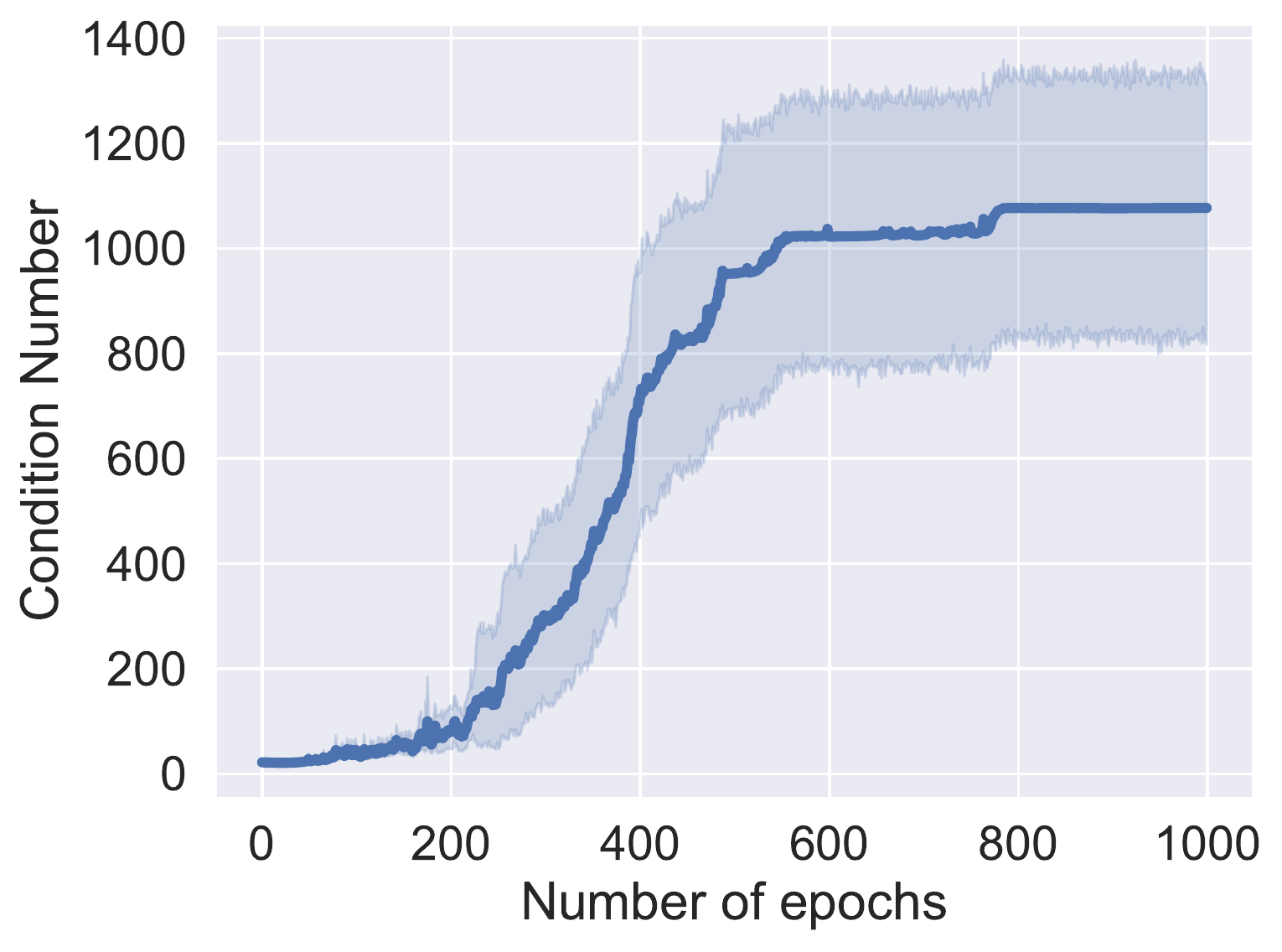}
}
\caption{Jigsaw Sudoku setting attacked using projected gradient descent, with the objective function modified to maximize the output of the optimization layer. We see that the average magnitude of the target saturates, as does the increase in the condition number.}
\label{fig:jigsaw_sudoku_maxpgd}
\end{figure}

% \begin{table*}[t]
% \caption{Comparison of attack success (\% of Successful NaNs) for all methods and datasets.}
% \centering
% \begin{tabular}{l c c c c c c}
% \toprule
%  &
%   \AllZeroRowCol &
%   \ZeroSingularValue &
%   \ConditionGrad \\
  
%   \midrule 
% \multicolumn{1}{l}{\textit{\textbf{Synthetic (m=40, n=50)}}} &
%   100.00 &
%   0.67 &
%   85.33 &
%  \\ 
% \multicolumn{1}{l}{\textit{\textbf{Synthetic (m=50, n=50)}}} &
%   98.00 &
%   0 &
%   0&
%  \\ 
% \multicolumn{1}{l}{\textit{\textbf{Jigsaw Sudoku}}} &
%   100.00 &
%   6.67 &
%   53.33 &
%  \\ 
% \multicolumn{1}{l}{\textit{\textbf{Speed Planning}}} &
%  100.00 &
%  96.67  &
%  96.67  &
% \\
% \multicolumn{1}{l}{\textit{\textbf{Defense (B=2,10,100,200)}}} &
%  \textbf{0.00} &
%  \textbf{0.00}  &
%  \textbf{0.00}  &
% \\
% \bottomrule
% \end{tabular}
% %}
% \label{tab:eq_attack_rate}
% \end{table*}

\section{Proofs Missing in Main Paper and Additional Theory Results}
\label{appendix:proofs}

\begin{lemma*} [Restatement of Lemma~\ref{lemma.1}]
For an optimization $\min_{\{x|Ax = b\}} f(x)$ with $f$ convex, the solution value (if it exists) can be made arbitrarily large by changing $\theta = \{A,b\}$ but keeping $A$ well-conditioned.
\end{lemma*}

\begin{proof}
It is always possible to choose $A,b$ such that the unconstrained minimum $x^*$ of $f(x)$ does not lie in $Ax = b$. This can seen by considering that if $Ax^* = b$, then choosing changing $b$ to $b - \epsilon$ makes $x^*$ infeasible. As this restriction does not affect $A$, let us start by choosing any well-conditioned, $A$ such that $x_0$ is the minimum with the constraint $Ax = b$. 
Also, $\nabla f(x_0) \neq 0$ since $x_0$ is not the unconstrained minimum. 
%Then, by convexity $f(y') \geq f(x_0) + \nabla f(x_0) (y' - x_0)$ for any $y'$. 
Further, if $Ay' = b$ and since $x_0$ is the minimum, by optimality condition of convex functions (with convex feasible region) we must have $\nabla f(x_0) (y' - x_0) \geq 0$ for any $y'$ with $A y' = b$. 

Consider the set $\{z ~|~ z = y + k\nabla f(x_0), Ay = b \}$. This set can be succinctly specified as $A y = b - k A \nabla f(x_0) = b'$. Let $y_0$ be the new minimum with this set of constraints. Then, by convexity $f(y_0) \geq f(x_0) + \nabla f(x_0) (y_0 - x_0)$. Since $y_0 = y' + k\nabla f(x_0)$ for some given $y'$ with $Ay'=b$, we have  $f(y_0) \geq f(x_0) + \nabla f(x_0) (y' - x_0) + k||\nabla f(x_0)||_2^2$. We know from last paragraph that $\nabla f(x_0) (y' - x_0) \geq 0$, thus, by choosing large $k$ we have that the output (minimum value) of the optimization can be made as large as possible. However, note that the matrix $A$ remains the same for this new optimization (with constraint $Ay = b'$), thus, forcing a large output of the optimization may not lead to an ill-conditioned matrix.
\end{proof}

\begin{lemma*} [Restatement of Lemma~\ref{lemma:gradcond}]
Let $A \in \mathbb{R}^{m \times n}$ with thin SVD $A = U \Sigma V^T$ and $\sigma_{\max} = \sigma_1 \geq \ldots \geq \sigma_r = \sigma_{\min}$ for $r = \min(m,n)$. Then, $\frac{\partial \kappa_2(A)}{\partial \theta_{i,j}}$ is given by $\tr\Big( \frac{\partial (\lvert\lvert A^{+}\rvert\rvert_2 * \lvert\lvert A\rvert\rvert_2)}{\partial A} \cdot\frac{\partial A}{\partial \theta_{i,j}}\Big)$ where:
\begin{align*}
    &\frac{\partial (\lvert\lvert A^{+}\rvert\rvert_2 * \lvert\lvert A\rvert\rvert_2)}{\partial A}  = 
    B^T   -  (A^{+} C A^{+})^T  + \\
    &\quad  (A^+)^T  A^+ C (I - A^+A)  +  (I - AA^+) C A^+ (A^+)^T
\end{align*}
with $B \!=\! \lvert\lvert A^+ \rvert\rvert_2 V e_1 e^T_1 U^T$\!, $C \!=\! \lvert\lvert A\rvert\rvert_2 U e_r e^T_r V^T $ and $e_i$ is the unit vector with one in the $i^{th}$ position.
\end{lemma*}

\begin{proof}[Proof of Lemma~\ref{lemma:gradcond}]
We use the differential technique in matrix calculus. We freely use the known result that the trace of a product of matrices is invariant under cyclic permutations: $\tr(ABC) = \tr(CAB) = \tr(BCA)$.
Let $A = U \Sigma V^T$, and it is assumed that the $\sigma_1 = \Sigma_{1,1}$ is the largest singular value. There are a total of $r = \min(m,n)$ singular values and $\sigma_r = \Sigma_{r,r}$ is the smallest singular value. Then $\lvert\lvert A\rvert\rvert = e_{1}^T \Sigma e_{1}$, where $e_{i}$ is a one-hot vector of size $r$ with one in position $i$. We first show that $\tr(d(\Sigma)) = \tr(U^T d(A) V)$. To see this, observe that
$$
d(A) = d\big(U \Sigma V^T) = d\big(U) \Sigma V^T + U d(\Sigma) V^T + U \Sigma d(V^T)
$$
Multiplying both sides on left by $U^T$ and on right by $V$ and recalling that $U^T U = I, V^T V = I$ we get
\begin{equation} %\label{eq:UTdAV}
 U^T d(A) V = U^T d\big(U) \Sigma + d(\Sigma)  +  \Sigma d(V^T) V   
\end{equation}

Also, multiplying by unit vectors we have
\begin{equation}
\small
\label{eq:UTdAV}
 e_i^T  U^T d(A) V e_i = e_i^T U^T d\big(U) \Sigma e_i + e_i^T d(\Sigma)e_i   +  e_i^T \Sigma d(V^T) V e_i 
\end{equation}

%$$
%\tr(U^T d(A) V) = \tr(U^T d\big(U) \Sigma) + %\tr(d(\Sigma))  +  \tr(\Sigma d(V^T) V)
%$$

Next, we use the property that $d(YZ) = d(Y)Z + Yd(Z)$ and $d(Y^T) =  d(Y)^T$ to get that $d(Y^T Y) = d(Y)^T Y + Y^T d(Y)$. Also, note that $ [d(Y)^T Y]^T = Y^T d(Y)$ and $\tr(Y^T + Y) = 2 \tr(Y)$, hence $\tr\big(d(Y^T Y)\big) = 2 \tr(Y^T d(Y))$. 
Then, since $U^TU = I$ and $d(I) = 0$, letting $Y = U^Td(U)$ we get $Y + Y^T = 0$. Thus, $Y$ is skew symmetric. It is known that the trace of the product of a symmetric and skew symmetric matrix is zero. $\Sigma e_i e^T_i$ is a symmetric matrix. Thus, $\tr(Y \Sigma e_i e^T_i) = \tr( e^T_i U^Td(U) \Sigma e_i ) = 0$. Very similar reasoning gives $\tr(e^T_i \Sigma d(V^T) V e_i) = 0$. Then, using these and by taking trace of Equation~\ref{eq:UTdAV} we get
$$
\tr(e^T_i U^T d(A) V e_i ) =  \tr(d(e^T_i \Sigma e_i))  
$$
Next, observe that since $\lvert\lvert A\rvert\rvert = \tr(e_{1}^T \Sigma e_{1})$ and the fact that  the fact that $d(\tr(AX) = \tr(d(X))$, we have 

$$d(\lvert\lvert A\rvert\rvert) = \tr(d(e^T_1 \Sigma e_1))  = \tr(e^T_1 U^T d(A) V e_1 ) $$
$$\!\!\!\!\!\!= \tr(V e_1 e^T_1 U^T d(A))$$

The last step above use cyclic permutation within trace.
Next, observe that $A^+ = V \Sigma^+ U^T$ and  $\lvert\lvert A^+ \rvert\rvert = \tr(e_{r}^T \Sigma^+ e_{r})$, thus, similar to $A$ we have 
$$d(\lvert\lvert A^+ \rvert\rvert) = \tr(d(e^T_r \Sigma^+ e_r))  = \tr(e^T_r V^T d(A^+) U e_r ) $$
\noindent $$\!\!\!\!\!\!= \tr(U e_r e^T_r V^T d(A^+))$$
Then,
\begin{align*}
    d(&\lvert\lvert A\rvert\rvert * \lvert\lvert A^{+} \rvert\rvert) \\
    = \;& d\big(\lvert\lvert A\rvert\rvert) * \lvert\lvert A^+ \rvert\rvert + \lvert\lvert A\rvert\rvert * d\big(\lvert\lvert A^+ \rvert\rvert) & \\
    = \;& \tr(V e_1 e^T_1 U^T d(A)) * \lvert\lvert A^+ \rvert\rvert + \lvert\lvert A\rvert\rvert * \tr(U e_r e^T_r V^T d(A^+)) & 
\end{align*}
It is known that~\cite{golub1973differentiation}:
\begin{align*}
    d(A^{+}) &= - A^{+} d(A) A^{+} + (I - A^+A) d(A^T) (A^+)^T A^+  \\
    &+ A^+ (A^+)^T d(A^T) (I - AA^+)
\end{align*}
% $$d(A^{+}) = - A^{+} d(A) A^{+} + (I - A^+A) d(A^T) (A^+)^T A^+  + A^+ (A^+)^T d(A^T) (I - AA^+) $$ 
Using the shorthand $B = \lvert\lvert A^+ \rvert\rvert V e_1 e^T_1 U^T$ and $C = \lvert\lvert A\rvert\rvert U e_r e^T_r V^T $, we continue the equations from above as
{\small 
\begin{align*}
    d(&\lvert\lvert A\rvert\rvert * \lvert\lvert A^{+} \rvert\rvert) \\
    = \;& \tr(B d(A))  + \tr(C d(A^+)) & \\
    = \;& \tr(B d(A))\!-\!\tr(C A^{+} d(A) A^{+}) + \tr(C (I - A^+A) d(A^T) (A^+)^T A^+)  \\
    & + \tr(C A^+ (A^+)^T d(A^T) (I - AA^+) ) & \\
    & \mbox{Using fact that trace is invariant under cyclic permutation}\\
    & \mbox{ for the last three terms} \\
    = \;& \tr(B d(A))\!-\!\tr( A^{+} C A^{+} d(A) ) + tr(d(A^T) (A^+)^T A^+ C (I - A^+A)  ) \\
    & + \tr(d(A^T) (I - AA^+) C A^+ (A^+)^T  ) & \\
    & \mbox{Using fact that trace of a transpose $Y^T$ is same as trace of $Y$ }\\
    & \mbox{ for the last two terms and $d(A^T) = (d(A))^T$} \\
    = \;& \tr(B d(A))  - \tr( A^{+} C A^{+} d(A) ) + \tr((I - A^+A)^T C^T (A^+)^T A^+ d(A)) \\
    & + \tr(   A^+ (A^+)^T C^T (I - AA^+)^T d(A) ) & \\
    = \;& \tr\Bigg( \Big( B   -  A^{+} C A^{+}   + (I - A^+A)^T C^T (A^+)^T A^+  \\
    & +  A^+ (A^+)^T C^T (I - AA^+)^T \Big) d(A) \Bigg) & \\
\end{align*}
}
 With this and as the differential is of the form $dx =\tr(Z d(Y))$, which gives $\frac{\partial x}{\partial Y} = Z^T$
\begin{align*}
  & \frac{\partial (\lvert\lvert A\rvert\rvert * \lvert\lvert A^{+}\rvert\rvert)}{\partial A} =  \Big( B^T   -  (A^{+} C A^{+})^T  + 
  \\
  & \quad (A^+)^T  A^+ C (I - A^+A)  +  (I - AA^+) C A^+ (A^+)^T   \Big )
\end{align*}

 Now, using the fact that $\frac{\partial g(U) }{\partial x} = \tr(\frac{\partial g(U) }{\partial U} \frac{\partial U }{\partial x})$ ($U$ is a matrix; $g(U), x$ are real numbers), our original derivative that we needed is
 \begin{align*} 
 \frac{\partial (\lvert\lvert A^{+}\rvert\rvert * \lvert\lvert A\rvert\rvert)}{\partial \theta_{i,j}} = \;&\tr \Big( \frac{\partial (\lvert\lvert A^{+}\rvert\rvert * \lvert\lvert A\rvert\rvert)}{\partial A} \frac{\partial A}{\partial \theta_{i,j}}\Big)
% = \;&  \frac{\lvert\lvert A^{-1}\rvert\rvert}{\lvert\lvert A\rvert\rvert} * \tr\Big(A \frac{\partial A}{\partial \theta_j} \Big) - \frac{\lvert\lvert A\rvert\rvert}{\lvert\lvert A^{-1}\rvert\rvert} * \tr\Big( (A^{-1})^T A^{-1} (A^{-1})^T \frac{\partial A}{\partial \theta_j} \Big)
 \end{align*} 
 which concludes our proof.
\end{proof}

\begin{lemma}[Lemma~\ref{lemma:gradcond} variant for $\kappa_F(A)$]
 $\frac{\partial \kappa_2(A)}{\partial \theta_{i,j}}$ is given by $\tr\Big( \frac{\partial (\lvert\lvert A^{+}\rvert\rvert * \lvert\lvert A\rvert\rvert)}{\partial A} \frac{\partial A}{\partial \theta_{i,j}}\Big)$ where $\frac{\partial (\lvert\lvert A^{+}\rvert\rvert * \lvert\lvert A\rvert\rvert)}{\partial A} $ is

\begin{align*}
    &\frac{\lvert\lvert A^{+}\rvert\rvert}{\lvert\lvert A\rvert\rvert} * A  + \frac{\lvert\lvert A\rvert\rvert}{\lvert\lvert A^{+}\rvert\rvert} * \Big( (A^{+})^T A^{+} (A^{+})^T - 
 \\
    & \qquad (A^+)^T A^+ (A^+)^T A^+  A  - A A^+ (A^+)^T A^+ (A^+)^T \Big )
\end{align*}
\end{lemma}

\begin{proof}

We use the differential technique. We freely use the known result that the trace of a product of matrices is invariant under cyclic permutations: $\tr(ABC) = \tr(CAB) = \tr(BCA)$. Let $X$ denote $A^{+}$.

\begin{align*}
    d(&\lvert\lvert A^{+}\rvert\rvert * \lvert\lvert A\rvert\rvert) \\
    = \;& \lvert\lvert X\rvert\rvert * d\big(\sqrt{\lvert\lvert A\rvert\rvert^2}\big) + d\big(\sqrt{\lvert\lvert X\rvert\rvert^2}\big) * \lvert\lvert A\rvert\rvert& \\
    = \;& \frac{\lvert\lvert X\rvert\rvert}{2\lvert\lvert A\rvert\rvert} * d\big(\lvert\lvert A\rvert\rvert^2\big) + d\big(\lvert\lvert X\rvert\rvert^2\big) * \frac{\lvert\lvert A\rvert\rvert}{2\lvert\lvert X\rvert\rvert}& \\
    = \;& \frac{\lvert\lvert X\rvert\rvert}{2\lvert\lvert A\rvert\rvert} * d\big(\tr(A^T A)\big) + d\big(\tr(X^T X)\big) * \frac{\lvert\lvert A\rvert\rvert}{2\lvert\lvert X\rvert\rvert} \\
    & \mbox{ as } \lvert\lvert Y\rvert\rvert^2 = \tr(Y^T Y)\\
    = \;& \frac{\lvert\lvert X\rvert\rvert}{2\lvert\lvert A\rvert\rvert} * \tr\big(d(A^T A)\big) + \tr\big(d(X^T X)\big) * \frac{\lvert\lvert A\rvert\rvert}{2\lvert\lvert X\rvert\rvert} \\
    & \mbox{ as } d(\tr(Y)) = \tr(d(Y))
\end{align*}

Next, we use the property that $d(YZ) = d(Y)Z + Yd(Z)$ and $d(Y^T) =  d(Y)^T$ to get that $d(Y^T Y) = d(Y)^T Y + Y^T d(Y)$. Also, note that $ [d(Y)^T Y]^T = Y^T d(Y)$ and $\tr(Y^T + Y) = 2 \tr(Y)$, hence $\tr\big(d(Y^T Y)\big) = 2 \tr(Y^T d(Y))$. Using this (with $Y$ as $A$ or $X$) we get

\begin{align*}
    d(&\lvert\lvert A^{+}\rvert\rvert * \lvert\lvert A\rvert\rvert) \\
    = \;& \frac{\lvert\lvert X\rvert\rvert}{\lvert\lvert A\rvert\rvert} * \tr\big(A^T d(A)\big) + \tr\big(X^T d( X)\big) * \frac{\lvert\lvert A\rvert\rvert}{\lvert\lvert X\rvert\rvert} & \\
 \end{align*}   
 
Since $X = A^{+}$, from~\cite{golub1973differentiation}
\begin{align*}
& d(A^{+}) = - A^{+} d(A) A^{+} + (I - A^+A) d(A^T) (A^+)^T A^+  +\\
& \qquad\qquad A^+ (A^+)^T d(A^T) (I - AA^+)     
\end{align*}  
and the trace is invariant under cyclic permutations, we get
{\small
\begin{align*}
    d(&\lvert\lvert A^{+}\rvert\rvert * \lvert\lvert A\rvert\rvert) \\
    = \;& \frac{\lvert\lvert A^{+}\rvert\rvert}{\lvert\lvert A\rvert\rvert} * \tr\big(A^T d(A)\big) + \tr\big(-(A^{+})^T A^{+} d( A) A^{+} \big) * \frac{\lvert\lvert A\rvert\rvert}{\lvert\lvert A^{+}\rvert\rvert} & \\
     & \quad + \tr\big( (A^+)^T (I - A^+A) d(A^T) (A^+)^T A^+ \big) * \frac{\lvert\lvert A\rvert\rvert}{\lvert\lvert A^{+}\rvert\rvert}\\
     & \quad + \tr\big( (A^+)^T A^+ (A^+)^T d(A^T) (I - AA^+)  \big) * \frac{\lvert\lvert A\rvert\rvert}{\lvert\lvert A^{+}\rvert\rvert}  \\
     & \mbox{Using fact that trace is invariant under cyclic permutation }\\
     & \mbox{ for the last three terms} \\
    = \;& \frac{\lvert\lvert A^{+}\rvert\rvert}{\lvert\lvert A\rvert\rvert} * \tr\big(A^T d(A)\big) + \tr\big(-A^{+} (A^{+})^T A^{+} d( A)  \big) * \frac{\lvert\lvert A\rvert\rvert}{\lvert\lvert A^{+}\rvert\rvert} & \\
         & \quad + \tr\big( d(A^T) (A^+)^T A^+ (A^+)^T (I - A^+A)  \big) * \frac{\lvert\lvert A\rvert\rvert}{\lvert\lvert A^{+}\rvert\rvert} \\
         & \quad + \tr\big( d(A^T) (I - AA^+) (A^+)^T A^+ (A^+)^T   \big) * \frac{\lvert\lvert A\rvert\rvert}{\lvert\lvert A^{+}\rvert\rvert}  \\
    & \mbox{Using fact that trace of a transpose $Y^T$ is same as trace of $Y$ }\\
    & \mbox{ for the last two terms and $d(A^T) = (d(A))^T$} \\
    = \;& \frac{\lvert\lvert A^{+}\rvert\rvert}{\lvert\lvert A\rvert\rvert} * \tr\big(A^T d(A)\big) + \tr\big(-A^{+} (A^{+})^T A^{+} d( A)  \big) * \frac{\lvert\lvert A\rvert\rvert}{\lvert\lvert A^{+}\rvert\rvert} & \\
    & \quad + \tr\big( (I - A^+A)^T A^+ (A^+)^T A^+  d(A)  \big) * \frac{\lvert\lvert A\rvert\rvert}{\lvert\lvert A^{+}\rvert\rvert} \\
    & \quad + \tr\big(  A^+ (A^+)^T A^+ (I - AA^+)^T d(A)  \big) * \frac{\lvert\lvert A\rvert\rvert}{\lvert\lvert A^{+}\rvert\rvert}  \\
    = \;& \tr \Bigg (\frac{\lvert\lvert A^{+}\rvert\rvert}{\lvert\lvert A\rvert\rvert} * \big(A^T d(A)\big) + \\
    & \Big(-A^{+} (A^{+})^T A^{+} d( A) + (I - A^+A)^T A^+ (A^+)^T A^+  d(A) \\
    & + A^+ (A^+)^T A^+ (I - AA^+)^T d(A)  \Big) * \frac{\lvert\lvert A\rvert\rvert}{\lvert\lvert A^{+}\rvert\rvert} \Bigg) & \\
    = \;& \tr \Bigg (\Big( \frac{\lvert\lvert A^{+}\rvert\rvert}{\lvert\lvert A\rvert\rvert} * A^T  + \\
    & \frac{\lvert\lvert A\rvert\rvert}{\lvert\lvert A^{+}\rvert\rvert} * \big( - A^{+} (A^{+})^T A^{+} + (I - A^+A)^T A^+ (A^+)^T A^+ +\\
    & A^+ (A^+)^T A^+ (I - AA^+)^T \big) \Big ) d(A) \Bigg) & \\
 \end{align*} 
 }
 It can be seen that $\big( - A^{+} (A^{+})^T A^{+} + (I - A^+A)^T A^+ (A^+)^T A^+ + A^+ (A^+)^T A^+ (I - AA^+)^T \big)$ reduces to $\big(  - A^T(A^+)^T A^+ (A^+)^T A^+ + A^+ (A^+)^T A^+ - A^+ (A^+)^T A^+ (A^+)^T A^T \big)$
 
 With this and as the differential is of the form $dx =\tr(Z d(Y))$, which gives $\frac{\partial x}{\partial Y} = Z^T$
\begin{align*}
& \frac{\partial (\lvert\lvert A^{+}\rvert\rvert * \lvert\lvert A\rvert\rvert)}{\partial A} = \frac{\lvert\lvert A^{+}\rvert\rvert}{\lvert\lvert A\rvert\rvert} * A  + \frac{\lvert\lvert A\rvert\rvert}{\lvert\lvert A^{+}\rvert\rvert} * \Big( (A^{+})^T A^{+} (A^{+})^T - \\
&\qquad\quad (A^+)^T A^+ (A^+)^T A^+  A  - A A^+ (A^+)^T A^+ (A^+)^T \Big )
\end{align*}

 Now, using the fact that $\frac{\partial g(U) }{\partial x} = \tr(\frac{\partial g(U) }{\partial U} \frac{\partial U }{\partial x})$ ($U$ is a matrix; $g(U), x$ are real numbers), our original derivative that we needed is
 \begin{align*} 
 \frac{\partial (\lvert\lvert A^{+}\rvert\rvert * \lvert\lvert A\rvert\rvert)}{\partial \theta_{i,j}} = \;&\tr \Big( \frac{\partial (\lvert\lvert A^{+}\rvert\rvert * \lvert\lvert A\rvert\rvert)}{\partial A} \frac{\partial A}{\partial \theta_{i,j}}\Big)\\
% = \;&  \frac{\lvert\lvert A^{-1}\rvert\rvert}{\lvert\lvert A\rvert\rvert} * \tr\Big(A \frac{\partial A}{\partial \theta_j} \Big) - \frac{\lvert\lvert A\rvert\rvert}{\lvert\lvert A^{-1}\rvert\rvert} * \tr\Big( (A^{-1})^T A^{-1} (A^{-1})^T \frac{\partial A}{\partial \theta_j} \Big)
 \end{align*} 
which concludes our proof.
\end{proof}

\begin{proposition*}[Restatement of Proposition~\ref{prop:cond_def}]
For the approximate $A'$ obtained from $A$ as described above and  $x'$ a solution for $A' x = b$, the following hold: (1)
$\norm{A' - A}_2 \leq  \sigma_{\max} /B $ 
and (2) $\frac{||x^* - x'||_2}{||x'||_2} \leq \kappa_2(A) / B$ for some solution $x^*$ of $A x = b$.
\end{proposition*}

\begin{proof}[Proof of Proposition~\ref{prop:cond_def}]

Let $\epsilon = 1/B$. As $A, A'$ have the same $U,V$ in their respective SVD, $A' - A = \sum_i (\sigma_i' - \sigma_i)u_i v_i^T$ where $u_i, v_i$ are columns of the matrix $U,V$. The largest value of $\sigma_i - \sigma_i'$ can be $\epsilon\sigma_{\max}$. Thus, from definition of matrix 2-operator norm and the fact that $\norm{A}_2 = \sigma_{\max}$, we obtain $\norm{A' - A}_2 < \epsilon \sigma_{\max}$.

Next, let $A' = A + \Delta A$ and let the solution obtained using $A'$ be $x'$ and one using $A$ be $x^*$. (Here by solution we mean the canonical $A^+ b$).  According to~\cite{grcar2010matrix}, a lower bound matrix norm $\norm{A}_l$ is defined and by Theorem 5.3 of~\cite{grcar2010matrix}, it can be shown that $\frac{\norm{x^* - x'}_2}{\norm{x'}_2} \leq \frac{\norm{\Delta A}_2}{\norm{A}_l}$. Further, according to Lemma 2.2 of~\cite{grcar2010matrix}, we have $1 \leq \norm{A}_l \norm{A^+}_2$. Using this and the fact we already proved that $\norm{\Delta A}_2 < \epsilon \sigma_{\max} = \epsilon \norm{A}_2$, we get $\frac{\norm{x^* - x'}_2}{\norm{x'}_2} \leq \epsilon \norm{A^+}_2 \norm{A}_2 = \epsilon \kappa_2(A)$.

\end{proof}

%\section{Final instructions}
%\begin{itemize}

%\item You should directly generate PDF files using \verb+pdflatex+.

%\item You can check which fonts a PDF files uses.  In Acrobat Reader, select the
%  menu Files$>$Document Properties$>$Fonts and select Show All Fonts. You can
%  also use the program \verb+pdffonts+ which comes with \verb+xpdf+ and is
%  available out-of-the-box on most Linux machines.

%\item The IEEE has recommendations for generating PDF files whose fonts are also
%  acceptable for NeurIPS. Please see
%  \url{http://www.emfield.org/icuwb2010/downloads/IEEE-PDF-SpecV32.pdf}

%\item \verb+xfig+ "patterned" shapes are implemented with bitmap fonts.  Use
%  "solid" shapes instead.

%\item The \verb+\bbold+ package almost always uses bitmap fonts.  You should use
%  the equivalent AMS Fonts:
%\begin{verbatim}
%   \usepackage{amsfonts}
%\end{verbatim}
%followed by, e.g., \verb+\mathbb{R}+, \verb+\mathbb{N}+, or \verb+\mathbb{C}+
%for $\mathbb{R}$, $\mathbb{N}$ or $\mathbb{C}$.  You can also use the following
%workaround for reals, natural and complex:
%\begin{verbatim}
 %  \newcommand{\RR}{I\!\!R} %real numbers
 %  \newcommand{\Nat}{I\!\!N} %natural numbers
 %  \newcommand{\CC}{I\!\!\!\!C} %complex numbers
%\end{verbatim}
%Note that \verb+amsfonts+ is automatically loaded by the %\verb+amssymb+ package.

%\end{itemize}

%If your file contains type 3 fonts or non embedded TrueType fonts, we will ask
%you to fix it.

%%%%%%%%%%%%%%%%%%%%%%%%%%%%%%%%%%%%%%%%%%%%%%%%%%%%%%%%%%%%%%%%%%%%%%%%%%%%%%%
%%%%%%%%%%%%%%%%%%%%%%%%%%%%%%%%%%%%%%%%%%%%%%%%%%%%%%%%%%%%%%%%%%%%%%%%%%%%%%%

\end{document}